\documentclass[11pt]{article} 
\usepackage{amsmath}
\usepackage{mathrsfs}
\usepackage{enumitem}
\usepackage{url} 

\usepackage[nohyperref]{jmlr2eARXIV}

\usepackage[colorlinks=false,allbordercolors={1 1 1}]{hyperref}

\DeclareMathOperator{\card}{card}
\DeclareMathOperator{\rate}{rate}
\DeclareMathOperator{\totient}{EulerTotient}

\newcommand\E{\mathbb{E}}
\newcommand\Z{\mathbb{Z}}
\newcommand\N{\mathbb{Z}_{++}}
\newcommand\R{\mathbb{R}}
\newcommand\IR{\mathcal{I}}
\newcommand\tIR{\tilde{\mathcal{I}}}
\newcommand\bR{\overline{\mathbb{R}}}
\newcommand\uuv{\mathscr{L}}
\newcommand\uvv{\mathscr{R}}
\newcommand\morph{\mathscr{M}}

\newcommand\abs[1]{{\left| #1 \right|}}
\newcommand\I[1]{{\bf 1}_{#1}}
\newcommand\floor[1]{\lfloor #1 \rfloor}

\newcommand\modn[1]{#1 \bmod n}
\newcommand{\caglad}[1]{c\`agl\`ad#1}
\newcommand{\cadlag}[1]{c\`adl\`ag#1}
\newcommand{\mword}[1]{$\mathcal{M}$-word#1}
\newcommand{\MWord}[1]{$\mathcal{M}$-Word#1}

\allowdisplaybreaks

\jmlrheading{1}{2000}{1-48}{4/00}{10/00}{Christopher Dance and Tomi Silander}{}

\ShortHeadings{Optimal Policies for Observing Time Series}{Dance and Silander}
\firstpageno{1}

\begin{document}

\title{Optimal Policies for Observing Time Series \\ and Related Restless Bandit Problems}

\author{\name Christopher R. Dance \email dance@xrce.xerox.com \\
       \name Tomi Silander \email silander@xrce.xerox.com \\
       \addr Xerox Research Centre Europe \\
       6 chemin de Maupertuis \\
       Meylan, 38240, France}


\maketitle

\begin{abstract}
The trade-off between the cost of acquiring and processing data, and uncertainty due to a lack of data is fundamental in machine learning.
A basic instance of this trade-off is the problem of deciding when to make noisy and costly observations of a discrete-time Gaussian random walk, 
so as to minimise the posterior variance plus observation costs.
We present the first proof that a simple policy, which observes when the posterior variance
exceeds a threshold, is optimal for this problem.
The proof generalises to a wide range of cost functions other than the posterior variance.

This result implies that optimal policies for 
\textit{linear-quadratic-Gaussian control with costly observations}
have a threshold structure.
It also implies that the restless bandit problem of observing multiple such time series,
has a well-defined \textit{Whittle index.}
We discuss computation of that index, give closed-form formulae for it,
and compare the performance of the associated index policy with heuristic policies.

The proof is based on a new verification theorem 
that demonstrates threshold structure for Markov decision processes,
and on the relation between binary sequences known as \textit{mechanical words} 
and the dynamics of discontinuous nonlinear maps, which frequently arise in physics, control and biology.
\end{abstract}
%
%
\section{Introduction}
\label{section:introduction}
This paper answers three closely-related questions about discrete-time filtering of
scalar time series with costly observations, where the nature of the observations is controlled through a query action. 
The first two questions concern the structure of 
optimal policies for observing a single time series 
so as to minimise either a function of the posterior variance 
(Theorem~\ref{theorem:ThresholdOptimality})
or a quadratic function of the system state and control input 
(Corollary~\ref{corr:lqg}). 
The third question concerns the observation of several such time series with a 
constraint on the number of time series that can be observed simultaneously.
This is an instance of a \textit{restless bandit problem} and it is interesting to know
that the problem has a well-defined \textit{Whittle index} 
(Theorem~\ref{Indexability}).

This introduction begins with the time-series model (Section~1.1) that the three questions have in common.
It then motivates, formulates and states the key results for each question in 
turn (Sections~1.2 to~1.4).
It concludes with an intuitive guide to the main concepts involved in the proof (Section~1.5)
and a description of the structure of the rest of the paper (Section~1.6).

\subsection{Time-Series Model}
We consider the classic discrete-time scalar normally-distributed state-space model. In this model, the state is partially observed through measurements as fully described by the conditional dependencies 
\begin{align}
\label{eq:system}
\left.
\begin{aligned}
X_0 &\sim \mathcal{N}(x_0,v_0) \\
X_{t+1}|X_t,u_t &\sim \mathcal{N}(A X_t + B u_t, \Sigma_X) \\
Y_{t+1}|X_{t+1},a_t &\sim \mathcal{N}(X_{t+1},\Sigma_Y(a_t))
\end{aligned}
\right\} \quad\text{for $t\in\Z_+$.}
\end{align}
(In this paper $\Z_+,\R_+$ include zero, unlike $\N, \R_{++}$.)
The state 
$X_t$ is a real-valued random variable with initial mean $x_0$ and variance $v_0$.
The sequence of states depends on the control or exogenous input $u_t\in\R$.
The measurement $Y_{t+1}$ is a real-valued random variable 
which depends on a query action $a_t\in\{0,1\}$.
The variances $\Sigma_X, \Sigma_Y(0), \Sigma_Y(1) > 0$ 
and real-valued parameters $A,B$ are known.
Query action $a_t = 1$ is assumed to correspond to a higher-quality
observation than query action $a_t = 0$, so that $\Sigma_Y(1) < \Sigma_Y(0)$
and it is possible that $\Sigma_Y(0) = \infty$ which represents
a totally uniformative observation or no observation at all.

The observed history $H_t$ at time $t$, is 
$x_0, v_0, a_0, a_1, \dots, a_{t-1}, u_0, u_1, \dots, u_{t-1}, Y_1, Y_2, \dots, Y_t$.
Under the Bayesian filter, the information state is given by the posterior mean
$x_t := \E[X_t | H_t]$ and variance $v_t := \E[(X_t-x_t)^2|H_t]$.
In this case, the Bayesian filter is the Kalman filter~\citep{Thiele1880,Kalman60} 
and it follows
that the information state undergoes the following Markovian transitions: 
\begin{align}
\label{eq:transitions}
\left. \begin{aligned}
x_{t+1} | x_t, v_t, a_t, u_t &\sim \mathcal{N}(Ax_t+Bu_t,A^2v_t + \Sigma_X - \phi_{a_t}(v_t)) \\
v_{t+1} | x_t, v_t, a_t, u_t &= \phi_{a_t}(v_t) 
\end{aligned} \right\} \quad\text{for $t\in\Z_+$}
\end{align}
where $\phi_{a} : \R_+ \rightarrow \R_+$ for $a\in\{0,1\}$ is the M{\"o}bius transformation
\begin{align}
\label{eq:phidef}
\phi_{a}(v) := 
 \frac{(A^2 v + \Sigma_X) \times \Sigma_Y(a)}{(A^2 v + \Sigma_X) + \Sigma_Y(a)}.
\end{align}

\subsection{Optimal Policies for Observing a Single Time Series}
The simplest problem addressed here involves a  
measurement cost $c(a_t) \in \R$ and uncertainty cost $C(v_t)$.
Cost $c(a_t)$ might reflect costs of energy, labour, communication, 
computational processing, hardware or risks associated with each measurement.
Recall that a policy is \textit{non-anticipative} if it selects actions at time $t$ based only
on information available up-to and including time $t$.
The objective is to find a non-anticipative policy $\pi$ that selects query actions $a_t$ 
so as to minimise the $\beta$-discounted performance functional, for $\beta\in [0,1)$,
\begin{align*}
\E\left( \sum_{t=0}^\infty \beta^{t} (c(a_t) + C(v_t))\ \bigg| \ \pi, x_0, v_0 \right) 
\end{align*}
where the expectation is over the Markovian transitions~(\ref{eq:transitions}).
As the transitions of the posterior variance are given by the M{\"o}bius
transformation~(\ref{eq:phidef}), this problem reduces to the following deterministic dynamic program 
for value function $V : \R_+ \rightarrow \R_+$,
\begin{align}
\label{eq:DP1}
V(v_t) = \min_{a_t\in\{0,1\}} \bigg\{ c(a_t) + C(v_t) + \beta V(\phi_{a_t}(v_t)) \bigg\} .
\end{align}

The first question addressed in this paper is: \textit{for what cost functions
is a threshold policy optimal for this problem?}
For instance, one may intuively guess that optimal policies for variance 
minimisation with $C(v)=v$, for entropy minimisation with $C(v) = \log(v)$,
or for precision maximisation with $C(v) = -1/v$,
might involve making expensive observations at time $t$ 
when the variance $v_t$ exceeds a threshold.
The following condition on $C(\cdot)$ covers these examples.
\\

\noindent\textbf{Condition~C.} \textit{
The state space $\IR$ is either $[0,\infty)$ or $(0,\infty)$. 
The first cost function $c : \{0,1\} \rightarrow \R$, representing data acquisition costs, has $c(0) < c(1)$.
Also, the second cost function $C : \IR \rightarrow \R$, representing the cost of uncertainty, 
is of the form $C(x) = \sum_{i=1}^{n_C} C_i(x)$ 
for some $n_C\in\N$, 
where each of the functions $C_i : \IR \rightarrow \R$ satisfies one of the following conditions:
\begin{enumerate}
\item[C1.] For $x\in\IR$, 
the derivatives $C_i'(x) := \frac{d}{dx} C_i(x)$ and 
$C_i''(x) := \frac{d^2}{dx^2} C_i(x)$ exist and
\begin{itemize}
\item the function $C_i(x)$ is concave, 
\item the function $\frac{1}{x^3}C_i''\left(\frac{1}{x}\right)$ is non-decreasing,
\item and the function $\frac{1}{x^2} C_i'\left(\frac{1}{x}\right)$ is non-increasing and convex.
\end{itemize}
\item[C2.] For $x\in\IR$, the function $C_i(x)$ is non-decreasing, convex and differentiable.
\end{enumerate}}

Note that we may work with the interval $\IR=(0,\infty)$ in cases where
the cost function $C(v)$ is not a real number for $v=0$, in order to include cases
like $\log(v)$ and $-1/v$, but the results of the paper continue to hold
in cases where $C(0)$ is defined. 
Also, note that $C(\cdot)$ need not be convex or concave, for instance in the case $C(v) = (v^2-1)/v$, and it is possible that $C(\cdot)$ is bounded, for instance in in the case $C(v) = v/(v+1).$ 

The above condition requires that functions $C_i$ satisfying C2 have a derivative $C_i'$.
This is simply for convenience in the proofs of Propositions~\ref{proposition:increasing}
and~\ref{proposition:continuous}.
As such functions are real-valued convex functions, one can instead
set $C_i'$ equal to any subderivative 
at points where the derivative is not defined.

\begin{theorem}
\label{theorem:ThresholdOptimality}
Suppose the state space $\IR$ and the cost functions $c : \{0,1\} \rightarrow\R$ and $ C : \IR \rightarrow \R$ satisfy Condition~C.
Then a threshold policy is optimal for the dynamic program~(\ref{eq:DP1}).
\end{theorem}

\begin{proof}
This result is an immediate consequence of Theorem~\ref{theorem:verification}
whose hypotheses hold according to 
Propositions~\ref{proposition:PCLI0},~\ref{proposition:POS},~\ref{proposition:increasing},~\ref{proposition:easy-increasing},~\ref{proposition:continuous} and~\ref{proposition:PCLI3}. 
\end{proof}

From one perspective, this answer is a rare example of an explicit solution to a real-state partially-observed Markov decision process (POMDP).
From another perspective, this answer is a rare example of an explicit solution to the problem of observation selection in sensor management~\citep{Hero11}.
Indeed, given a collection of variables which can (in principle) be observed
and a single variable to predict, which are jointly Gaussian with known covariance,
even the problem of deciding whether there exists a subset of $k$ observations
that reduces the prediction variance below a given threshold is NP-hard~\citep{Davis97}.
Work has therefore focused on finding covariance structures
for which the problem is tractable, for instance~\citet{Das08} show that
selection of Gaussian observations with an exponential covariance can be solved 
by a simple discrete dynamic program, and on finding appropriate 
choices of cost functions for which there are guaranteed approximation algorithms~\citep{Krause08,Krause11,BadanidiyuruKDD14,ChenICML14}.

\subsection{The Linear Quadratic Gaussian Problem with Costly Observations}
The second question addressed by this paper is: 
\textit{when are threshold policies optimal for making observations 
in a generalisation of the linear-quadratic-Gaussian 
control problem in which observations are costly but controlled through
a query action?}
This is an old but unsolved problem~\citep{Meier67,Wu05,Molin09}.
Specifically, suppose the states and observations are as in~(\ref{eq:system})
but the objective is to find a non-anticipative policy $\pi$ that selects 
a feedback-control action $u_t\in\R$
and a sensor-query action $a_t\in\{0,1\}$ 
so as to minimise the $\beta$-discounted performance functional
\begin{align*}
\E\left( \sum_{t=0}^\infty \beta^{t} (DX_t^2 + Fu_t^2 + c(a_t))\ \bigg| \ \pi, x_0, v_0 \right), 
\end{align*}
where $D, F \in\R_+$ and the expectation is over the Markovian transitions~(\ref{eq:transitions}).

An immediate corollary of Theorem~\ref{theorem:ThresholdOptimality} is the following answer to the above question.
\begin{corollary}
\label{corr:lqg}
Suppose that
$A \in [-1,1]$, 
$D \in\R_+$, 
$F\in\R_+$, 
$\beta\in (0,1)$, 
$\Sigma_Y(q) \in [0,\infty]$ for $q\in \{0,1\}$ with $\Sigma_Y(0)\ge\Sigma_Y(1)$,
$c(q) \in \R$ for $q\in \{0,1\}$ with $c(0)\le c(1)$, 
where $A$ and $\Sigma_Y(\cdot)$ are as in equation~(\ref{eq:system}).
Then an optimal policy for linear-quadratic-Gaussian control with costly observations is to set 
\begin{align*}
a_t &= \begin{cases} 1 & \text{if $v_t \ge z$} \\ 0& \text{if $v_t < z$} \end{cases} 
\quad \text{and} \quad u_t = - L x_t
\end{align*}
for some $L \in \R$ and $z \in [0,\infty]$.
\end{corollary}
A proof of Corollary~\ref{corr:lqg} is presented in Appendix~D.

\subsection{Multi-Target Tracking and Restless Bandits}
This paper also addresses the problem of monitoring \textit{multiple} time
series so as to maintain a precise belief while minimising the cost of sensing.
This problem is often called the \textit{multi-target tracking} problem.

To formulate the problem, suppose there are $n\in\N$ independent time series of the form~(\ref{eq:system}), 
indexed by $i \in \{1, 2, \dots, n\}$,
and time series $i$ has state $X_{i,t}$ at time $t\in\Z_+$.
Each time series may have 
its own parameters $x_{i,0}, v_{i,0}, A_{i,0}, B_{i,0}, \Sigma_{X_i}$,
its own input $u_{i,t}$
and its own uncertainty cost $C_i : \IR \rightarrow \R$,
where the interval $\IR$ is as in Condition~C.
Corresponding to these time series there are $n$ query actions $a_{i,t} \in \{0,1\}$ 
which specify the nature of the observation $Y_{i,t}$ of time series $i$. 
These observations have their own parameters $\Sigma_{Y_i} : \{0,1\} \rightarrow [0,\infty]$
and costs $c_i : \{0,1\} \rightarrow \R$. 
However, these actions are subject to the constraint that only $m\in\N$ 
with $m<n$ expensive observations can be made at each time.

The problem is to minimise the total $\beta$-discounted observation cost
and uncertainty cost
\begin{align*}
\sum_{i=1}^n \sum_{t=0}^\infty \beta^{t} (c_i(a_{i,t}) + C_i(v_{i,t})) 
\end{align*}
subject to the constraint on the number of observations
\begin{align*}
\sum_{i=1}^n a_{i,t} = m \qquad \text{for $t\in\Z_+$}.
\end{align*}

Continuous-time versions of this problem were previously addressed by \citet{LeNy11}
and versions of the discrete-time problem given above have attracted considerable
attention~\citep{Gupta06,Mourikis06,LaScala06,Washburn08,NinoMora09,Villar12,Zhao14,Dance15,NinoMora16}.
One example of a real-world application of the discrete-time problem, which was our original motivation for studying the problems in this paper, is the measurement of on-street parking occupancy~\citep{Dey14}, in a setting where cheap-but-low-quality observations are available through payment data (at parking meters or through mobile phones), expensive-but-high-quality observations are available through portable cameras, which are moved daily or weekly (and thus in discrete time), and there are a limited number of portable cameras with which to observe many streets.

\paragraph{Restless Bandits.}
The multi-target tracking problem is an instance of a \textit{restless bandit problem}~\citep{Whittle88}.
Typically, such problems are defined in terms of a set of $n\in\N$ two-action Markov decision processes (MDPs), although generalisations to a time-varying number of MDPs~\citep{Verloop16} and to more than two actions per MDP~\citep{Glazebrook11} have been explored.
The two actions are usually referred to as \textit{active} or \textit{play} versus \textit{inactive} or \textit{passive} and each of the MDPs is referred to as an \textit{arm} or \textit{project.}

In a restless bandit problem, these $n$ MDPs are coupled into a single MDP as follows.
The state space is the Cartesian product of the state spaces of the arms, and the state of each arm transitions independently of the other arms given the actions taken on that arm. 
Thus the transitions of an arm depend only on the actions taken on that arm and on that arm's current state.
The objective is to find a non-anticipative policy that minimises the sum over the arms of each arm's cost-to-go.

However, the action space is only a subset of the Cartesian product of the action spaces of the arms, as there is a constraint on the number $m$ of arms that are simultaneously active at each time, where $m\in\N$ with $m<n$.
Typically, the constraint is that \textit{exactly} $m$ arms are active at each time, but this is readily relaxed to a constraint that \textit{at most} $m$ arms are active by including ``dummy arms'', whose cost is always zero, in the population of $n$ arms.
More general constraints have been explored~\citep{NinoMora15}, in which each arm consumes resources as a function of both its state and the action taken, and the total cost of the resources consumed at each time is constrained. 
In the absence of any such action constraint, the problem would be solved by applying an optimal policy for each arm independently.
Moreover, it turns out that if the constraint were only on the (discounted) time-average number of arms that are simultaneously active, rather than a constraint at each time, the problem could again be separated into $n$ smaller problems after introducing a Lagrange multiplier.

Let us relate the above definition to the typical usage of the term \textit{bandit} in the machine-learning literature.
In that context, multi-armed bandits are reinforcement-learning problems involving a set of arms whose reward distributions are unknown.
At each time, the learner must select which arm to play. 
Such bandits involve a trade-off between exploring arms to acquire information about their expected payoffs and exploiting arms with the highest expected payoffs. 
In the simplest versions of such problems, where the prior on the reward distributions is independent over arms, each arm can be viewed as an MDP whose state correponds to the belief about that arm's payoff distribution.
Each time the arm is played, its reward is observed and this belief is updated.
Such updates correspond to state transitions.
Each time the arm is inactive, its state does not change. 

If we allow arms to make general Markovian state transitions, not just transitions corresponding to belief updates, while preserving the requirement that an arm only changes state when it is played, then we arrive at a more general class of problems known as \textit{ordinary} or \textit{classical bandits}~\citep{Gittins11}.
In turn, restless bandits generalise ordinary bandits in two ways.
Firstly, restless bandits allow more than one arm to be simultaneously active (if $m>1$).
Secondly, restless bandits allow the state of an arm to change even when the arm is not active, which is why they are called \textit{restless}.

While this additional generality is important in modelling real-world problems, it comes at a price.
On the one hand, the \textit{Gittins index policy} is optimal for ordinary bandit problems and can be computed in polynomial time for problems with finite state spaces~\citep{NinoMora07}.
On the other hand, it is in general PSPACE-hard~\citep{Papadimitriou99,Guha10} to find policies that approximate optimal policies for restless bandit problems with finite state spaces to any non-trivial factor.
At first glance, this might suggest that the multi-target tracking problem addressed here, with uncountable state-space $\R_+$ or $\R_{++}$, is impossibly difficult.
At second glance, this poses an interesting question: for which restless bandit problems can we find approximately-optimal policies efficiently?

\paragraph{Whittle Index Policy.}
\citet{Whittle88} proposed a policy which generalises the Gittins index policy to restless bandit problems.
This policy associates a real (or in some definitions an extended-real) number $\lambda_i^*(x_i)$ called the \textit{Whittle index} with the state $x_i$ of each arm $i$ and plays the $m$ arms with the largest Whittle index at each time. 
Ties are usually broken uniformly at random or according to a predefined priority ordering.

The literature contains many definitions of the Whittle index $\lambda^*_i(x_i)$ of arm $i$, of which we describe only three.
These definitions are not equivalent in general, although they turn out to be equivalent for the problem addressed in this paper.
All the definitions involve a modified version of arm $i$'s MDP, which we call the \textit{$\lambda$-MDP}, in which the cost $C_i(x_i,a_i)$ for taking action $a_i$ in state $x_i$ is replaced by $C_i(x_i,a_i) + \lambda a_i$ where $\lambda\in\R$ represents a price for taking the active action $a_i = 1$. 
\citet{Verloop16} then defines $\lambda_i^*(x_i)$ as the least price $\lambda$ for which action $a_i = 0$ is optimal for the $\lambda$-MDP in state $x_i$.
Meanwhile, \citet{Guha10} define $\lambda_i^*(x_i)$ as the largest price $\lambda$ for which the actions $a_i=0$ and $a_i=1$ are both optimal for the $\lambda$-MDP in state $x_i$.
In this paper, we use the following definition from~\citet{NinoMora15}. 
\begin{definition}
The {\bf Whittle index} of arm $i$ in state $x_i$ is a price $\lambda^*_i(x_i)$ for which 
\begin{enumerate}
\item Action $a_i=1$ is optimal in state $x_i$ of the $\lambda$-MDP if and only if $\lambda \le \lambda_i^*(x_i)$,
\item Action $a_i=0$ is optimal in state $x_i$ of the $\lambda$-MDP if and only if $\lambda \ge \lambda_i^*(x_i)$.
\end{enumerate}
Arm $i$ is {\bf indexable} if the Whittle index $\lambda^*_i(x_i)$ exists for all states $x_i$ in arm $i$'s state space.
\end{definition}
For all of the above definitions, it is immediate that the Whittle index is unique if it exists.
Verloop's definition has the advantage that the Whittle index, and hence the Whittle index policy, exist for a wider range of arms.
On the other hand, if we know arm $i$ is indexable, the definition used in this paper has the advantage that we know we have found the Whittle index when we find a price $\lambda\in\R$ for which actions $a_i=0$ and $a_i=1$ are both optimal in state $x_i$ of the $\lambda$-MDP.

Whittle's index policy has been the subject of great interest for computational, empirical and theoretical reasons.
The policy is potentially attractive in terms of computational cost as it reduces the original restless bandit problem, whose state space is the Cartesian product of the state spaces of the arms, to the computation of $n$ Whittle indexes for individual arms.
The policy is also attractive from a systems-architecture point-of-view, as it allows one to mix-and-match different types of arms, and it naturally accomodates the arrival or departure of arms in the sense that the Whittle index does not depend on the number of arms $n$.
Additionally, extensive numerical tests of Whittle's policy in different applications repeatedly demonstrate that it performs remarkably well when the arms are all indexable.
Indeed, 12 references to such empirical work are cited in Section~8 of~\citet{Verloop16}.

Although Whittle's policy is not an optimal policy for general restless bandits, under certain sufficient conditions and for a certain limit, it is an \textit{asymptotically} optimal policy.
Specifically, in the limit as the number of arms $n$ tends to infinity, while the number of arms that can be simultaneously active $m$ varies in such a way that $m/n$ is as constant as possible, the ratio of the cost-rate of Whittle's policy to the cost-rate of an optimal policy for the given $n,m$ tends to one.
Assuming an average-cost setting, for collections of identical arms whose size $n$ does not vary with time, where each arm has a finite state space, \citet{Whittle88} originally conjectured that it was sufficient that the identical arm was indexable for such an asymptotic optimality result to hold.
However, \citet{Weber90} found counterexamples to this conjecture.
Nevertheless, Weber and Weiss also found sufficient conditions for asymptotic optimality to hold, and those sufficient conditions imply that the arms are indexable (Lemma~2 of that paper).
Under similar sufficient conditions, this asymptotic optimality result has recently been generalised by~\citet{Verloop16} to restless bandits with dynamic populations of non-identical arms.
Both the results of Weber and Weiss and the results of Verloop assume an average-cost setting and arms with finite state spaces. 
So new theoretical work may be required to understand asymptotic optimality for arms with uncountable state spaces, as studied here.

\paragraph{Whittle Index for the Multi-Target Tracking Problem.}
The above discussion prompts the third question addressed in this paper:
\textit{is the multi-target tracking problem indexable, and if so, what is a computationally-convenient expression for its Whittle index?}
The answer is given by the following Theorem.

\begin{theorem}
\label{Indexability}
Suppose the state space $\IR$ and the cost functions $c : \{0,1\} \rightarrow\R$ and $C : \IR \rightarrow \R$ satisfy Condition C.
Let the price $\nu\in\R$, the discount factor $\beta\in [0,1)$ and let the transitions $\phi_a : \IR \rightarrow \IR$ for $a\in \{0,1\}$ be given by~(\ref{eq:phidef}). 
Then the family of dynamic programs
\begin{align*}
V(x;\nu) = \min_{a\in\{0,1\}} \left\{ \nu c(a) + C(x) + \beta V(\phi_a(x);\nu) \right\}
\end{align*}
is indexable.
Furthermore, for each $x\in\IR$ the Whittle index is
\begin{align*}
\lambda(x) &:= \frac{\sum_{t=0}^\infty \beta^t (C(X_t(x,0;x))-C(X_t(x,1;x)))}{
\sum_{t=0}^\infty \beta^t (c(A_t(x,1;x))-c(A_t(x,0;x)))}
\end{align*}
where for any $s\in [-\infty,\infty]$ we define $X_t(x,a;s)$ to be the state at time $t=0,1,\dots$ if the system starts in state $x\in \IR$
at time $t=0$, then action $a\in \{0,1\}$ is taken 
and a policy which takes the actions $A_t(x,a;s) := \I{X_t(x,a;s) \ge s}$ is followed thereafter 
(this is the $x$-threshold policy).
\end{theorem}

\begin{proof}
This result is an immediate consequence of Theorem~\ref{theorem:verification}
whose hypotheses hold according to Propositions~\ref{proposition:PCLI0},~\ref{proposition:POS},~\ref{proposition:increasing},~\ref{proposition:easy-increasing},~\ref{proposition:continuous} and~\ref{proposition:PCLI3}. 
\end{proof}

This paper thus generalises the work of~\citet{Dance15} by 
demonstrating that threshold policies are in fact optimal
for the single arm problem, which was Assumption~A1 of~\citep{Dance15}.
It also generalises by considering the case of multipliers $A<1$ rather than only considering $A=1$,
where $A$ is as in equation~(\ref{eq:system}), and by considering cost functions $C(v_t)\ne v_t$ other than the posterior variance.

\subsection{Intuitive Guide to the Paper}
As with other work on Markov decision processes, 
we work with the cost-to-go $Q(x,a)$ 
when starting in initial state $x$ and taking initial action $a$,
but then following an optimal policy.
A common way to prove that threshold policies are optimal 
when the state $x$ is real-valued, is to 
show that the difference $Q(x,1)-Q(x,0)$ is a non-increasing
function of $x$.
Such approaches have been studied by~\citet{Serfozo76},
~\citet{Altman95} and~\citet{Altman00}.
Unfortunately, as shown in Figure~\ref{fig:counter}, 
such an approach fails for the process considered in this paper,
even when the cost equals the variance.

\begin{figure}
\begin{center} 
\includegraphics[viewport=20mm 105mm 195mm 185mm,clip=true,width=0.7\textwidth]{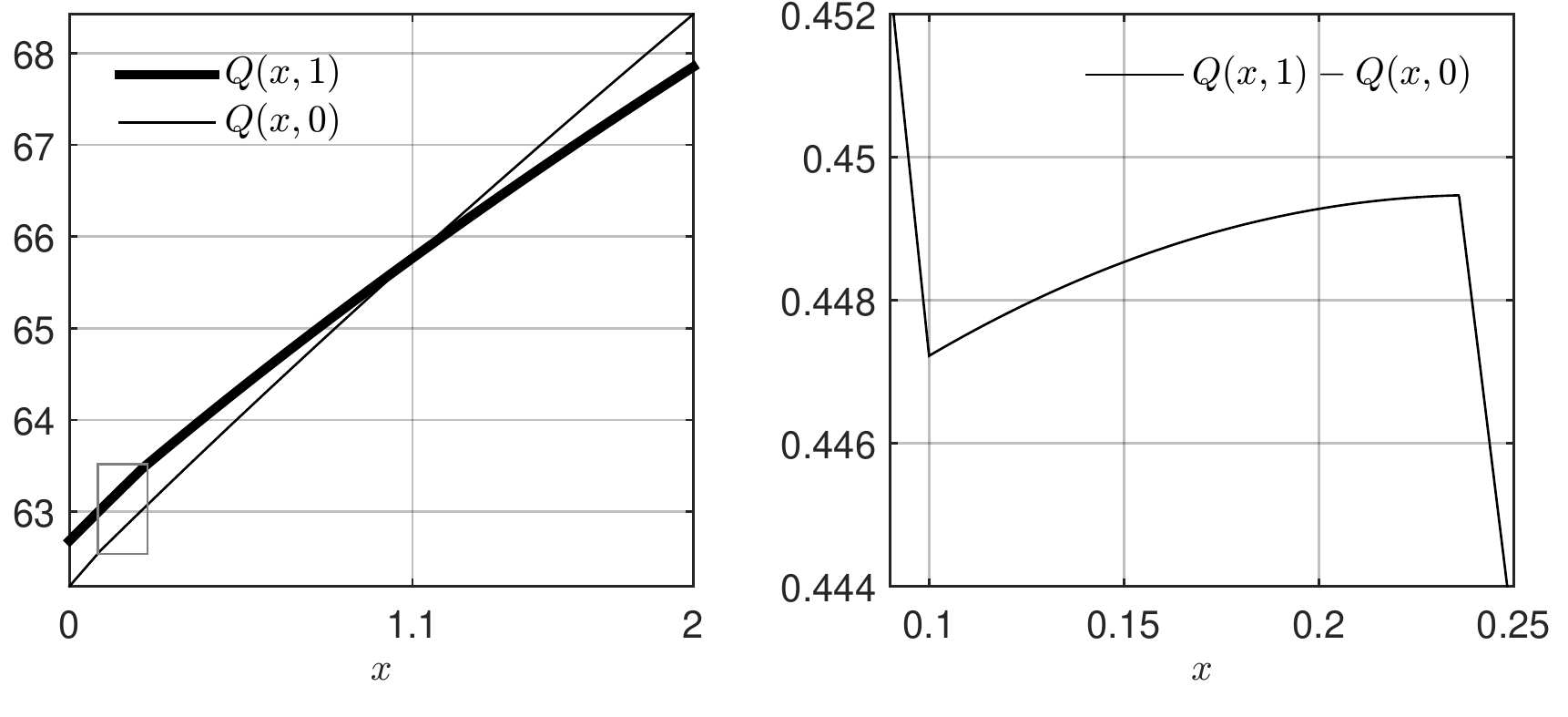}
\end{center}
\caption{Counterexample to monotonicity of the difference in $Q$-functions.
The functions $Q(x,0)$ and $Q(x,1)$ cross only a single time at $x=1.1$
(\textit{left plot}). 
However, the difference $Q(x,1)-Q(x,0)$ is increasing
for some $x$ (\textit{right plot}, for $x$ in the left plot's grey box).
The model has 
$\beta = 0.95, \ C(x)=x, \ \phi_0(x) = x+1, \ \phi_1(x) = 1/(a_1+1/(x+1))$ with $a_1 = 0.1$ and $\nu = 0.7647$.}
\label{fig:counter}
\end{figure}

Instead, this paper proves the optimality of threshold policies using a new verification 
theorem by~\citet{NinoMora15}.
This theorem applies to Markov decision processes that satisfy the so-called
\textit{partial conservation law indexability} (PCLI) conditions (Section~2).
The central concept underlying the verification theorem 
is the \textit{marginal productivity index}
which turns out to be equal to the ratio $\lambda$ given in Theorem~\ref{Indexability}.

One of the PCLI conditions requires 
that the marginal productivity index $\lambda(x)$ is 
a non-decreasing function of the state $x$. 
This is the most challenging of the conditions to verify.
As a quick check, we plot $\lambda(x)$ in Figure~\ref{fig:approximation}.
Although $\lambda(x)$ is increasing, the numerator and denominator have a 
fractal structure, so it is surprising that the index is continuous.
Furthermore, if we subtract a cubic fit to $\lambda(x)$, the
residual has a complicated sequence of cusps.
Therefore the paper then focusses on characterising the sequence of actions $A_t(x,a;x)$
that give rise to this fractal pattern.
We prove that these sequences are 
special binary strings called \textit{mechanical words} (Section~3).

\begin{figure}
\begin{center} 
\includegraphics[viewport=5mm 124mm 205mm 175mm,clip=true,width=\textwidth]{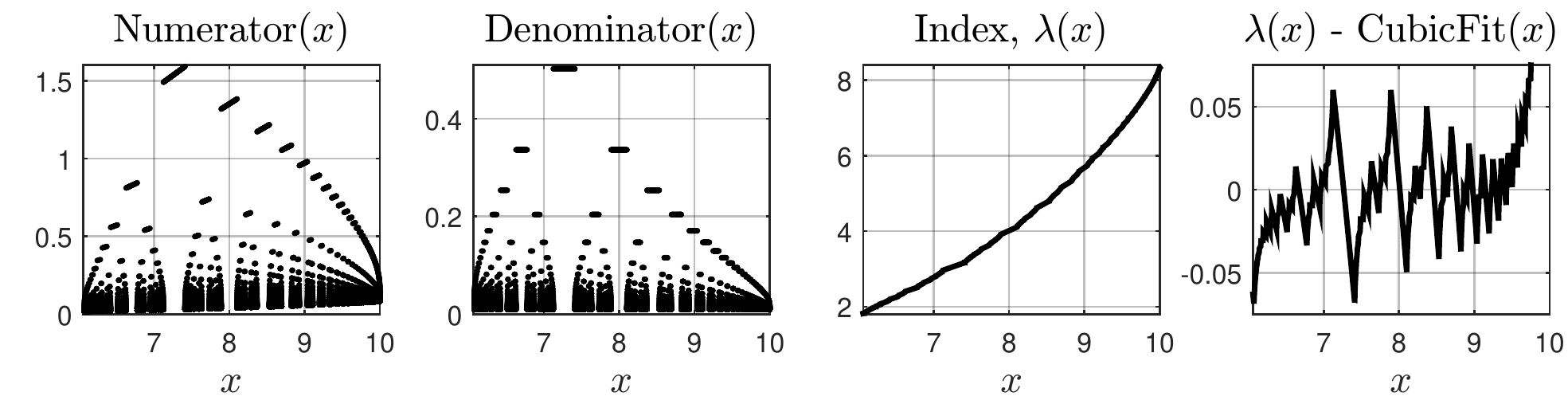}
\end{center}
\caption{The numerator (\textit{left}) and denominator (\textit{mid-left}) of the index (\textit{mid-right}), and the error in a cubic fit to the index (\textit{right}). 
The model has cost $C(x)=x$, 
discount factor $\beta = 0.99$,
map-with-a-gap $\phi_0(x) = rx+1$ and $\phi_1(x) = 1/(a_1 + 1/(rx+1))$
with $r = 0.9$ and $a_1 = 0.01$.}
\label{fig:approximation}
\end{figure}

Another key to proving that the index is non-decreasing is the fact that
the mappings $\phi_0(x)$ and $\phi_1(x)$ are M{\"o}bius transformations
of the form
\begin{align*}
\mu_A(x) := \frac{A_{11} x + A_{12}}{A_{21} x + A_{22}}
\end{align*}
for some $A\in \R^{2\times 2}$.
Now the composition of M{\"o}bius transformations is homeomorphic to 
matrix multiplication, so that 
\begin{align*}
\mu_A (\mu_B(x)) = \mu_{AB}(x) 
\end{align*}
for any $A,B\in\R^{2\times 2}$.
Further, if $\det(A)=1$ then the gradient of the corresponding M{\"o}bius
transformation is 
\begin{align*}
\frac{d}{dx} \mu_{A}(x) = \frac{1}{(A_{21} x + A_{22})^2} 
\end{align*}
which is a convex function for $x\in\R_+$ and $A_{21}, A_{22} \in \R_{++}$.
So the gradient of the numerator of the index, in the case $C(x)=x$, 
is the difference of the sums of a convex function of a sequence of
linear functions of $x$. 
Such sums can be addressed by the theory of \textit{majorisation}
\citep{Marshall10}, provided the sequence of linear functions of $x$
satisfy certain \textit{majorisation conditions}. 
It turns out that those conditions are satisfied because of a special
palindromic property of mechanical words.

\subsection{Stucture of the Paper}
First we give Ni{\~n}o-Mora's theorem about the optimality of threshold policies,
which is based on four partial conservation law indexability (PCLI) conditions
(Section~2). 
Then we relate the sequence of actions under threshold policies to \textit{mechanical words} (Section~3). 
We use the properties of those words to demonstrate that each PCLI condition holds.
These conditions concern bounded variation (Section~4), 
the positivity of so-called \textit{marginal work} (Section~5),
the non-decreasing nature of the marginal productivity index (Section~6),
the continuity of that index (Section~7)
and a condition that characterises the index
as a \textit{Radon-Nikodym derivative} (Section~8).

Having completed the proof, we then turn to 
closed-form expressions for the index and numerical methods
for evaluating it when such closed forms are not available (Section~9).
We demonstrate the accuracy of such numerical methods 
and show how the index varies as its parameters change.
Also, we compare the performance of Whittle's index policy with
other well-known heuristics (Section~9).

Finally, we discuss interesting avenues for further work (Section~10).
The appendices contain detailed proofs about the relation of itineraries to mechanical words (Appendix~A),
of a key majorisation inequality (Appendix~B), 
about the linear systems orbits to which this majorisation result is applied (Appendix~C)
and of the optimality of threshold policies for the LQG problem with costly observations (Appendix~D).

%
%
\section{Verification Theorem for the Optimality of Threshold Policies}
\label{section:verification}
We present a theorem which guarantees
the optimality of threshold policies for two-action Markov decision problems
under certain hypotheses.
This is a special case of a theorem due to \citet{NinoMora15} 
which extends previous work on countable state spaces \citep{NinoMora01,NinoMora06} to 
problems where the state space is an interval of the real line.
Ni{\~n}o-Mora calls the theorem's hypotheses the 
\textit{partial conservation law indexability (PCLI)} conditions.
This terminology was chosen to contrast with 
the \textit{strong} conservation law conditions of~\citet{Shanthikumar92} 
and the \textit{generalised} conservation law conditions of~\citet{Bertsimas96}, 
which have their roots in \textit{conservation laws} for queueing systems under 
which waiting times are invariant to the queueing discipline~\citep{Kleinrock65}.

The theorem relates to a family of dynamic programming equations with
a single parameter $\nu \in \R$, which might be interpreted as a wage,
a tax, the cost of activating a sensor or a Lagrange multiplier. 
For each $\nu\in\R$, we consider the following simple 
dynamic program for value function $V(\cdot;\nu) : \IR \rightarrow \R$, where the state space $\IR$ is an interval of $\R$:
\begin{align}
\label{eq:simpleDP}
V(x;\nu) = \min_{a\in \{0,1\}} \left\{ \nu c(x,a) + C(x,a) + \beta V(\phi_a(x); \nu)  \right\}
\end{align}
where $c : \IR \times \{0,1\} \rightarrow \R$ is called the \textit{work},
$C : \IR \times \{0,1\} \rightarrow \R$ is called the \textit{cost}, 
the discount factor is $\beta \in [0,1)$ and 
the state transitions are given by $\phi_a : \IR \rightarrow \IR$ for each action $a\in \{0,1\}$.
We discuss generalisations of this dynamic program after the proof of the verification theorem.

To state the PCLI conditions, we first recall the well-known definitions of \cadlag, \caglad{ }and bounded-variation functions.
Let $\mathcal{J}\subseteq\R$ and let $(\mathcal{M},d)$ be a metric space. 
A function $f : \mathcal{J}\rightarrow \mathcal{M}$ is a \textit{\caglad{ }}function 
if both the left limit $f(x^-):=\lim_{u\uparrow x} f(u)$ and the right limit $f(x^+) := \lim_{u\downarrow x} f(u)$ exist and $f(x^-)=f(x)$, for all $x\in \mathcal{J}$. 
A function $f : \mathcal{J} \rightarrow \mathcal{M}$ is \textit{\cadlag{ }}if both of the limits $f(x^-), f(x^+)$ exist and $f(x^+)=f(x)$, for all $x\in \mathcal{J}$.
(C\`adl\`ag is an abbreviation of the French description 
\textit{``continue \`a droite, limite \`a gauche''}, which means ``right continuous, left limit''.)

Let $\IR$ be an interval of $\R$.
A \textit{partial subdivision} of $\IR$ is a collection 
$\{\mathcal{I}_1, \mathcal{I}_2, \dots, \mathcal{I}_n\}$ of closed intervals of $\IR$, 
where $n\in\N$,
such that the set $\mathcal{I}_j\cap \mathcal{I}_k$ is either empty 
or consists of a single point that is an endpoint of both $\mathcal{I}_j$ and $\mathcal{I}_k$,
for all $1 \le j < k \le n$.
Let $\mathcal{S}$ be the set of partial subdivisions of $\IR$. 
A function $f : \IR \rightarrow \R$ has \textit{bounded variation} if
\begin{align*}
\sup_{\{ [a_1,b_1], [a_2, b_2], \dots, [a_n, b_n]\} \in\mathcal{S}} \sum_{i=1}^n \abs{f(b_i)-f(a_i)} < \infty.
\end{align*}

Now we need some definitions concerning $s$-threshold policies.
For state $x\in\IR$, 
initial action $a\in \{0,1\}$ 
and threshold $s\in\bR = [-\infty,\infty]$, 
let $X_t(x,a;s)$ and $A_t(x,a;s)$ be the state and action at time $t\in\Z_+$ 
when the initial state is $X_0(x,a;s) = x$,
the initial action is $A_0(x,a;s) = a$ 
and the \textit{$s$-threshold policy} is followed for $t\in\N$,
which is the policy that takes action $A_t(x,a;s)=1$ if and only $X_t(x,a;s)\ge s$.
Also, for $t\in\Z_+$, let $X_t(x;s) := X_t(x,\I{x\ge s};s)$ and $A_t(x;s) := A_t(x,\I{x\ge s};s)$
denote the state and action when all actions are taken according to the $s$-threshold policy.
The \textit{cost-to-go} $f(x,a;s)$ and the \textit{work-to-go} $g(x,a;s)$ are
\begin{align*}
f(x,a;s) &:= \sum_{t=0}^\infty \beta^t C(X_t(x,a;s),A_t(x,a;s)), &
g(x,a;s) &:= \sum_{t=0}^\infty \beta^t c(X_t(x,a;s),A_t(x,a;s)).
\end{align*}
We also define $f(x;s) := f(x,\I{x\ge s};s)$ and $g(x;s) := g(x,\I{x\ge s};s)$
as the cost-to-go and work-to-go when the first action is taken according to the $s$-threshold policy.

We are now ready to state the partial conservation law indexability (PCLI) conditions.\begin{definition} 
\label{def:marginals}
Consider the family of dynamic programs~(\ref{eq:simpleDP}).
For state $x\in\IR$, 
action $a\in \{0,1\}$ 
and threshold $s\in\bR$, 
the 
\textbf{marginal cost} $c_x(s)$ and the \textbf{marginal work} $w_x(s)$ are
\begin{align*}
c_x(s) &:= f(x,0;s)-f(x,1;s), &
w_x(s) &:= g(x,1;s)-g(x,0;s), 
\end{align*}
and the \textbf{marginal productivity index} $\lambda(x)$ is
\begin{align*}
\lambda(x) &:= c_x(x)/w_x(x).
\end{align*} 
Family~(\ref{eq:simpleDP}) is \textbf{partial conservation law indexable} if
for all $x\in \IR$ and all $s\in \bR$:
\begin{enumerate}
\setlength{\itemindent}{.7cm}
\item[PCLI0.] The marginal work $w_x(s)$ is a \caglad{ }function of $s$ with bounded variation.
\item[PCLI1.] The marginal work $w_x(s)$ is positive.
\item[PCLI2.] The marginal productivity index $\lambda(x)$ is non-decreasing and continuous.
\item[PCLI3.] The marginal cost satisfies $c_x(b)-c_x(a) = \int_{[a,b)} \lambda \, dw_x$ for all $[a,b) \subseteq \IR$. 
\end{enumerate}
\end{definition}

\citet{NinoMora15} uses three rather than four conditions,
as he derives an equivalent to our condition PCLI0 from additional assumptions
and with a different definition of the $s$-threshold policy resulting in
the marginal work $w_x(s)$ being a \cadlag{ }rather than \caglad{ }function of $s$.
Condition PCLI3 requires that $\lambda$ is a Radon-Nikodym derivative 
of the signed Lebesgue-Stieltjes measure~\citep{Carter00} corresponding to the marginal cost $c_x(\cdot)$
with respect to the marginal work $w_x(\cdot)$.
In analyses of discrete-state Markov decision processes, condition PCLI3
is not required as it is implied by PCLI1 and PCLI2~\citep{NinoMora15}.

\begin{theorem}
\label{theorem:verification}
Suppose PCLI0-PCLI3 hold for the family~(\ref{eq:simpleDP}). Let $x\in\IR$ and $\nu\in\R$.
\begin{enumerate}
\item If $\lambda(s) = \nu$ for some $s\in\IR$, then the $s$-threshold policy is an optimal policy
and actions 0 and 1 are both optimal in state $x$ if and only if $\lambda(x)=\nu$.
\item If $\lambda(s) > \nu$ for all $s\in\IR$, then the always-active policy is the unique optimal policy. 
\item If $\lambda(s)<\nu$ for all $s\in\IR$, then the always-passive policy is the unique optimal policy.
\item The family is indexable and its Whittle index is the marginal productivity index $\lambda$. 
\end{enumerate}
\end{theorem}

\begin{proof}
Let $Q(x,a;s,\nu)$ be the total-cost-to-go under the $s$-threshold policy
from initial state $x$ when the initial action is $a$, so that 
\begin{align*}
Q(x,a;s,\nu) &= f(x,a;s)+\nu g(x,a;s).
\end{align*}

\noindent\textit{Claim 1.} Suppose that $s\in\IR$ with $\lambda(s)=\nu$.
We shall show that 
\begin{align*}
Q(x,0;s,\nu) &\le Q(x,1;s,\nu) \quad \text{if $x\le s$}\\
Q(x,1;s,\nu) &\le Q(x,0;s,\nu) \quad \text{if $x\ge s$}
\end{align*}
and these inequalities are strict if and only if $\lambda(x) \ne \lambda(s)$.
Thus the value function of the $s$-threshold policy, given by
\begin{align*}
V(x;s,\nu) := \begin{cases}
	Q(x,0;s,\nu) & \text{if $x\le s$} \\
	Q(x,1;s,\nu) & \text{if $x\ge s$} \end{cases}
\end{align*}
satisfies the dynamic program (\ref{eq:simpleDP}). 
Therefore the $s$-threshold policy is optimal
and actions 0 and 1 are both optimal if and only if $\lambda(x)=\nu$.

Say $x\le s$.
Noting that $w_x$ has bounded variation (by PCLI0) 
and $\lambda$ is a continuous function (by PCLI2),
Lebesgue-Stieltjes integration-by-parts~\citep{Carter00} gives
\begin{align*}
\int_{[x,s)} \lambda \, dw_x + \int_{[x,s)} w_x \, d\lambda 
&= \lambda(s^-) w_x(s^-) - \lambda(x^-) w_x(x^-) .
\end{align*}
Now the second integral is non-negative,
as $\lambda$ is non-decreasing (by PCLI2)
and the integrand $w_x$ is non-negative (by PCLI1).
Also, 
$\lambda(s^-) w_x(s^-) = \lambda(s) w_x(s)$
and $\lambda(x^-) w_x(x^-) = \lambda(x) w_x(x)$,
as $\lambda$ is continuous (by PCLI2) and $w_x$ is a \caglad{ }function (by PCLI0).
Therefore
\begin{align*}
\int_{[x,s)} \lambda \, dw_x &\le \lambda(s) w_x(s) - \lambda(x) w_x(x) .
\end{align*}
Using PCLI3, it follows that the marginal cost is bounded by
\begin{align*}
c_x(s) = c_x(x) + \int_{[x,s)} \lambda \, dw_x 
\le c_x(x) + \lambda(s) w_x(s) - \lambda(x) w_x(x) = \lambda(s) w_x(s)
\end{align*}
where we cancelled two terms as the definition of $\lambda$ gives $c_x(x)=\lambda(x) w_x(x)$.
Finally, using the definitions of $c_x$ and $w_x$ in conjunction with this bound gives
\begin{align*}
Q(x,1;s,\nu) - Q(x,0;s,\nu) &= f(x,1;s)+\nu g(x,1;s)-f(x,0;s)-\nu g(x,0;s) \\
&= \nu w_x(s)-c_x(s) \\
&\ge \nu w_x(s) - \lambda(s) w_x(s) \\
&= 0
\end{align*}
where the last line follows from the hypothesis that $\lambda(s) = \nu$.

Otherwise, $x\ge s$, and the claim follows easily from a symmetric argument.
However, for additional insight,
let us reorder the argument for $x\ge s$ as a single chain of equalities:
\begin{align*}
Q(x,0;s,\nu) - Q(x,1;s,\nu) 
&= -\nu w_x(s)+c_x(s) \\
&=-\nu w_x(s) + c_x(x) - \int_{[s,x)} \lambda \, dw_x \\
&=-\nu w_x(s) + c_x(x) - \lambda(x) w_x(x) + \lambda(s) w_x(s) + \int_{[s,x)} w_x \, d\lambda \\
&= (\lambda(s) - \nu) w_x(s) + \int_{[s,x)} w_x \, d\lambda \\
&=\int_{[s,x)} w_x \, d\lambda .
\end{align*}
Now, the fact that $w_x$ is positive (by PCLI1)
and that $\lambda$ is non-decreasing (by PCLI2), show that the integral in the last line is 
positive if $\lambda(x) > \nu$ and vanishes if $\lambda(x)=\nu$.

\noindent\textit{Claim 2.} Suppose $\lambda(s')>\nu$ for all $s'\in\IR$. We shall show that 
\begin{align*}
Q(x,0;-\infty,\nu) &> Q(x,1;-\infty,\nu) \qquad\text{for all $x\in\IR$.} 
\end{align*}
Therefore the always-active policy is the unique optimal policy.

We consider two cases: either the interval $\IR$ is left-open, being of the form
$(l,h)$ or $(l,h]$, or it is left-closed, being of the form $[l,h)$ or $[l,h]$. 
Say the interval is left-open and consider any $s\in\bR$ with $s\le l$.
As $\phi_a : \IR \rightarrow \IR$ for $a\in\{0,1\}$, we have $X_t(x,a;s') > l$ for all $t\in\Z_+$ and all $s'>l$.
Thus $A_t(x,a;l^+) = 1 = A_t(x,a;s)$ for all $t\in\Z_+$.
Therefore 
\begin{align}
\label{eq:cw-at-l}
c_x(l^+) = c_x(s) \qquad\text{and} \qquad w_x(l^+) = w_x(s).
\end{align}
Recalling Theorem~6.1.3 (i) of~\citet{Carter00}, which says that if $w_x$ is continuous at $u$
then $\int_{[u,x)} \lambda \, dw_x = \int_{(u,x)} \lambda \, dw_x$, we have
\begin{align*}
c_x(x) - c_x(s) 
&=\lim_{u\downarrow l} (c_x(x) - c_x(u)) & \text{by (\ref{eq:cw-at-l})} 	\\
&=\lim_{u\downarrow l} \int_{[u,x)} \lambda \, dw_x & \text{by PCLI3} \\
&=\lim_{u\downarrow l} \int_{(u,x)} \lambda \, dw_x & \text{by (\ref{eq:cw-at-l})} \\
&=  \int_{(l,x)} \lambda \, dw_x.
\end{align*}
Now $\lambda(l^+), \lambda(x) \in \R$, and $\lambda$ is non-decreasing, so $\lambda$
has bounded variation on the interval $(l,x)$.
Thus, the integration-by-parts argument used for Claim~1 gives
\begin{align*}
Q(x,0;s,\nu) - Q(x,1;s,\nu) 
&= c_x(s) - \nu w_x(s) \\
&= c_x(x) - \int_{(l,x)} \lambda \, dw_x - \nu w_x(s) \\
&= c_x(x) - \lambda(x^-) w_x(x^-) + \lambda(l^+) w_x(l^+) + \int_{(l,x)} \lambda \, dw_x - \nu w_x(s) \\
&= (\lambda(l^+) - \nu) w_x(l) + \int_{(l,x)} \lambda \, dw_x \\
&>0
\end{align*}
as claimed.

Now say the interval is left-closed, with lower limit $l$ and consider any $s\in\bR$ with $s<l$.
As for left-open intervals, we argue that $c_x(l) = c_x(s)$ and $w_x(l) = w_x(s)$.
The argument for Claim~1 (without needing to consider any $\lim_{u\downarrow l}$, 
since PCLI3 immediately covers intervals of the form $[l,x)$) then gives 
$Q(x,0;s,\nu) - Q(x,1;s,\nu) > 0$ as claimed.
\\\vspace{0.1cm}

\noindent\textit{Claim 3.} Suppose $\lambda(s')<\nu$ for all $s'\in\IR$.
Then an argument similar to for Claim~2 gives 
\begin{align*}
Q(x,0;\infty,\nu) &< Q(x,1;\infty,\nu) \qquad\text{for all $x\in\IR$.}
\end{align*}
Therefore the always-passive policy is the unique optimal policy.
\\\vspace{0.1cm}

\noindent\textit{Claim 4.} 
If $\lambda(x) < \nu$ then 
either $\lambda(s)=\nu$ for some $s\in\IR$
in which case Claim~1 shows that action 0 is strictly optimal,
or $\lambda(s)<\nu$ for all $s\in\IR$
in which case Claim~3 shows that action 0 is strictly optimal.
If $\lambda(x)>\nu$ then a similar argument using Claims~1 and~2 shows that action 1
is strictly optimal.
If $\lambda(x)=\nu$ then Claim~1 shows that actions 0 and 1 are both optimal.
Thus
\begin{align*}
\left. \begin{array}{l}
\text{action 0 is optimal in state $x$ if and only if $\lambda(x)\le\nu$} \\
\text{action 1 is optimal in state $x$ if and only if $\lambda(x)\ge\nu$}
\end{array} \right\} \quad \text{for all $x\in\IR$.}
\end{align*}
Therefore, the family is indexable with Whittle index $\lambda$.

This completes the proof.
\end{proof}

\begin{remark} 
The verification theorem presented by \citet{NinoMora15}
is considerably more general than that given here 
as it covers the case of stochastic rather than 
deterministic transition kernels, but the proof is considerably longer. 
Further generalisation of that theorem may be interesting, 
for instance to state spaces that are subsets of $\R^n$ and 
to semi-Markov problems.
\end{remark}

%
%
\section{Itineraries and Mechanical Words}
\label{section:mechanical}
The transitions from state-to-state under an $s$-threshold policy
are given by a discontinuous mapping known as a \textit{map-with-a-gap},
and the corresponding action sequences
are known as the \textit{itinerary} of that map. 
The purpose of this section is to introduce a central result of the paper (Theorem~\ref{theorem:characterisation}) which show that these itineraries are given
by special binary strings known as \textit{mechanical words}.
Before giving that result, 
we must first describe some important properties of maps-with-gaps 
and mechanical words.

\subsection{Maps-with-Gaps}
Many phenomena involve the iterated application of discontinous maps.
Such phenomena are important in control problems~\citep{Haddad14},
in physics, electronics and mechanics~\citep{Bernado08, Makarenkov12}, economics~\citep{Tramontana10}, 
biology and medicine~\citep{Aihara10}.
Such maps either arise directly from a discrete-time model or 
they may arise as the Poincar{\'e} maps of continuous-time systems.

For the purposes of this paper, we shall call a function $\psi : \mathcal{I} \rightarrow \mathcal{I}$,
where $\mathcal{I}$ is an interval of $\R$, a \textbf{map-with-a-gap} if 
\begin{align*}
\psi(x) = \begin{cases} \phi_0(x) & \text{if $x< z$} \\ \phi_1(x) & \text{otherwise} \end{cases}
\end{align*}
for some functions $\phi_0 : \mathcal{I} \rightarrow \mathcal{I}$ and $\phi_1 : \mathcal{I} \rightarrow \mathcal{I}$ and some threshold $z \in [-\infty,\infty]$.
This allows for thresholds $z\notin \mathcal{I}$, so 
the map $\psi$ may not really have a discontinuity, but this is helpful to allow for a 
general analysis of threshold policies.
The key concepts associated with such maps
are given in the following definition, as illustrated in Figure~\ref{fig:mapwithgap}.

\begin{figure}
\begin{center}
\includegraphics[viewport=50mm 90mm 155mm 200mm,clip=true,width=0.5\textwidth]{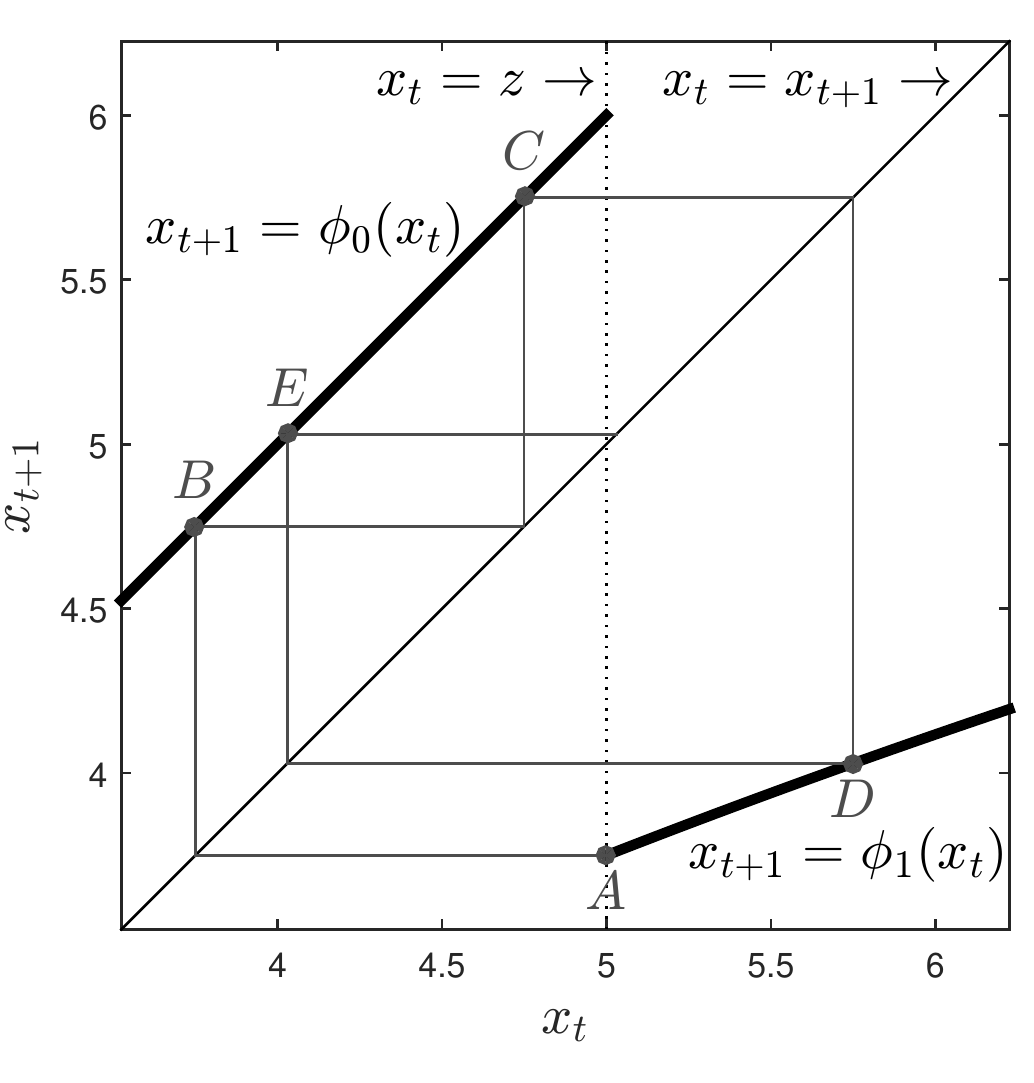}
\end{center}
\caption{Map-with-a-gap. The orbit traces the path $ABCDE\dots$ corresponding to the itinerary $10010\dots$. 
The map has $\phi_0(x) = x+1, \ \phi_1(x) = 1/(a_1+1/(x+1))$ with $a_1 = 0.1$ and threshold $z = 5$.}
\label{fig:mapwithgap}
\end{figure}

\begin{definition}
\label{def:mapwithgap}
Suppose $\psi : \mathcal{I}\rightarrow\mathcal{I}$ is a map-with-a-gap with threshold $z\in [-\infty,\infty]$.
The \textbf{orbit} of $\psi$ from initial state $x\in \mathcal{I}$ is the sequence
$(x_t : t \in \N)$ with
\begin{align*}
\text{$x_1 = x$ and $x_{t+1} = \psi(x_t)$ for $t \in \N$.}
\end{align*}
The \textbf{itinerary} of $\psi$ from state $x\in\mathcal{I}$ is the infinite binary string $\sigma(x|z)$ with $t^{\text{th}}$ letter
\begin{align*}
\text{$\sigma(x|z)_t := \I{x_t \ge z}$ for $t\in\N$.}
\end{align*}
\end{definition}

It is helpful to view such itineraries as \textit{words} as we now explain.

\subsection{Mechanical Words and \texorpdfstring{$\mathcal{M}$-Words}{M-Words}}
In this paper, a \textit{word} $w$ is a string on the alphabet $\{0,1\}$ and
the empty word is denoted by $\epsilon$.
The \textit{length} of a word $w$ is the number of letters in the string, which is finite or countably infinite,
and is denoted by $\abs{w}$.
The $k^{\text{th}}$ letter of word $w$ is $w_k$ for $k\in\N$ with $k\le \abs{w}$.
Letters $i$-through-$j$ of word $w$ are denoted by $w_{i:j} := w_i w_{i+1} \dots w_j$
for $i,j\in\N$ with $i\le j\le \abs{w}$.
For $j<i$, we treat $w_{i:j}$ as the empty word.
The \textit{reverse} word of a finite word $w$ is denoted by $w^R := w_{\abs{w}} \dots w_2 w_1$.
A finite word satisfying $w^R = w$ is called a \textit{palindrome}.

The \textit{concatenation} of a finite word $u$ and a word $v$ is denoted by $uv$.
For $n\in\Z_+$, the $n$-fold concatenation of a finite word $w$ is denoted by $w^n$,
with the convention that $w^0 = \epsilon$,
and the word resulting from infinitely concatenating the 
word $w$ is denoted by $w^\omega$. 
For an infinite word $w$ and $n\in\N$ we define $w^n = w^\omega = w$.

A finite word $f$ is a \textit{factor} of a word $w$ if $w = ufv$ for some finite word $u$
and some word $v$.
The number of times that word $f$ appears in $w$, overlapping appearances included,
 is denoted by $\abs{w}_f$.
A finite word $p$ is a \textit{prefix} of word $w$ if $w = ps$ for some word $s$
and a word $s$ is a \textit{suffix} of word $w$ if $w = ps$ for some finite word $p$.

We say that a word $u$ is \textit{lexicographically less than} a word $v$, written $u\prec v$,
if either $u$ is a finite word and $v=ua$ for some non-empty word $a$,
or if $u = a0b$ and $v=a1c$ for some finite word $a$ and some words $b$ and $c$.
We use $\succ, \preceq$ and $\succeq$ for the other lexicographic ordering relations.

We say an infinite word $w$ is the \textit{limit} of a sequence of  
words $(x^{(n)} : n\in\N)$ and write
\begin{align*}
w = \lim_{n\rightarrow\infty} x^{(n)}
\end{align*}
if for each $i\in\N$ there is an $n\in\N$
such that $w_i = x^{(m)}_i$ for all $m\in\N$ with $m\ge n$.

\paragraph{Example.} For $w=010111$ we have 
$\abs{w}=6$, $w_3=0$, $w_{2:4} = 101$, 
$\abs{w}_{01}=\abs{w}_{11} = 2$ 
and $w^2 = ww = 010111010111$. 
Also for $a=01, b = 11$ we have $w = aab$ and $a\prec w\prec b$.

One can view the itineraries of maps-with-gaps as words.

\begin{definition}
Sequence $(x_k : k \in\N)$ is the \textbf{$x$-threshold orbit}
for $\phi_0, \phi_1$ and $x\in\IR$ if
\begin{align*}
x_1 &= \phi_1( x), & 
x_{k+1} &= \begin{cases} \phi_1 (x_k) & \text{if $x_k \ge x$} \\ \phi_0 (x_k ) & \text{if $x_k < x$} \end{cases} & \text{for $k\in \N$} .
\intertext{The \textbf{$x$-threshold word} for $\phi_0$ and $\phi_1$, denoted by $\pi(x,\phi_0,\phi_1)$,
is the shortest word $w$ with}
&& x_{k+1} &= \phi_{(w^\omega)_k} (x_k) & \text{for $k\in \N$.}
\end{align*}
\end{definition}

We shall just say ``$x$-threshold orbit'', ``$x$-threshold word'' and write $\pi(x)$ in place of $\pi(x,\phi_0,\phi_1)$ 
when $\phi_0$ and $\phi_1$ are obvious from the context. 
Clearly, the itinerary is related to the $x$-threshold word 
by $\sigma(x|x) = 1\pi(x)^\omega.$ 

The \textit{rate} of any any finite non-empty word $w$ is the ratio
\begin{align*}
\rate(w) &:= \abs{w}_1 / \abs{w} 
\intertext{whereas for an infinite word $w$, when the limit exists, we define}
\rate(w) &:= \lim_{n\rightarrow\infty} \abs{w_{1:n}}_1 / n .
\end{align*}
While some authors refer to such ratios as the ``slope'' of a word, we use the term ``rate''
as the ``slope'' of a word $w$ is sometimes defined as the ratio $\abs{w}_{1} / \abs{w}_0$
and this seems justified from a geometrical point of view in terms of \textit{digital straight lines} \citep{Berstel08}. 

We characterise itineraries of maps-with-gaps in terms of following type of words.
\begin{definition}
The \textbf{\mword{ }}of rate $\alpha\in [0,1]$ is the shortest word $w$ such that
\begin{align*}
(w^\omega)_n = \floor{\alpha n}-\floor{\alpha (n-1)}\quad \text{for $n \in \N$.}
\end{align*}
If $\alpha$ is rational then $w$ is called a \textbf{Christoffel word}. 
If $\alpha$ is irrational then $w$ is called a \textbf{Sturmian \mword.}
\end{definition}

\paragraph{Example.} The shortest \mword{s }are the words $0, 1$ and $01$ with rates $0, 1$ and $\frac{1}{2}$. 

We call such words \mword{s }as our definition is closely related to the set of \textit{mechanical words}.
For a given \textit{slope} $\alpha\in [0,1]$ and \textit{intercept} $\rho\in\R$, \citet{Morse40} defined the \textit{upper} and \textit{lower mechanical words} 
to be the infinite sequences, for $n\in \Z_+$, 
\begin{align*}
u_n &= \left\lceil \alpha (n+1) + \rho \right\rceil - \left\lceil \alpha n + \rho \right\rceil \\
l_n &=  \left\lfloor \alpha (n+1) + \rho \right\rfloor - \left\lfloor \alpha n + \rho \right\rfloor 
\end{align*}
\citet{Lothaire02} and \citet{Berstel08} offer rich introductions to the mathematics 
of mechanical words, while \citet{Bousch02} and \citet{Altman00} 
explore other optimisation problems that give rise to such words.
Our \mword{s }are prefixes of lower-mechanical-words-of-zero-intercept up to a change of indexing from $n\in\Z_+$ to $n\in\N$.

It is not hard to see that the Christoffel word of rate $a/b$, 
where $a, b$ are relatively-prime integers, 
has length $b$. 
In contrast, Sturmian \mword{s }are infinite and aperiodic.

In general the \mword{ }$w$ of rate $\alpha$ does have $\rate(w) = \alpha$.
Indeed if $w$ is the \mword{ }of rate $a/b$ for some $a,b\in\N$, then 
\begin{align*}
\rate(w) &= \abs{w}_1 / \abs{w} = \lfloor (a/b) \abs{w} \rfloor / \abs{w} = a/b,
\intertext{whereas, if $w$ is an \mword{ }of irrational rate $\alpha$, then}
\rate(w) &= \lim_{n\rightarrow\infty} \abs{w_{1:n}}_1 / n =  \lim_{n\rightarrow\infty} \lfloor \alpha n \rfloor / n = \alpha .
\end{align*}

Furthermore, as remarked by \citet{Christoffel75}, all Christoffel words other than the words $0$ and $1$ are of the form
$0p1$ where the word $p$ is a \textit{palindrome}.
Indeed for relatively-prime positive integers
$m<n$, the letters of the Christoffel word $w$ of rate $m/n$ satisfy
\begin{align*}
w_{n-k} = \left\lfloor \frac{m}{n} (n-k) \right\rfloor - \left\lfloor \frac{m}{n} (n-k-1) \right\rfloor 
= \left\lfloor - \frac{m}{n} k \right\rfloor - \left\lfloor - \frac{m}{n} (k+1) \right\rfloor 
= w_{k+1}
\end{align*}
for $k = 1, 2, \dots, n-2$.

\begin{figure}
\begin{center}
\includegraphics[viewport=20mm 130mm 192mm 162mm,clip=true,width=0.98\textwidth]{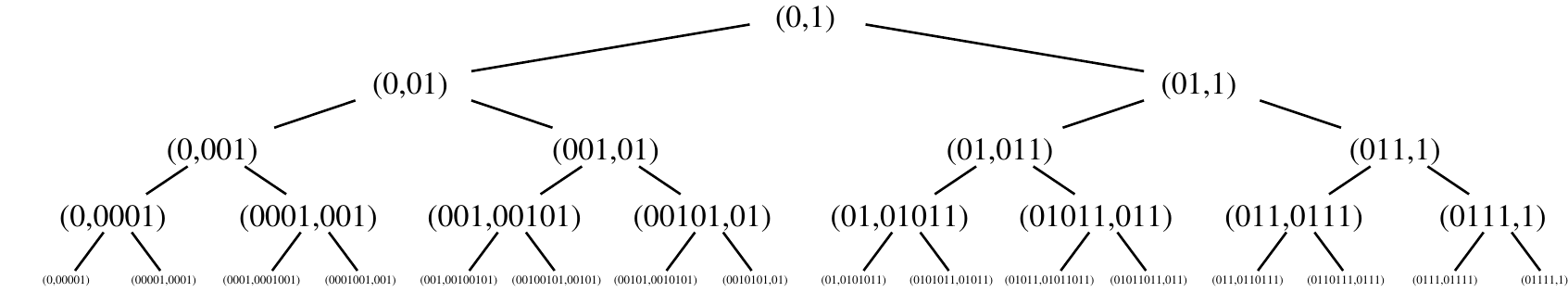}
\end{center}
\caption{Part of the Christoffel tree.}
\label{fig:tree}
\end{figure}

The Christoffel words can be defined in other ways.
In this paper the most important alternative-but-equivalent definition is in terms of the 
\textit{Christoffel tree} (Figure~\ref{fig:tree}), which is an infinite complete
binary tree~\citep{Berstel08} in which each node is labelled with a pair $(u,v)$
of words, called a \textit{Christoffel pair}.
The root of the tree is labelled with the pair $(0,1)$ and the 
left and right children of node $(u,v)$ are the nodes
$(u,uv)$ and $(uv,v)$ respectively. 
In fact the Christoffel words are the words $0,1$
and the set of concatenations $uv$ for all $(u,v)$ in the Christoffel tree.

Another definition of Christoffel words
is in terms of modular arithmetic, as in the following Lemma, where 
we use a bar to denote the remainder modulo the length $n = \abs{w}$
of a Christoffel word $w$, so that $\overline{x} := \modn{x}$ for $x\in\Z$.
\begin{lemma}
\label{lemma:modulo}
Suppose $w$ is a Christoffel word of length $n$. Let $m:=\abs{w}_1$ and $p:=\abs{w}_0$. Then
\begin{align*}
w_{i+1} = \I{\overline{m i} \ge p} \qquad (i\in\Z_n).
\end{align*}
\end{lemma}
\begin{proof}
As $n \lfloor mi / n \rfloor = mi - \overline{mi}$, the definition of Christoffel words gives
\begin{align*}
w_{i+1} &= - \lfloor mi/n \rfloor + \lfloor m(i+1)/n \rfloor \\
&= (-mi + \overline{mi} + m(i+1) - \overline{m(i+1)})/n \\
&= (-mi + \overline{mi} + m(i+1) - (\overline{mi}+m-n\I{\overline{mi} \ge n-m}))/n,
\end{align*}
which simplifies to $\I{\overline{mi}\ge p}$, as claimed.
\end{proof}

Finally, we give two results about \mword{s }that play a key role elsewhere in 
the paper. 
The first result is about \textit{conjugacy} and lexicographic order.
In particular, we say two finite words $a$ and $b$ are conjugate if $a=uv$ and $b=vu$
for some words $u$ and $v$. 
For instance, the words $a=00011$ and $b=01100$ are conjugate. 

\begin{lemma}
\label{lemma:BWT}
Suppose $w$ is a Christoffel word of length $n$ and 
that $l$ satisfies $\overline{lm}=1$ where $m=\abs{w}_1$.
Then the conjugates
$u(i):=w_{(\overline{i}+1):n}w_{1:\overline{i}}$ satisfy
\begin{align*}
w=u(0) \prec u(l) \prec u(2l) \prec \cdots \prec u((n-1)l)=w^R.
\end{align*}
\end{lemma}\begin{proof}
Let $x_i := \overline{mi-1},y_i := \overline{mi}$ and $p:=n-m$. Then $x_0=n-1$ and $x_{n-1}=p-1$.
As $\gcd(m,n)=1$, the sequence $x_0,\dots,x_{n-1}$ is a permutation of $\Z_n$.
So, $x_i \notin \{p-1,n-1\}$ for $i\in\{1,\dots,n-2\}$.
As $y_i = \overline{x_i+1}$ these results give
\begin{align*}
\I{x_i\ge p} &> \I{y_i\ge p} &&\text{for $i=0$} \\
\I{x_i\ge p} &= \I{y_i\ge p} &&\text{for $i=1,\dots,n-2$} \\
\I{x_i\ge p} &< \I{y_i\ge p} &&\text{for $i=n-1$.} 
\end{align*}
But Lemma~\ref{lemma:modulo} gives $u(0)_{j+1}=\I{y_j\ge p}$ and $u((n-1)l)_{j+1}=\I{x_j\ge p}$ for $j\in\Z_n$. Thus $u(0)=0a1$ and $u((n-1)l)=1a0$ for some word $a$. But $u(0)=w$ and $w$ is a Christoffel word, so $a$ is a palindrome. Therefore $u((n-1)l)=w^R$.

Now for $i=0,\dots,n-2$, the conjugates $u(il)$ and $u((i+1)l)$ 
are related to $u((n-1)l)$ and $u(0)$ respectively by the same non-zero cyclic rotation.
Thus $u(il)=c01d$ and $u((i+1)l)=c10d$ for some words $c$ and $d$ with $dc=a$.
Therefore $u(il) \prec u((i+1)l)$.
\end{proof}

The second result shows how the prefixes of \mword{s }vary as a function of their rates. It requires one more definition.

\begin{definition}
For each positive integer $n$, the \textbf{Farey sequence} $F_n$ is 
the sequence of rational numbers on $[0,1]$ whose denominator is at most $n$.
\end{definition}

For example, the Farey sequence $F_5$ is $0, \frac{1}{5}, \frac{1}{4}, \frac{1}{3}, \frac{2}{5}, \frac{1}{2}, \frac{3}{5}, \frac{2}{3}, \frac{3}{4}, 1$.

\begin{lemma}
\label{lemma:nprefix}
Suppose $n\in\N$ and $q \in [0,1]$.
Let $q_1<q_2<\dots<q_m$ be the Farey sequence $F_n$.
Let $p(s)$ be the first $n$ letters of the word $w^\omega$ where $w$ is the \mword{ }of rate $s\in [0,1]$.
Then $p(q) = p(q_i)$ 
if and only if 
either $q=q_i=1$ or $q \in [q_i,q_{i+1})$ for some $1\le i<m$.
\end{lemma}
\begin{proof}
Let $b(q):=(\lfloor q \rfloor, \lfloor 2q\rfloor, \dots, \lfloor nq \rfloor)$
and consider the intervals 
$\mathcal{Q}_i := [q_i, q_{i+1})$ for $i<m$ and $\mathcal{Q}_m := \{1\}$.
As the line $y=qx$ hits an integer point $(x,y)\in \Z^2$ with $1\le x \le n$
and $0\le y\le x$ if and only if $q$ is an element of $F_n$,
it follows that $b(q) = b(q_i)$ if and only if $q\in \mathcal{Q}_i.$
Let $g(x_1, x_2, \dots, x_n):=(x_1, x_2-x_1, \dots, x_n-x_{n-1})$.
By definition of \mword{s}, $p(q)=g(b(q))$.
As $g$ is invertible it follows that $p(q)=p(q_i)$ if and only if $q\in \mathcal{Q}_i$.
\end{proof}

\subsection{Characterising Itineraries as \texorpdfstring{\MWord{s}}{M-Words}}
Our aim here is to characterise the itineraries of maps-with-gaps.
We first set up some notation and then demonstrate a simple result
about the lexicographical ordering of itineraries.
Then we state an assumption under which itineraries are guaranteed
to be specific \mword{s }whose proof is given in Appendix~A. 
 
Let $\mathcal{I}$ be an interval of $\R$ and consider two mappings $\phi_0 : \mathcal{I} \rightarrow \mathcal{I}$ and $\phi_1 : \mathcal{I} \rightarrow \mathcal{I}$. 
For any finite word $w$, the \textit{composition} $\phi_w : \mathcal{I} \rightarrow \mathcal{I}$ is the mapping 
\begin{align*}
\phi_w(x) := \phi_{w_{\abs{w}}} \circ \cdots \circ \phi_{w_2} \circ \phi_{w_1} (x) \quad\text{and} \quad \phi_\epsilon(x) := x. 
\end{align*}

A simple application of compositions gives the following result about lexicographic ordering of the itineraries
$\sigma(\cdot|z)$ of a map-with-a-gap given by mappings $\phi_0 : \mathcal{I} \rightarrow \mathcal{I}$ and
$\phi_1 : \mathcal{I} \rightarrow \mathcal{I}$ and threshold $z$.

\begin{lemma}
\label{lemma:lexy}
Suppose $\phi_0,\phi_1$ are increasing mappings and that $x, y\in \mathcal{I}$ with $\sigma(x|z) \prec \sigma(y|z)$.
Then $x < y$.
\end{lemma}
\begin{proof}
If $\sigma(x|z) \prec \sigma(y|z)$ then $\sigma(x|z)=a0b$ and $\sigma(y|z)=a1c$
for some finite word $a$ and some infinite words $b,c$, by the definition of lexicographic order. 
So, the definition of $\sigma(\cdot|z)$ gives $\phi_a(x) < z \le \phi_a(y)$.
But $\phi_a(\cdot)$ increasing as it is a finite composition of increasing functions.
It follows that $x<y$.
\end{proof}

However, without additional assumptions about the mappings $\phi_0,\phi_1$, it is not possible to precisely characterise the itineraries of the associated maps-with-gaps.
\\

\noindent\textit{{\bf Assumption A1} Functions $\phi_0 : \IR \rightarrow \IR$ and $\phi_1 : \IR \rightarrow \IR$, where $\IR$ is an interval of $\mathbb{R}$, are increasing, contractive
and have unique fixed points $y_0$ and $y_1$ on $\IR$ which 
satisfy $y_1 < y_0$.}
\\

Equivalently, for all $x, y\in\IR$ such that $x < y$ and for $a \in \{0,1\}$ we have
\begin{align*}
\underbrace{\phi_a(x) < \phi_a(y)}_{\text{increasing}} \qquad \qquad \text{and} \qquad \qquad \underbrace{\phi_a(y) - \phi_a(x) < y-x}_{\text{contractive}} .
\end{align*}

A \textit{fixed point} for a finite word $w$, is a solution to the equation $x = \phi_w(x)$.
If Assumption A1 holds, then it turns out that there is a unique such fixed point on $\mathcal{I}$ 
for any finite non-empty word $w$ and we shall denote it by $y_w$.

In general, it is not clear what a ``fixed point'' corresponding to an \textit{infinite} word $w$ might mean.
One approach might be to consider a sequence $(w^{(n)} : n\in\N)$ of words
with $w = \lim_{n\rightarrow\infty} w^{(n)}$ and to define ``$y_w$'' as $\lim_{n\rightarrow\infty} y_{w^{(n)}}$
if that limit exists. However, for any word $a$, the sequence with elements $w^{(n)}a$ also converges to $w$
and it is not hard to find examples where 
\begin{align*}
\lim_{n\rightarrow\infty} y_{w^{(n)}} \ne \lim_{n\rightarrow\infty} y_{w^{(n)}a}.
\end{align*}
Therefore we shall only define fixed points for a particular class of infinite words, as follows.
Let $0s$ be the Sturmian \mword{ }of rate $\alpha$. Consider the sequence of Christoffel words $0w^{(n)}1$ 
that lie on the following path through the Christoffel tree. 
We start from the root, so that $w^{(1)} = \epsilon$. Then for $n\in\N$, we set $0w^{(n+1)}1$
equal to the left child of $0w^{(n)}1$ if the slope of $0w^{(n)}1$ exceeds $\alpha$ and equal to the right child otherwise.
We call 
\begin{align*}
y_s := \lim_{n\rightarrow\infty} y_{01w^{(n)}} = \lim_{n\rightarrow\infty} y_{10w^{(n)}}
\end{align*}
the \textit{fixed point} of the Sturmian \mword{ }$0s$. The fact that these limits exist and are equal is proved in Appendix~A.

We are now ready to fully characterise the itineraries
of maps-with-gaps when the initial point equals the threshold. 

\begin{theorem}
\label{theorem:characterisation}
Suppose A1 holds, $0p1$ is a Christoffel word and $0s$ is a Sturmian \mword.
Then the fixed points $y_{01p}, y_{10p}, y_{s}$ exist in $\mathcal{I}$.
Also, the itinerary $\sigma(z|z)$ is a lexicographically non-increasing function of $z\in\mathcal{I}$
and is of the form $\sigma(z|z) = 1 \pi(z)^\omega$ for some mapping $\pi : \mathcal{I} \rightarrow \{0,1\}^*$ whose image is the set of $\mathcal{M}$-words. Specifically,
\begin{align*}
\sigma(z|z) = \begin{cases} 
	1^\omega & \text{if and only if $z\le y_1$} \\
	(10p)^\omega & \text{if and only if $z\in [y_{01p},y_{10p}]$} \\
	10s & \text{if and only if $z=y_{s}$} \\
	10^\omega & \text{if and only if $z\ge y_0$.} \end{cases}
\end{align*}
\end{theorem}

This result is previously known for \textit{linear} maps-with-gaps
\citep{Rajpathak12}, although those authors do not draw any relation to mechanical words.
\citet{Dance15} previously extended those authors' proof to the nonlinear case under Assumption A1.
The proof presented in Appendix~A of this paper can be seen as a simplification of that extension. 
On the other hand, it is known that itineraries of a broader class of nonlinear maps-with-gaps that do not necessarily satisfy Assumption A1 also
correspond to mechanical words \citep{Kozyakin03}. However such generality comes at a cost, as it is not clear in that work
which range of thresholds gives rise to which words.

Finally, not all the maps-with-gaps considered in this paper satisfy Assumption A1.
However, this does not always prevent the application of Theorem~\ref{theorem:characterisation}.
Notably for $\mathcal{I} := [0,\infty)$ and $a\in (0,\infty)$, the pair 
\begin{align*}
\phi_0(x) &= x+1, \\
\phi_1(x) &= 1/(a+1/(x+1))
\end{align*}
involves the non-contractive map $\phi_0$.
Nevertheless, after the change of coordinates $$g : x \mapsto x/(x+1),$$ the transformed functions
\begin{align*}
\tilde\phi_0(x) &:= g( \phi_0( g^{(-1)}(x))) = 1/(2-x), \\ 
\tilde\phi_1(x) &:= g( \phi_1( g^{(-1)}(x))) = 1/(2+a-x)
\end{align*}
and the interval $\tilde{\mathcal{I}} := [0,1]$ do satisfy Assumption A1. 
Indeed
\begin{align*}
\frac{d \tilde\phi_1(x)}{dx} = 1 / (2+a-x)^2 \in (0,1]  
\end{align*}
for $x \in \tIR$ and $a \in [0,\infty)$, and this derivative only equals $1$ for the endpoint $x=1$.
Thus $\tilde\phi_1$ is increasing and contractive on $\tIR$. 
Noting that $\tilde\phi_0(x) = \lim_{a\rightarrow 0} \tilde\phi_1(x)$,
the same holds for $\tilde\phi_0$.
Also $\tilde\phi_1$ has a fixed point at $y(a) = (2+a-\sqrt{a^2+4a})/2$ which lies in $\tIR$ 
for $a\in [0,\infty)$, and $\tilde\phi_0$ has a fixed point at $y(0) = 1 > y(a)$.
As $g$ is an increasing function, all conclusions of Theorem~\ref{theorem:characterisation} 
still hold for the original functions $\phi_0,\phi_1$.

We conclude our discussion of mechanical words by showing that
the itinerary, viewed as a function of the threshold, has at most a polynomial number of discontinuities.
This result is important for changing the order of certain summations
when showing that conditions PCLI0 and PCLI3 hold.

\begin{theorem}
\label{theorem:xs-characterisation}
Suppose $\phi_0,\phi_1$ satisfy A1, that $n\in\N$ and $x\in\IR$. Then $\sigma(x|s)$
is a lexicographically non-increasing function of $s\in\IR$. 
Also, for any fixed $x,s\in\IR$, we have
\begin{align*}
\sigma(x|s)_{1:n} = l^m w
\end{align*}
for some $l\in\{0,1\}$, some $m \in \{0, 1, \dots, n\}$, 
and some factor $w$ of a lower mechanical word.
Furthermore, for any $x\in\IR$, the mapping $s\mapsto \sigma(x|s)_{1:n}$ for $s\in\IR$ has at 
most a polynomial number $p(n)$ of discontinuities.
\end{theorem}

The proof of this result is given at the end of Appendix~A.

%
%
\section{Bounded Variation Condition (PCLI0)}
\label{section:bounded-variation}
The next few sections of this paper prove that conditions PCLI0-PCLI3 
hold for the systems considered in the introduction, that is, 
systems satisfying the following condition.
\\

\noindent\textbf{Condition D.} \textit{For the tuple $\langle \IR, C, \phi_0, \phi_1, \beta \rangle$, 
\begin{itemize}
\item The interval $\IR$ and cost function $C$ satisfy Condition~C.
\item The transitions $\phi_0 : \IR \rightarrow \IR$ and $\phi_1 : \IR \rightarrow \IR$ are 
of the form
\begin{align*}
\phi_0(x) &:= \frac{r^2 x + 1}{a_0 r^2 x + a_0+1}, &
\phi_1(x) &:= \frac{r^2 x + 1}{a_1 r^2 x + a_1+1},
\end{align*}
for some $0 \le a_0 < a_1 < \infty$ and some $r \in (0,1]$. 
\item The discount $\beta$ is in $[0,1)$.
\end{itemize}}

Associated with any such tuple, we consider the family of dynamic programs with 
a single parameter $\nu\in\R$, 
of the form~(\ref{eq:simpleDP}), in which the work is of the form $c(x,a) := a$ 
and where the cost is independent of the action:
\begin{align*}
V(x;a) = \min_{a\in\{0,1\}} \left\{ \nu a + C(x) + \beta V(\phi_a(x);\nu) \right\}.
\end{align*}

(In fact, the results of the paper cover a slightly more general set of systems than those satisfying Condition D. It is possible to work with a smaller interval $\IR$ as long as the initial state is contained
in $\IR$ and the open interval containing the fixed points $(y_1,y_0)$ is in $\IR$.)

After a simple Lemma, we give Proposition~\ref{proposition:PCLI0}
which shows that condition PCLI0 holds.

\begin{lemma}
\label{lemma:caglad}
Suppose $\phi_0, \phi_1:\IR\rightarrow\IR$ are continuous functions, 
$t\in\Z_+$, $x\in\IR$ and $a\in\{0,1\}$.
Then $X_t(x,a;s)$ and $A_t(x,a;s)$ are piecewise-constant \caglad{ }functions of $s\in\bR$.
\end{lemma}
\begin{proof}
Let $\mathcal{P}$ denote the set of piecewise-constant \caglad{ }functions from $\R$ to $\R$
and recall the following easily-demonstrated facts about limits. 
\begin{enumerate}
\item[(i)] If $a \in \mathcal{P}$ and $f:\R\rightarrow\R$ is continuous
then $f\circ a$ is in $\mathcal{P}$.
\item[(ii)] If $a \in \mathcal{P}$ then both $s\mapsto \I{a(s) < s}$ and $s\mapsto \I{a(s) \ge s}$ are in $\mathcal{P}$.
\item[(iii)] If $a, b \in \mathcal{P}$ 
then $s\mapsto a(s)b(s)$ is in $\mathcal{P}$.
\item[(iv)] If $a,b \in \mathcal{P}$ 
then $s\mapsto a(s)+b(s)$ is in $\mathcal{P}$.
\end{enumerate}

Consider the claim about $X_t(x,a;s)$. We use induction on $t\in\Z_+$. 
In the base case, $X_0(x,a;s)=x$ by definition,
and the mapping $s\mapsto x$ is in $\mathcal{P}$.
For the inductive step, suppose $s \mapsto X_t(x,a;s)$ is in $\mathcal{P}$.
Then, by definition  
\begin{align*}
X_{t+1}(x,a;s) &= 
\underbrace{\phi_0(X_t(x,a;s)) \I{X_t(x,a;s) < s}}_{=:T_0(s)} + \underbrace{\phi_1(X_t(x,a;s)) \I{X_t(x,a;s) \ge s}}_{=:T_1(s)} .
\end{align*}
As $\phi_0$ is continuous, the induction hypothesis and (i) show that
$s\mapsto \phi_0(X_t(x,a;s))$ is in $\mathcal{P}$.
Also, the induction hypothesis and (ii) show that $s\mapsto \I{X_t(x,a;s)<s}$ is in $\mathcal{P}$.
Thus (iii) shows that $T_0$ is in $\mathcal{P}$.
A similar argument shows that $T_1$ is in $\mathcal{P}$.
Therefore (iv) shows that $s\mapsto T_0(s)+T_1(s) = X_{t+1}(x,a;s)$ 
is in $\mathcal{P}$.

The claim about $A_t(x,a;s) = \I{X_t(x,a;s) \le s}$ 
follows from (ii) and the fact that $s\mapsto X_t(x,a;s)$ is in $\mathcal{P}$. 
This completes the proof.
\end{proof}

Recalling the definitions of work-to-go and marginal work
(Definition~\ref{def:marginals}), we are now ready to show that condition PCLI0 holds.

\begin{proposition}
\label{proposition:PCLI0}
Suppose $\langle \IR, C, \phi_0, \phi_1, \beta \rangle$ satisfy Condition~D, 
$t\in\Z_+$, $x\in\IR$ and $a\in\{0,1\}$. 
Then the work-to-go $g(x,a;s)$ and the marginal work $w_x(s)$
are \caglad{ }functions of $s\in\bR$ with bounded variation.
\end{proposition}
\begin{proof}
First we show that $s\mapsto g(x,a;s)$ is \caglad. For any $b\in\bR$, we have 
\begin{align*}
g(x,a;b^-) &= \lim_{s\uparrow b} \sum_{t=0}^\infty \beta^t A_t(x,a;s) & \text{by definition} \\
&= \lim_{n\rightarrow \infty} \sum_{t=0}^\infty \beta^t A_t(x,a;s_n) & \text{for any sequence with $s_n \uparrow b$.} 
\intertext{Now $\abs{\beta^t A_t(x,a;s)} \le \beta^t$ for $t\in\Z_+$ and $\sum_{t=0}^\infty \beta^t < \infty$. Thus the dominated convergence theorem (treating sums as integrals with
respect to the counting measure) gives}
\lim_{n\rightarrow \infty} \sum_{t=0}^\infty \beta^t A_t(x,a;s_n) 
&= \sum_{t=0}^\infty \lim_{n\rightarrow\infty} \beta^t A_t(x,a;s_n) & \\
&= \sum_{t=0}^\infty \beta^t A_t(x,a;b) & \text{by Lemma~\ref{lemma:caglad}} \\
&= g(x,a;b).
\end{align*}
Therefore $s\mapsto g(x,a;s)$ is a \caglad{ }function.

Next we show that $s\mapsto g(x,a;s)$ has bounded variation.
Let $\mathcal{D}_t$ be the set of discontinuities of $A_t(x,a;s)$ as a function of $s\in\bR$.
As $s\mapsto A_t(x,a;s)$ is a piecewise-constant and \cadlag{ }with values in $\{0,1\}$,
for any $b,c\in\bR$ with $b\le c$, we can write
\begin{align*}
A_t(x,a;c)-A_t(x,a;b) &= \sum_{d\in\mathcal{D}_t} a_t(d) \I{d\in [b,c)}
\end{align*}
for some $a_t(d) \in \{-1,1\}$.
So for any partial subdivision $\mathcal{S}_n = \{[b_1, c_1], \dots, [b_n, c_n]\}$ of $\IR$,
\begin{align*}
\sum_{i=1}^n \abs{g(x,a;c_i)-g(x,a;b_i)} \hspace{-3cm} &\\
&= \sum_{i=1}^n \abs{\sum_{t=0}^\infty \beta^t A_t(x,a;c_i) - \sum_{t=0}^\infty \beta^t A_t(x,a;b_i)} \\
&= \sum_{i=1}^n \abs{\sum_{t=0}^\infty \beta^t (A_t(x,a;c_i) -  A_t(x,a;b_i))}
\\
&= \sum_{i=1}^n \abs{\sum_{t=0}^\infty \sum_{d\in\mathcal{D}_t}  \beta^t a_t(d) \I{d\in [b_i,c_i)}}\\
&\le \sum_{i=1}^n \sum_{t=0}^\infty \sum_{d\in\mathcal{D}_t}  \beta^t \I{d\in [b_i,c_i)}\\
&= 
\sum_{t=0}^\infty \sum_{d\in\mathcal{D}_t} \sum_{i=1}^n  \beta^t \I{d\in [b_i,c_i)}\\
&\le \sum_{t=0}^\infty \sum_{d\in\mathcal{D}_t} \beta^t \\
&\le \sum_{t=0}^\infty \beta^t p(t)
\end{align*}
using Tonelli's theorem,
the fact that $\mathcal{S}_n$ is a partial subdivision, 
and the polynomial function $p(t)$ of Theorem~\ref{theorem:xs-characterisation}.
But the bound $\sum_{t=0}^\infty \beta^t p(t)$ is finite 
and independent of the choice of $\mathcal{S}_n$.
Therefore $s\mapsto g(x,a;s)$ has bounded variation.

Now $w_x(s)$ is defined as $g(x,1;s)-g(x,0;s)$.
But the difference of \caglad{ }functions is \caglad{ }and the difference of functions of bounded variation has bounded variation.
Therefore $s\mapsto w_x(s)$ is a \caglad{ }function of bounded variation.
This completes the proof.
\end{proof}

%
%
\section{Positivity of Marginal Work (PCLI1)}
\label{section:positive-work}
We prove that condition PCLI1 holds.
The argument is based on the following notion of \textit{swapping}, which is partly inspired by results about the Burrows-Wheeler transform of Christoffel words~\citep[Chapter~6]{Berstel08}. 

\begin{definition}
A finite word $a$ \textbf{swaps to} a finite word $b$ if either $a=b$
or there exist words $p_1, q_1, p_2, q_2, \dots, p_n, q_n$ for some $n\in\N$ with
\begin{align*}
a &= p_110q_1, & p_101q_1 &= p_210q_2, &\dots,&& p_n01q_n &= b.
\end{align*}
We call a transformation $p_k10q_k \rightarrow p_k01q_k$ an \textbf{exchange}.
\end{definition}

\paragraph{Example.} The word $1100$ swaps to the word $0101$ via the exchanges
\begin{align*}
1100 \rightarrow 1010 \rightarrow 0110 \rightarrow 0101
\end{align*}
for which $p_1=1$, $q_1=0$, $p_2=\epsilon$, $q_2=10$, $p_3=01$ and $q_3=\epsilon$.

The idea of our proof is as follows.
First we find conditions on the number of 1's in prefixes of two words
that make it possible to swap one word for another (Lemma~\ref{lemma:swaps}). 
Proposition~\ref{proposition:POS} shows that those conditions are satisfied by the dynamical system (Claim~1), and they imply positivity of marginal work (Claim~2).

%
%
%

\begin{lemma}
\label{lemma:swaps}
Suppose $a,b$ are words of common length $\abs{a}=\abs{b}=n\in\Z_+$ with
\begin{align*}
\abs{a}_1 = \abs{b}_1 \quad\text{and}\quad
\abs{a_{1:k}}_1 \ge \abs{b_{1:k}}_1 \quad\text{for $k < n$.}
\end{align*}
Then $a$ swaps to $b$.
\end{lemma}
\begin{proof}
Given any words $u, v$ of length $n$, consider the distance 
\begin{align*}
d(u,v) := \sum_{i=1}^n \abs{ \abs{u_{1:i}}_1 - \abs{v_{1:i}}_1 } .
\end{align*}
If $a=b$ then $a$ swaps to $b$ after $d(a,b) = 0$ exchanges.
Otherwise $a\ne b$ and we shall show that there exists a word $a'$ such that
\begin{align}
&\text{$a$ and $a'$ differ by a single exchange,} \label{eq:swape} \\
&\text{$d(a',b) = d(a,b)-1$,} \label{eq:swapd}\\
\text{and }&\text{$a'$ and $b$ satisfy the hypotheses of this Lemma.} \label{eq:swaph} 
\end{align}
Repeating this argument shows that $a$ swaps to $b$ after $d(a,b)$ exchanges.

We now define an appropriate word $a'$. 
As $a\ne b$ and $\abs{a_{1:i}}_1 \ge \abs{b_{1:i}}_1$ for $i = 1, 2, \dots, n-1$,
there must exist a first index $i$ such that $\abs{a_{1:i}}_1 > \abs{b_{1:i}}_1$.
Also, as $\abs{a}_1 = \abs{b}_1$, there must exist a first index $j>i$ such that $a_j = 0$.
As $i,j$ are the first such indices, it follows that $a_k = 1$ for $i\le k < j$.
Thus 
\begin{align*}
a &= a_{1:(j-2)} 10 a_{(j+1):n}
\intertext{with the convention that $a_{1:0} = a_{(n+1):n} = \epsilon$.
Now consider the word} 
a' &:= a_{1:(j-2)} 01 a_{(j+1):n}.
\end{align*}
It is immediate that~(\ref{eq:swape}) holds.
Furthermore, as $a_k = 1$ for $i\le k < j$, we have
\begin{align*}
\abs{a_{1:(j-1)}}_1 - \abs{b_{1:(j-1)}}_1 \ge \abs{a_{1:i}}_1 - \abs{b_{1:i}}_1 > 0
\end{align*}
where the second inequality follows from the definition of $i$.
Thus the definition of $a'$ gives
\begin{align}
\label{eq:swap4}
\abs{a'_{1:l}}_1 = \abs{a_{1:l}}_1 - \I{l=j-1} \ge \abs{b_{1:l}}_1 \quad\text{for $l = 1, 2, \dots, n$,}
\end{align}
so that
\begin{align*}
d(a',b) =  \sum_{i=1}^n \left( \abs{a'_{1:i}}_1 - \abs{b_{1:i}}_1\right)
= d(a,b)-1.
\end{align*}
Therefore~(\ref{eq:swapd}) holds.

Finally, combining~(\ref{eq:swap4}) with the fact that $\abs{a'}_1 = \abs{a}_1 = \abs{b}_1$, 
we conclude that~(\ref{eq:swaph}) holds. This completes the proof.
\end{proof}

The hypotheses of the following proposition 
are clearly satisfied if $\langle \IR, C, \phi_0, \phi_1, \beta \rangle$ satisfy Condition~D, 
therefore PCLI1 holds.

\begin{proposition}
\label{proposition:POS}
Suppose $\IR$ is an interval of $\R$ and that $\phi_0 : \IR\rightarrow\IR, \phi_1 : \IR\rightarrow\IR$ satisfy
\begin{enumerate}
\item[(i)] $\phi_0(\cdot), \phi_1(\cdot)$ are increasing functions
\item[(ii)] $\phi_{01}(z)<\phi_{10}(z)$ for all $z\in\IR$.
\end{enumerate}
Also suppose that $x\in\IR$, $s\in\bR$ and consider the itineraries
\begin{align*}
a := 1\sigma(\phi_1(x)|s)_{1:(n-1)} \quad\text{and}\quad b := 0\sigma(\phi_0(x)|s)_{1:(n-1)}.
\end{align*}
Then 
\begin{enumerate}
\item For any $n\in\N$, we have $\abs{a_{1:n}}_1 \ge \abs{b_{1:n}}_1$. 
\item For any $\beta \in (0,1)$, the marginal work $w_x(s)$ is positive.
\end{enumerate}
\end{proposition}
\begin{proof}
We prove Claim~1 by induction. 
In the base case  $\abs{a_{1}}_1  = 1 \ge 0 = \abs{b_{1}}_1$.
For the inductive step, suppose $\abs{a_{1:k}}_1 \ge \abs{b_{1:k}}_1$ 
for all $k \le m$ for some $m\in\N$.
This induction hypothesis shows that either $\abs{a_{1:m}}_1 > \abs{b_{1:m}}_1$
or $\abs{a_{{1:m}}}_1 = \abs{b_{1:m}}_1$.
In the first case, $\abs{a_{1:(m+1)}}_1 \ge \abs{b_{1:(m+1)}}_1$
as we are only adding one letter to $a_{1:m}$ and $b_{1:m}$.
In the second case, the induction hypothesis shows that the words $a_{1:m}$ and $b_{1:m}$ satisfy
the assumptions of Lemma~\ref{lemma:swaps},
so there is a sequence of swaps that transforms $a_{1:m}$ into $b_{1:m}$.
Consider any swap $p10q$ to $p01q$ in this sequence.
Then hypothesis (ii) gives 
$\phi_{10}(\phi_p(x)) > \phi_{01}(\phi_p(x))$
and hypothesis (i) implies that $\phi_{q}(\cdot)$ is increasing, 
so
\begin{align*}
\phi_{p10q}(x) > \phi_{p01q}(x) .
\end{align*}
Repeating this argument over the sequence of swaps gives
\begin{align*}
\phi_{a_{1:m}}(x) > \phi_{b_{1:m}}(x).
\end{align*}
Thus, it follows from the definition of itineraries 
that the last letters of $a_{1:(m+1)},b_{1:(m+1)}$
have $a_{m+1}b_{m+1} \in \{00,10,11\}$.
Hence $\abs{a_{1:(m+1)}}_1 \ge \abs{b_{1:(m+1)}}_1$.
This proves Claim~1.

To prove Claim 2, note that the definition of $w_x(s)$ gives
\begin{align*}
w_x(s) &= \sum_{k=1}^\infty \beta^{k-1} (a_k-b_k) \\
&= \sum_{k=1}^\infty \beta^{k-1} (\abs{a_{1:k}}_1 - \abs{a_{1:(k-1)}}_1 - \abs{b_{1:k}}_1 + \abs{b_{1:(k-1)}}_1)\\
&= (1-\beta) \sum_{k=1}^\infty \beta^{k-1} (\abs{a_{1:k}}_1 - \abs{b_{1:k}}_1) \\
&\ge 1-\beta 
\end{align*}
where the last line follows from Claim~1 and the fact that $a_1 = 1, b_1 = 0$.
As $\beta < 1$, this completes the proof.
\end{proof}

%
%
\section{Non-Decreasing Marginal Cost (PCLI2, First Part)}
\label{section:non-decreasing}
Condition PCLI2 requires that the marginal productivity $\lambda(x) = c_x(x) / w_x(x)$ is non-decreasing for $x\in\IR$. In view of Theorem~\ref{theorem:characterisation}, the interval $\IR$ can be divided up into intervals corresponding to Christoffel words on which $w_x(x)$ is constant and points corresponding to Sturmian \mword{s}. The main result of this section is Proposition~\ref{proposition:increasing}, which shows that the marginal cost $c_x(x)$ is non-decreasing for $x$ in the interval corresponding to any given Christoffel word of the form $0p1$, for systems satisfying Condition~D. This result is complemented by Proposition~\ref{proposition:easy-increasing}, which shows that $c_x(x)$ is also increasing for $x\le y_1$ and $x\ge y_0$ where $y_0, y_1$ are the fixed points of $\phi_0,\phi_1$. As the marginal work $w_x(x)$ is positive by PCLI1, this implies that $\lambda(x)$ is non-decreasing on such intervals. In the Section~7, we show that $\lambda(x)$ is continuous for $x\in\IR$, so that PCLI2 is satisfied.

A related proof was given by~\citet{Dance15}. However, that proof only covers systems for which the multiplier $r$ in Condition~D is $r=1$, rather than multipliers $r \in (0,1]$ as addressed here. For $r=0$, the sum in Proposition~\ref{proposition:increasing} is a constant, so the analysis presented here is unnecessary. Also, the proof of~\citet{Dance15} only addressed the cost function $C(x)=x$, whereas here we generalise to any cost function satisfying Condition~C, which includes any cost function of the form $x^q / q$ for $q\in [-1, \infty)$. A counterexample presented in Section~9, shows that marginal cost is not necessarily increasing for $C(x) = x^q/q$ with $q<-1$.

We use the following well-known result about \textit{majorisation} \citep{Marshall10}. 
A proof is given in Appendix~B.

\begin{lemma}
\label{lemma:major}
Suppose that:
\begin{enumerate}
\item The sequences $a_{1:n}$ and $b_{1:n}$ are non-decreasing sequences on $\R_{++}$
\item The inequality $\sum_{i=1}^k a_i \le \sum_{i=1}^k b_i$ holds for $k=1, 2, \dots, n$  
\item For $i = 1, 2, \dots, n,$ the function $f_i : \R_{++}\rightarrow \R$ is non-increasing and convex
\item For $i = 2, 3, \dots, n$, the difference $f_{i-1}(x)-f_{i}(x)$ is non-increasing for $x\in\R_{++}$.
\end{enumerate}
Then
\begin{align*}
\sum_{i=1}^n f_i(a_i) &\ge \sum_{i=1}^n f_i(b_i) .
\end{align*}
\end{lemma}

We apply this majorisation result to the sequences appearing in Lemma~\ref{lemma:integrated} below. To state that Lemma, we first define some matrices that are motivated by the form of the Kalman Filter variance updates.

\begin{definition}
Let $I$ be the 2-by-2 identity matrix.
For $r\in (0,1]$ and $0\le a\le b$, let
\begin{align*}
F&:= \begin{pmatrix} r & 1/r \\ ar & (a+1)/r \end{pmatrix}, &
G&:= \begin{pmatrix} r & 1/r \\ br & (b+1)/r \end{pmatrix}.
\end{align*}
Let $M(\epsilon) = I, M(0)=F, M(1)=G$ and for any finite non-empty word $w$ let
\begin{align*}
M(w) &= M(w_\abs{w}) \cdots M(w_2) M(w_1).
\end{align*}
\end{definition}

Thus, if $a = a_0, b = a_1$ and $\phi_0,\phi_1$ are as in Condition~D, we have
\begin{align*}
\phi_0(x) &= \frac{F_{11} x + F_{12}}{F_{21} x + F_{22}}, &
\phi_1(x) &= \frac{G_{11} x + G_{12}}{G_{21} x + G_{22}} .
\end{align*}

As remarked in Section~3, the central portion of Christoffel words are palindromes. The following result holds for any palindromes, not just palindromes generated by Christoffel words.

\begin{lemma}
\label{lemma:integrated}
Suppose $p$ is a palindrome, $r\in (0,1]$, $n\in\Z_+$ and $x$ satisfies
\begin{align*}
\phi_p(0) \le x \le \phi_p\left( \frac{1}{1-r^2} \right).
\end{align*}
Let $m:=\abs{01p}$ and for $k=1, 2, \dots, m$ let
\begin{align*}
\begin{pmatrix} a_k(x) \\ c_k(x) \end{pmatrix} &:= M\biggl( (01p)^n(01p)_{1:k} \biggr) \begin{pmatrix} x \\ 1 \end{pmatrix}, &
\begin{pmatrix} b_k(x) \\ d_k(x) \end{pmatrix} &:= M\biggl( (10p)^n(10p)_{1:k} \biggr) \begin{pmatrix} x \\ 1 \end{pmatrix}.
\end{align*}
Then 
\begin{enumerate}
\item The sequences $a_{1:m}(x), b_{1:m}(x), c_{1:m}(x)$ and $d_{1:m}(x)$ are non-decreasing and positive
\item The inequality $a_k(x) \le b_k(x)$ holds for $k=1, 2, \dots, m$
\item The inequality $\sum_{i=1}^k c_i(x) \le \sum_{i=1}^k d_i(x)$ holds for $k=1, 2, \dots, m$
\item The inequalities $c_1(x) \le d_1(x)$ and $c_k(x) \ge d_k(x)$ hold for $k=2,3, \dots, m$.
\item The fixed points $y_{01p}$ and $y_{10p}$ satisfy
\begin{align*}
\phi_p(0) \le y_{01p} < y_{10p} \le \phi_p\left( \frac{1}{1-r^2} \right)
\end{align*}
so Claims~1-4 hold if $x\in [y_{01p},y_{10p}]$.
\end{enumerate}
\end{lemma}

The proof of the above Lemma is given in Appendix~C.

Now we are ready to demonstrate that marginal cost is non-decreasing for a wide range of cost functions on the intervals $[y_{01p},y_{10p}]$ corresponding to a Christoffel word $0p1$ appearing in Theorem~\ref{theorem:characterisation}. While we state the result for functions satisfying Conditions~C1 and~C2 separately, as sums of non-decreasing functions are non-decreasing, the result holds for any function
satisfying Condition~C.

\begin{proposition}
\label{proposition:increasing}
Suppose the interval $\IR$ is either $\R_+$ or $\R_{++}$ and that
$C:\IR\rightarrow \R$ has one of the following two properties:
\begin{enumerate}
\item For all $x\in\IR$, 
\begin{itemize}
\item the derivatives $C'(x) := \frac{d}{dx} C(x)$ and 
$C''(x) := \frac{d^2}{dx^2} C(x)$ exist, 
\item the function $C(x)$ is concave, 
\item the function $\frac{1}{x^2} C'\left(\frac{1}{x}\right)$ is non-increasing and convex 
\item and the function $\frac{1}{x^3}C''\left(\frac{1}{x}\right)$ is non-decreasing.
\end{itemize}
\item For all $x\in\IR$, 
\begin{itemize}
\item the derivative $C'(x) := \frac{d}{dx} C(x)$ exists,
\item the function $C(x)$ is non-decreasing and convex.
\end{itemize}
\end{enumerate}
Further suppose that $0p1$ is a Christoffel word, $\beta\in [0,1]$, $r\in (0,1]$ and $N\in\Z_+$.
Then 
\begin{align*}
\sum_{k=1}^{n} \beta^k (C(\phi_{(01p)^N(01p)_{1:k}}(x))-C(\phi_{(10p)^N(10p)_{1:k}}(x)))
\end{align*}
is a non-decreasing function of $x$ for $y_{01p} \le x\le y_{10p}$, where $n=\abs{0p1}$.
\end{proposition}


\begin{proof}[Proof when $C(\cdot)$ satisfies Property 1.]
Let $a_k(x), b_k(x), c_k(x), d_k(x)$ be as defined in Lemma~\ref{lemma:integrated}. 
Then the proposition is proved by the following inequalities (as justified immediately below):
\begin{align}
\nonumber
\hspace{1cm}&\hspace{-1cm} 
\frac{d}{dx} \sum_{k=1}^n \beta^k (C(\phi_{(01p)^N(01p)_{1:k}}(x)) -  C(\phi_{(10p)^N(10p)_{1:k}}(x))) \\
\label{inc:step1}
&= \sum_{k=1}^n \left( \frac{\beta^k}{c_k(x)^2} C'\left( \frac{a_k(x)}{c_k(x)} \right) 
- \frac{\beta^k}{d_k(x)^2} C'\left( \frac{b_k(x)}{d_k(x)} \right) \right) \\
\label{inc:step2}
&\ge \sum_{k=1}^n \left( \frac{\beta^k}{c_k(x)^2} C'\left( \frac{b_k(x)}{c_k(x)} \right) 
- \frac{\beta^k}{d_k(x)^2} C'\left( \frac{b_k(x)}{d_k(x)} \right) \right) \\
\nonumber
&= \sum_{k=1}^n (f_k(c_k(x))- f_k(d_k(x))) \quad\text{where $f_k(u) := \beta^k C'(b_k(x)/u)/u^2$} \\
\label{inc:step3}
&\ge 0.
\end{align}
Step (\ref{inc:step1}) follows from the chain rule as the homeomorphism of matrix multiplication and composition of M{\"o}bius transformations gives
\begin{align*}
\phi_{(01p)^n(01p)_{1:k}}(x) = 
\frac{\left[M((01p)^N(01p)_{1:k}) \begin{pmatrix} x \\ 1 \end{pmatrix} \right]_1}{\left[M((01p)^N(01p)_{1:k}) \begin{pmatrix} x \\ 1 \end{pmatrix} \right]_2}
=\frac{a_k(x)}{c_k(x)}
\end{align*}
while the fact that matrices $F$ and $G$ have unit determinant implies that the matrix $M((01p)^N(01p)_{1:k})$ also has unit determinant so that
\begin{align*}
\frac{d}{dx} \phi_{(01p)^N(01p)_{1:k}}(x) &= \frac{1}{c_k(x)^2} .
\end{align*}

Step (\ref{inc:step2}) follows as $\beta^k, c_k(x) > 0$, as $C(\cdot)$ is concave and as $a_k(x)\le b_k(x)$ by Lemma~\ref{lemma:integrated}. 

Step (\ref{inc:step3}) follows from Lemma~\ref{lemma:major}. In particular, Lemma~\ref{lemma:integrated} shows that the sequences $c_{1:n}(x)$ and $d_{1:n}(x)$ satisfy the hypotheses of Lemma~\ref{lemma:major}. Also, the fact that $C(\cdot)$ satisfies Property~1 shows that the functions $f_i(\cdot)$ for $i=1, \dots, n$ satisfy hypotheses~1 and~2 of Lemma~\ref{lemma:major}. Indeed, $f_i(\cdot)$ is non-increasing and convex as $\frac{1}{u^2} C'\left( \frac{1}{u} \right)$ is non-increasing and convex and $\beta^i, b_i(x) > 0$. Also, as $\frac{1}{u^3} C''\left( \frac{b}{u} \right)$ is non-decreasing in $u$ for $b>0$ and $0<b_{i-1}(x) \le b_{i}(x)$ for $i=2, \dots, n$, by Claim~1 of Lemma~\ref{lemma:integrated}, the following integral is also non-decreasing in $u$:
\begin{align*}
\int_{b_{i-1}(x)}^{b_i(x)} \frac{1}{u^3} C''\left( \frac{b}{u} \right) db
= \frac{1}{u^2} C'\left( \frac{b_{i}(x)}{u} \right) - 
\frac{1}{u^2} C'\left( \frac{b_{i-1}(x)}{u} \right) = \frac{1}{\beta^i} \left( f_{i}(u) -\beta f_{i-1}(u) \right)
\end{align*}
So $f_{i-1}(u)-f_{i}(u)$ is the sum of the non-increasing
functions $\beta f_{i-1}(u)- f_{i}(u)$ and $(1-\beta) f_{i-1}(u)$.

This completes the proof.
\end{proof}

\begin{proof}[Proof when $C(\cdot)$ satisfies Property~2.]
For $k\in\Z_n$ let 
\begin{align*}
a_k' &:= \beta^{k+1} \frac{d}{dx} \phi_{(01p)^N(01p)_{1:(k+1)}}(x), &
b_k' &:= \beta^{k+1} \frac{d}{dx} \phi_{(10p)^N(10p)_{1:(k+1)}}(x),\\
a_k &:= C'(y_{(1p0)_{(k+1):n}(1p0)_{1:k}}), & 
b_k &:= C'(y_{(0p1)_{(k+1):n}(0p1)_{1:k}}). 
\end{align*} 
Let $r_{[0]}\ge \cdots \ge r_{[n-1]}$ denote real numbers $r_0, \dots, r_{n-1}$ in non-increasing numerical order. Let $\overline{x}$ denote $x$ modulo $n$ and let $l$ satisfy $\overline{l \abs{0p1}_1} = 1$.
Then the proposition is proved by the following inequalities (as justified immediately below):
\begin{align}
\nonumber
\hspace{1cm}&\hspace{-1cm}
\frac{d}{dx}  \sum_{k=1}^{n} \beta^k \Bigl( 
	 C\left( \phi_{(01p)^N(01p)_{1:k}}(x) \right) - C\left( \phi_{(10p)^N(10p)_{1:k}}(x) \right) 
\Bigr) \\
\label{cvx:step1}
&= \sum_{k=1}^{n} \Bigl( 
	 C'\left( \phi_{(01p)^N(01p)_{1:k}}(x) \right) a_{k-1}' - C'\left( \phi_{(10p)^N(10p)_{1:k}}(x) \right)  b_{k-1}'  \Bigr) \\
\label{cvx:step2}
&\ge \sum_{k=1}^{n} \Bigl( 
	 C'\left( y_{(1p0)_{k:n} (1p0)_{1:(k-1)}} \right) a_{k-1}' - C'\left( y_{(0p1)_{k:n} (0p1)_{1:(k-1)}} \right)  b_{k-1}'  \Bigr) \\
\label{cvx:step3}
&= 	\sum_{k=0}^{n-1} \left( a_k a_k'- b_k b_k' \right) \\
\label{cvx:step4}
&= 	\sum_{k=0}^{n-1} a_{[k]} \left( a_{\overline{l(n-k)}}'-b_{\overline{l(n-k-1)}}' \right) \\
\label{cvx:step5}
&\ge 	\sum_{k=0}^{n-1} a_{[k]} \left( a_{\overline{l(n-k)}}'-b_{\overline{l(n-k)}}' \right) \\
\label{cvx:step6}
&\ge 	\sum_{k=0}^{n-1}a_{[0]}  \left( a_{\overline{l(n-k)}}'-b_{\overline{l(n-k)}}' \right) \\
\label{cvx:step7}
&\ge 	0 
\end{align}
Step (\ref{cvx:step1}) follows from the chain rule and definition of $a_k', b_k'$.

Step (\ref{cvx:step2}) follows as $C(\cdot)$ is convex, as $a_k', b_k' \ge 0$, and as, for $k=1, 2, \dots, n$, we have 
\begin{align*}
\phi_{(01p)^N(01p)_{1:k}}(x) &= \phi_{(01p)_{1:k}}(\phi_{(01p)^N}(x)) \\
&\ge \phi_{(01p)_{1:k}}(y_{01p}) \\
&= y_{(01p)_{(k+1):n} (01p)_{1:k}} \\
&= y_{(1p0)_{k:n} (1p0)_{1:(k-1)}} 
\end{align*}
where the inequality holds as $x\ge y_{01p}$ so that $\phi_{(01p)^N}(x) \ge y_{01p}$, and as $\phi_{(01p)_{1:k}}(\cdot)$ is increasing. The same argument using $x\le y_{10p}$ gives an upper bound on $C'(\phi_{(10p)^N(10p)_{1:k}}(x))$.

Step~(\ref{cvx:step3}) follows by shifting the summation indices and from the definition of $a_k, b_k$.

Step~(\ref{cvx:step4}) follows from Lemmas~\ref{lemma:BWT} and~\ref{lemma:lexy} and the convexity of $C(\cdot)$. Let $w_{[0]}\succeq \cdots \succeq w_{[n-1]}$ denote words $w(0), \dots, w(n-1)$ in non-increasing lexicographic order and let $c_k := (1p0)_{(k+1):n}(1p0)_{1:k}, d_k := (0p1)_{(k+1):n}(0p1)_{1:k}$ for $k\in\Z_n$. Then Lemma~\ref{lemma:BWT} shows that $c_{[i]} = d_{[i]} = c_{\overline{l(n-i)}} = d_{\overline{l(n-i-1)}}.$ Thus Lemma~\ref{lemma:lexy} gives $y_{c_{[i]}} = y_{d_{[i]}} = y_{c_{\overline{l(n-i)}}} = y_{d_{\overline{l(n-i-1)}}}$. Therefore the convexity of $C(\cdot)$ gives
\begin{align*}
a_{[i]} = b_{[i]} = a_{\overline{l(n-i)}} = b_{\overline{l(n-i-1)}}.
\end{align*}

Step~(\ref{cvx:step5}) follows as $C(\cdot)$ is non-decreasing, so that $a_{[i]} \ge 0$, and as $\phi_0(\cdot), \phi_1(\cdot)$ are non-decreasing and non-expansive, so that $b_i'$ is a product of derivatives where each derivative is in $[0,1]$. Thus $a_{[0]},\dots, a_{[n-1]}$ and $b_0', \dots, b_{n-1}'$ are non-negative non-increasing sequences. Therefore the rearrangement inequality $a_{[i]} b_j' + a_{[i+1]} b_0' \le a_{[i]} b_0' + a_{[i+1]} b_j'$ holds for all $i \in \Z_{n-1}$ and all $j\in \Z_n.$ But $b_{\overline{l(n-(n-1)-1)}}' = b_0'$ so repeated application of the rearrangement inequality gives
\begin{align*}
\sum_{k=0}^{n-1} a_{[k]} b_{\overline{l(n-k-1)}}' \le \sum_{k=0}^{n-1} a_{[k]} b_{\overline{l(n-k)}}'.
\end{align*}

Step (\ref{cvx:step6}) follows from Claim~4 of Lemma~\ref{lemma:integrated}, as for the $c_i(x), d_i(x)$ defined in that Lemma, we have
\begin{align*}
a_{i-1}'-b_{i-1}' 
= \frac{\beta^{i}}{c_{i}(x)^2} - \frac{\beta^{i}}{d_{i}(x)^2} 
\begin{cases} \le 0 & \text{for $i=1$} \\ \ge 0 & \text{for $i=2,3,\dots,n$.} \end{cases}
\end{align*}

Step (\ref{cvx:step7}) follows from this Proposition using the function $\tilde C(x)=x$ which satisfies Property~1. 

This completes the proof.
\end{proof}

It is much simpler to show that the marginal work is non-decreasing when the itinerary is $0^\omega$ (that is, $x\ge y_0$) or when the itinerary is $1^\omega$ (that is, $x\le y_1$).

\begin{proposition}
\label{proposition:easy-increasing}
Suppose Condition~D holds. Then 
\begin{align*}
\sum_{k=1}^\infty \beta^k \left( C(\phi_{0^k}(x)) - C(\phi_{10^{k-1}}(x)) \right)
\qquad\text{and}\qquad
\sum_{k=1}^\infty \beta^k \left( C(\phi_{01^{k-1}}(x)) - C(\phi_{1^{k}}(x)) \right)
\end{align*}
are non-decreasing functions of $x\in\IR$.
\end{proposition}
\begin{proof}
Consider the first sum. As in Lemma~\ref{lemma:integrated}, for $k\in\N$, we define
\begin{align*}
\begin{pmatrix} a_k(x) & b_k(x) \\ c_k(x) & d_k(x) \end{pmatrix} :=
\begin{pmatrix} 
	M(0^k) \begin{pmatrix} x \\ 1 \end{pmatrix} & 
	M(10^{k-1}) \begin{pmatrix} x \\ 1 \end{pmatrix} \end{pmatrix}.
\end{align*}
As $G\ge F$, the entries of $F$ are non-negative and $x\ge 0$, it follows that
\begin{align}
\label{eq:abcd-order}
\begin{pmatrix} b_k(x)-a_k(x) \\ d_k(x)-c_k(x) \end{pmatrix}
= F^{k-1} (G-F) \begin{pmatrix} x \\ 1 \end{pmatrix} \ge 0.
\end{align}

If $C$ satisfies Condition C1, then for any $k\in\N$,
\begin{align*}
\frac{d}{dx} \left( C(\phi_{0^k}(x)) - C(\phi_{10^{k-1}}(x)) \right) &=
\frac{1}{c_k(x)^2} C'\left( \frac{a_k(x)}{c_k(x)} \right) - 
\frac{1}{d_k(x)^2} C'\left( \frac{b_k(x)}{d_k(x)} \right) \\
&\ge \frac{1}{c_k(x)^2} C'\left( \frac{b_k(x)}{c_k(x)} \right) - 
\frac{1}{d_k(x)^2} C'\left( \frac{b_k(x)}{d_k(x)} \right) 
\ge 0.
\end{align*}
The first inequality holds as $C'$ is concave and $a_k(x) \le b_k(x)$ by~(\ref{eq:abcd-order}).
The second inequality holds as $C'(1/x) / x^2$ is non-increasing and $c_k(x) \le d_k(x)$ by~(\ref{eq:abcd-order}).

If $C$ satisfies Condition C2, then for any $k\in\N$,
\begin{align*}
\frac{d}{dx} \left( C(\phi_{0^k}(x)) - C(\phi_{10^{k-1}}(x)) \right) &=
\frac{1}{c_k(x)^2} C'(\phi_{0^k}(x)) - \frac{1}{d_k(x)^2} C'(\phi_{10^{k-1}}(x)) \ge 0.
\end{align*}
The inequality is justified as follows.
As $a\le b$ in the definition of $F,G$, we have $\phi_0(x) \ge \phi_1(x)$.
As $\phi_{0^{k-1}}(\cdot)$ is an increasing function,
it follows that $\phi_{0^k}(x) \ge \phi_{10^{k-1}}(x)$.
As $C$ is convex, it follows that $C'(\phi_{0^k}(x)) \ge  C'(\phi_{10^{k-1}}(x))$.
Furthermore, $c_k(x) \le d_k(x)$, by~(\ref{eq:abcd-order}).

Thus, if $C$ satisfies Condition C, the sum
\begin{align*}
\sum_{k=1}^\infty \beta^k \left( C(\phi_{0^k}(x)) - C(\phi_{10^{k-1}}(x)) \right)
\end{align*} 
is the sum of non-decreasing functions. Therefore this sum is non-decreasing.

The proof for the second sum is similar. This completes the proof.
\end{proof}
%
%
\section{Continuity (PCLI2, Second Part)}
\label{section:continuity}
We demonstrate Proposition~\ref{proposition:continuous} which shows that
the marginal productivity index $\lambda$ for systems satisfying Condition~D
is a continuous function. 

\begin{definition}
For $k\in\Z_+$ and $x\in\IR$, let $f_k(x) := (\pi(x)^\omega)_{1:k}$, where
$\pi(x)$ is the $x$-threshold word for $\phi_0,\phi_1$.
\end{definition}

\begin{lemma}
\label{lemma:jump}
Suppose $k\in\N$ and $d\in\R_+$ with $f_k(d) \ne f_k(d^+)$.
Then $\lambda(d) = \lambda(d^+).$
\end{lemma}
\begin{proof}
For $x\in \R_+$, let $\pi(x)$ be the $x$-threshold word and let $s(x)$ be the rate of $\pi(x)$.
The rate $s(x)$ is a non-increasing function and our characterisation of 
the $x$-threshold word shows that we can find a range of $x$ corresponding to 
the word of rate $q$ for any $q\in [0,1]$. 
So, Lemma~\ref{lemma:nprefix} implies that 
$f_k(d) \ne f_k(d^+)$ if and only if $d$ is the upper fixed point of a Christoffel
word of length at most $k$. 
That is, if $d=y_1$ or $d=y_{10b}$ for some Christoffel word $0b1$ with $\abs{0b1} \le k$.

Let $S(w,x) := \sum_{n=1}^{\abs{w}} \beta^{n-1} C(\phi_{w_{1:n}}(x))$
for any word $w$. 

Say $d=y_{10b}$. Let $(0a1,0c1)$ be the Christoffel pair for $0b1$.
Then $\pi(d)=0b1$ and by going left in the Christoffel tree then repeatedly turning right
we get $\pi(d^+)=0a1(0b1)^\omega$.
But $(0a1,0b1)=(0a1,0a10c1)$ is also a Christoffel pair
and as $0a1, 0b1, 0a10b1$ are Christoffel words, $a,b,a10b$ are palindromes. 
So $a10b = b01a$.
Repeated application of this result gives
$a(10b)^\omega = b01a(10b)^\omega = (b01)^\omega a$.
Thus putting $m:=\abs{0b1}$ and noting that $\phi_{10b}(d)=d$ gives
\begin{align*}
(1-\beta) \lambda(d^+) 
&= S(01a1(0b1)^\omega,d) - S(10a1(0b1)^\omega,d)\\ 
&= S((01b)^\omega,d) - S(10b(01b)^\omega,d) \\
&= S((01b)^\omega,d) - S(10b,d) - \beta^m S((01b)^\omega,\phi_{10b}(d)) \\
&=(1-\beta^m) S((01b)^\omega,d) - (1-\beta^m) S((10b)^\omega,d) \\
&= (1-\beta) \lambda(d) .
\intertext{
Now say $d=y_1$ then $\pi(d)=1$ and $\pi(d^+)=01^\omega$. 
Then} 
(1-\beta) \lambda(d^+) 
&= S(01^\omega,d)-S(101^\omega,d) \\
&= S(01^\omega,d)-S(1,d)-\beta S(01^\omega,\phi_1(d)) \\
&= (1-\beta) S(01^\omega,d)-(1-\beta) S(1^\omega,d) \\
&= (1-\beta) \lambda(d).
\end{align*}
This completes the proof.
\end{proof}

\begin{lemma}
\label{lemma:prefix}
Suppose Condition~D holds, that $k\in\N$ and $0\le x \le y$ with $f_k(x)=f_k(y)$. 
Let $K:=\sup_{z\in\{\phi_1(x),\phi_0(y)\}} C'(z)$ 
where $C'(\cdot)$ is the derivative of $C(\cdot)$.
Then
\begin{align*}
\abs{\lambda(x) - \lambda(y)}  
\le \frac{K (3 \beta^{k}(y+1) + 2(y-x))}{(1-\beta)^2}.   
\end{align*}
\end{lemma}

\begin{proof}
For $z\in\R_+$, let $\pi(z)$ be the $z$-threshold word.
Define the words $p,s$ and $s'$ by $\pi(x)=0ps1$ and $\pi(y)=0ps'1$ where $\abs{0p} = k$.
Let
\begin{align*}
a_n &:= C(\phi_{((01ps)^\omega)_{1:n}}(x)) - C(\phi_{((10ps)^\omega)_{1:k}}(x)), &
e &:= (1-\beta^{\abs{0ps1}})/(1-\beta),\\
b_n &:= C(\phi_{((01ps')^\omega)_{1:n}}(y)) - C(\phi_{((10ps')^\omega)_{1:k}}(y)), &
f &:= (1-\beta^{\abs{0ps'1}})/(1-\beta).
\end{align*}
Then the simple bounds
\begin{align*}
\sup_{m\ge 1} \abs{a_m}  &\le K(y+1),  &
\sup_{m\ge 1} \abs{a_m-b_m} & \le 2 K(y+1), &
\sup_{m\le k} \abs{a_m-b_m} \le 2 K (y-x) 
\end{align*}
follow from the facts that:
(i) the lowest and highest points on the $z$-threshold orbit are $\phi_1(z)$ and $\phi_0(z)$
(by Lemma~\ref{lemma:orbit});
(ii) $y \ge x$; 
(iii) function $C(\cdot)$ is non-decreasing and either convex or concave;
(iv) function $\phi_w(\cdot)$ is non-expansive for any word $w$;
(v) for any $z\in\R_+$, $\phi_0(z) \le z+1$ and $\phi_1(z) \ge 0$.

Also, as $\abs{0p}=k$ it follows that $\abs{e-f} \le \beta^k/(1-\beta).$
Therefore
\begin{align*}
\beta \abs{\lambda(x) - \lambda(y)}  \hspace{-1cm}&\hspace{1cm}\\
&=\abs{(e-f) \sum_{n=1}^\infty \beta^n a_n 
+ f \sum_{n=1}^k \beta^n (a_n-b_n) + f \sum_{n=k+1}^\infty \beta^n (a_n-b_n)} \\
&\le \abs{e-f} \sum_{n=1}^\infty \beta^n 
\sup_{m\ge 1} \abs{a_m}  
+ f \sum_{n=1}^k \beta^n \sup_{m\le k} \abs{a_m-b_m} + f \sum_{n=k+1}^\infty \beta^n \sup_{m>k} \abs{a_m-b_m} \\
&\le \frac{\beta^k}{1-\beta} \frac{\beta}{1-\beta} K (y+1)
+ \frac{1}{1-\beta} \frac{\beta}{1-\beta} 2 K (y-x) + \frac{1}{1-\beta} \frac{\beta^{k+1}}{1-\beta} 2 K (y+1) 
\end{align*}
which rearranges to the inequality claimed.
\end{proof}

\begin{proposition}
\label{proposition:continuous}
Suppose Condition~D holds. Then the marginal productivity index $\lambda(s)$ is a continuous function of $s \in \IR$.
\end{proposition}
\begin{proof}
We show that for any $\epsilon>0$ 
there is a $\delta>0$ such that $\abs{\lambda(x)-\lambda(y)} < \epsilon$ for any $x,y$ in the domain of $\lambda(\cdot)$ with $\abs{x-y}<\delta$.
Without loss of generality we assume that $y\ge x$. 

For $k\in\N$, let $l_k$ be the distance between the closest pair of discontinuities of $f_k(\cdot)$. 
By Lemma~\ref{lemma:nprefix}, 
these discontinuities are at the upper fixed points of Christoffel 
words of length at most $k$.
But the upper fixed points of distinct Christoffel words are distinct.
Therefore $l_k>0$.
Also, the words $01^{k-1}$ and $01^k$ have rates that are adjacent in the Farey sequence $F_k$.
But the sequence $y_{101^{k-1}}$ converges to $y_1$.
Thus $l_k \le y_{101^{k-1}}-y_{101^k}$ 
is a non-increasing function of $k$ that converges to 0. 
Therefore, for any $\epsilon>0$ and $x$ in the domain of $\lambda(\cdot)$
it is always possible to select a $k < \infty$ such that
\begin{align*}
\frac{3 \beta^{k}(x+l_k+1) + 2 l_k}{(1-\beta)^2}\left( \sup_{z\in\{\phi_0(x+l_k),\phi_1(x)\}} C'(z) \right) < \frac{\epsilon}{2}
\end{align*}
where $C'(\cdot)$ is the gradient of $C(\cdot)$ (which exists as Condition~C is satisfied).

Let $\delta = l_k$ for such a $k$ and note that the above argument shows that $\delta>0$. 
Then for $0\le x \le y \le x+\delta$, 
the definition of $l_k$ shows that $f_k(\cdot)$ has at most one discontinuity in $(x,y]$.
If there is no discontinuity, let $d$ be an arbitrary point in $(x,y]$, 
otherwise let $d$ be the point of discontinuity.
Thus Lemmas~\ref{lemma:jump} and~\ref{lemma:prefix} give
\begin{align*}
\abs{\lambda(x)-\lambda(y)} 
&\le \abs{\lambda(x)-\lambda(d)} + \abs{\lambda(d)-\lambda(d^+)} + \abs{\lambda(d^+)-\lambda(y)} < \frac{\epsilon}{2} + 0 + \frac{\epsilon}{2} = \epsilon.
\end{align*}
This completes the proof.
\end{proof}

%
%
\section{Radon-Nikodym Condition (PCLI3)}
\label{section:radon-nikodym}
We demonstrate that PCLI3 holds provided that $s$-threshold policies 
result in itineraries that are mechanical words.
The argument is based on the fact that the work-to-go and cost-to-go
are discrete measures, as defined just below, 
because the number of factors of mechanical 
words is polynomially bounded, by Theorem~\ref{theorem:xs-characterisation}. 

\begin{definition}
Let $\mu$ be a signed measure defined on the Lebesgue measurable sets of $\R$
and taking values in $[-\infty,\infty]$. 
Then measure $\mu$ is \textbf{discrete} if there is a 
countable set $\mathcal{S} = \{(a_1,s_1), (a_2,s_2), \dots \}$ of pairs of
real numbers such that
\begin{align*}
\mu(\mathcal{X}) = \sum_{(a,s)\in \mathcal{S}} a \I{ s\in \mathcal{X}} 
\end{align*}
for any Lebesgue measurable set $\mathcal{X}$ of $\R$.
\end{definition}

\begin{proposition}
\label{proposition:PCLI3}
Suppose $\langle \IR, C, \phi_0, \phi_1, \beta \rangle$ satisfy Condition~D. 
Then PCLI3 holds.
\end{proposition}
\begin{proof}
Let $\mathcal{D}_t(x)$ be the set of discontinuities of $A_t(x;s)$ as a function of $s\in \IR$.
As $A_t(x;s)$ is in $\{0,1\}$ for any $s\in \IR$, it follows that 
\begin{align*}
A_t(x;s^+) = A_t(x;s_0^-) + \sum_{d\in \mathcal{D}_t(x)} (A_t(x;d^+)-A_t(x;d^-)) \I{d \in [s_0,s]}
\end{align*}
for any $s_0 \le s$ with $s_0 \in \IR$. Thus the definition of the work-to-go gives
\begin{align*}
g(x;s^+) = g(x;s_0^-) + \sum_{t=0}^\infty \sum_{d\in \mathcal{D}_t(x)}  \beta^t (c(A_t(x;d^+))-c(A_t(x;d^-))) \I{d \in [s_0,s]}.
\end{align*}
As $\card(\mathcal{D}_t(x)) \le p(t)$ for some polynomial function $p(t)$ by Theorem~\ref{theorem:xs-characterisation} and $A_t(x;s)\in \{0,1\}$,
the series on the right-hand side of this expression is absolutely summable.
Indeed,
\begin{align*}
 \sum_{t=0}^\infty \sum_{d\in\mathcal{D}}  \abs{\beta^t (c(A_t(x;d^+))-c(A_t(x;d^-)))
\I{d \in [s_0,s]}} \le \abs{c(1)-c(0)} \sum_{t=0}^\infty \beta^t p(t) < \infty
\end{align*}
where $\mathcal{D} := \cup_{t=0}^\infty \mathcal{D}_t(x)$. 
Therefore Fubini's theorem gives
\begin{align*}
g(x;s^+) &= g(x;s_0^-) + \sum_{d\in \mathcal{D}} a_g(d) \I{d \in [s_0,s]} \\
a_g(d) &:= \sum_{t=0}^\infty \beta^t (c(A_t(x;d^+))-c(A_t(x;d^-))) 
\end{align*}
which corresponds to a discrete measure. 

Now for any $d\in [s_0,s]$, the sequence of states $X_t(x;d)$ only depends on $d$
via the sequence of actions $A_{0:t-1}(x;d)$. 
Also, an argument similar to that of Lemma~\ref{lemma:orbit} shows that both
$X_t(x;d^+)$ and $X_t(x;d^-)$ lie in the interval
\begin{align*}
[\min\{\phi_1(d),x\}, \max\{\phi_0(d),x\}]
\end{align*}
for all $t\in\Z_+$. As the cost function $C(x)$ is bounded and continuous on that interval
(as Condition~D requires that Condition~C is satisfied),
a similar argument to that given above for the work-to-go $g$ gives
\begin{align*}
f(x;s^+) &= f(x;s_0^-) + \sum_{d\in \mathcal{D}} a_f(d) \I{d \in [s_0,s]} \\
a_f(d) &:= \sum_{t=0}^\infty \beta^t (C(X_t(x;d^+))-C(X_t(x;d^-))) 
\end{align*}
which is also a discrete measure.

Noting that $X_0(s;s^+) = s$, let $\tau$ be the 
next time that $X_\tau(s;s^+) = s$ or let $\tau = \infty$ if there is no such time. 
Thus $X_t(s;s^+)$ is a periodic sequence if $\tau < \infty$.
From the definition of the policies it follows that 
$X_t(s;s)=X_t(s,1;s)$ and $A_t(s;s)=A_t(s,1;s)$ 
for $t=0, 1, \dots$, so that
\begin{align}
\label{eq:gss1}
g(s;s) &= \sum_{t=0}^\infty \beta^t c(A_t(s;s)) = g(s,1;s) .
\end{align}
Using the definition of $g(\cdot;\cdot)$, it follows that
\begin{align*}
g(s;s^+) &= \sum_{t=0}^\infty \beta^t c(A_t(s;s^+)) \\
&= \sum_{t=0}^{\infty} \beta^t c(A_t(s,0;s)) + \beta^\tau \sum_{t=0}^\infty \beta^t (c(A_t(s;s^+))-c(A_t(s,1;s)))  \\
&= g(s,0; s) + \beta^\tau (g(s;s^+) - g(s,1;s))
\end{align*}
so that
\begin{align}
\label{eq:gss}
g(s;s^+) &= \frac{g(s,0;s)-\beta^\tau g(s,1;s)}{1-\beta^\tau} .
\end{align}
Let $\tau_1$ be the first time that $X_{\tau_1}(x;s^+)=s$
or $\tau_1 = \infty$ if there is no such time.
For $t=0,\dots,\tau_1-1$, the definition of the policies then gives 
$X_t(x;s^+) = X_t(x;s)$ and $A_t(x;s^+) = A_t(x;s)$.
Thus
\begin{align}
g(x;s)-g(x;s^+) &= \beta^{\tau_1} (g(s;s)-g(s;s^+)) \nonumber\\
&= \beta^{\tau_1} \left( g(s,1;s) - \frac{g(s,0;s)-\beta^\tau g(s,1;s)}{1-\beta^\tau} \right) \nonumber\\
&= \frac{\beta^{\tau_1}}{1-\beta^\tau} \left( g(s,1;s) - g(s,0;s) \right) \nonumber\\
\label{eq:gfromw}
&= \frac{\beta^{\tau_1}}{1-\beta^\tau} w_s(s)  
\end{align}
where the second equality follows from (\ref{eq:gss}) and (\ref{eq:gss1})
and the last from the definition of the marginal work $w(\cdot;\cdot)$.

A similar argument for the cost-to-go gives
\begin{align*}
f(x;s^+)-f(x;s) &= \frac{\beta^{\tau_1}}{1-\beta^\tau} c_s(s) .
\end{align*}
Combining this equation with (\ref{eq:gfromw}) and recalling that $\lambda(s) = c_s(s)/w_s(s)$ gives
\begin{align}
\label{eq:ffgg}
f(x;s^+)-f(x;s) &= - \lambda(s)(g(x;s^+)-g(x;s)) .
\end{align}
Now by definition the marginal cost and marginal work are
\begin{align}
\label{eq:cxwx}
c_x(s) &= \beta( f(\phi_0(x);s) - f(\phi_1(x);s)) \\
w_x(s) &= 1 + \beta( g(\phi_1(x);s) - g(\phi_0(x);s)) . \nonumber
\end{align}
Combined with (\ref{eq:ffgg}) these give
\begin{align*}
c_x(s^+) - c_x(s) 
&= \beta( f(\phi_0(x);s^+) - f(\phi_1(x);s^+) -  f(\phi_0(x);s) + f(\phi_1(x);s) ) \\
&= -\beta \lambda(s) ( g(\phi_0(x);s^+) - g(\phi_1(x);s^+) -  g(\phi_0(x);s) + g(\phi_1(x);s) ) \\
&= \lambda(s) (w_x(s^+) - w_x(s) ).
\end{align*}
As $f,g$ are discrete measures, it follows that $c_x, w_x$ are also discrete measures,
so the Lebesgue-Stieltjes integral of this expression over any interval $[a,b) \subseteq \IR$ is
\begin{align*}
c_x(b^-) - c_x(a^-) = \int_{[a,b)} \lambda \ dw_x.
\end{align*}
Noting that $f(x;s)$ is a \caglad{ }function of $s$ and 
$\phi_0, \phi_1$ are continuous, 
it follows from (\ref{eq:cxwx}) 
that $c_x(s)$ is also a \caglad{ }function of $s$, so that
the left-hand side of this expression is $c_x(b) - c_x(a)$.
Thus Condition PCLI3 holds.
\end{proof}

%
%
\section{Numerical Experiments}
\label{section:numerical}
We discuss algorithms for computing the Whittle index 
given in Theorem~\ref{Indexability},
we present closed-form expressions for that index
and compare the performance of the Whittle index policy with two previously-proposed heuristics.

\subsection{Approximating the Index for Discount Factor \texorpdfstring{$\beta\le 0.999$}{Under 0.999}}
Truncating the sums defining the marginal productivity index $\lambda(x)$
after a suitably large number of terms $T$ suggests the approximation
\begin{align}
\label{eq:approximation}
\hat\lambda(x) = 
\frac{
	\sum_{t=0}^T \beta^t \big( 
		C(X_t(x,0;x)) - C(X_t(x,1;x)) 
	\big)
}{
	\sum_{t=0}^T \beta^t \big( 
		A_t(x,1;x) - A_t(x,0;x) 
	\big) } .
\end{align}
Assuming accurate calculation of the terms in the numerator and denominator,
as well as continuity of the cost function $C$,
this approximation requires $O(T)$ basic arithmetical and comparison operations, 
and setting $T$ to $\Omega(\log(\epsilon)/\log(\beta))$ guarantees absolute errors in the numerator and denominator of $O(\epsilon)$.
Of course, the constants hidden by the $O(\cdot)$ or $\Omega(\cdot)$ depend on 
the detailed properties of $C$ and choice of $x$.

However, this approximation faces a potential complication.
Indeed, some of the iterates
$X_t(x,a;x)$ may be so close to the threshold $x$ that an arbitrarily
small tolerance is required to correctly decide whether $X_t(x,a;x) \ge x$.
This might be problematic as errors in such decisions can result
in large changes to the numerator and denominator of this approximation.

\citet{Dance15} overcame this problem by constraining the sequence of
decisions to correspond to mechanical words.
This resulted in a polynomial-time algorithm for approximating the index  $\lambda(x)$
using variable-precision arithmetic.
Specifically, suppose that basic arithmetic operations to tolerance $2^{-m}$ on positive numbers less than $2^n$ require at most $M(n+m)$ operations,
and that $C(x) = x$.
Then, the approximate index $\hat\lambda(x)$ output by their algorithm 
satisfies $\abs{\hat\lambda(x)-\lambda(x)} < \epsilon$ when 
$x < 2^n$ and $\epsilon > 2^{-m}$ and is computed in $O((n+m)^3 M(n+m))$
operations. 
Nevertheless, that algorithm requires the tabulation of the fixed points
of all Christoffel words of at most a given length. 

\begin{figure}
\begin{center}
\includegraphics[viewport=3mm 95mm 210mm 200mm,clip=true,width=0.8\textwidth]{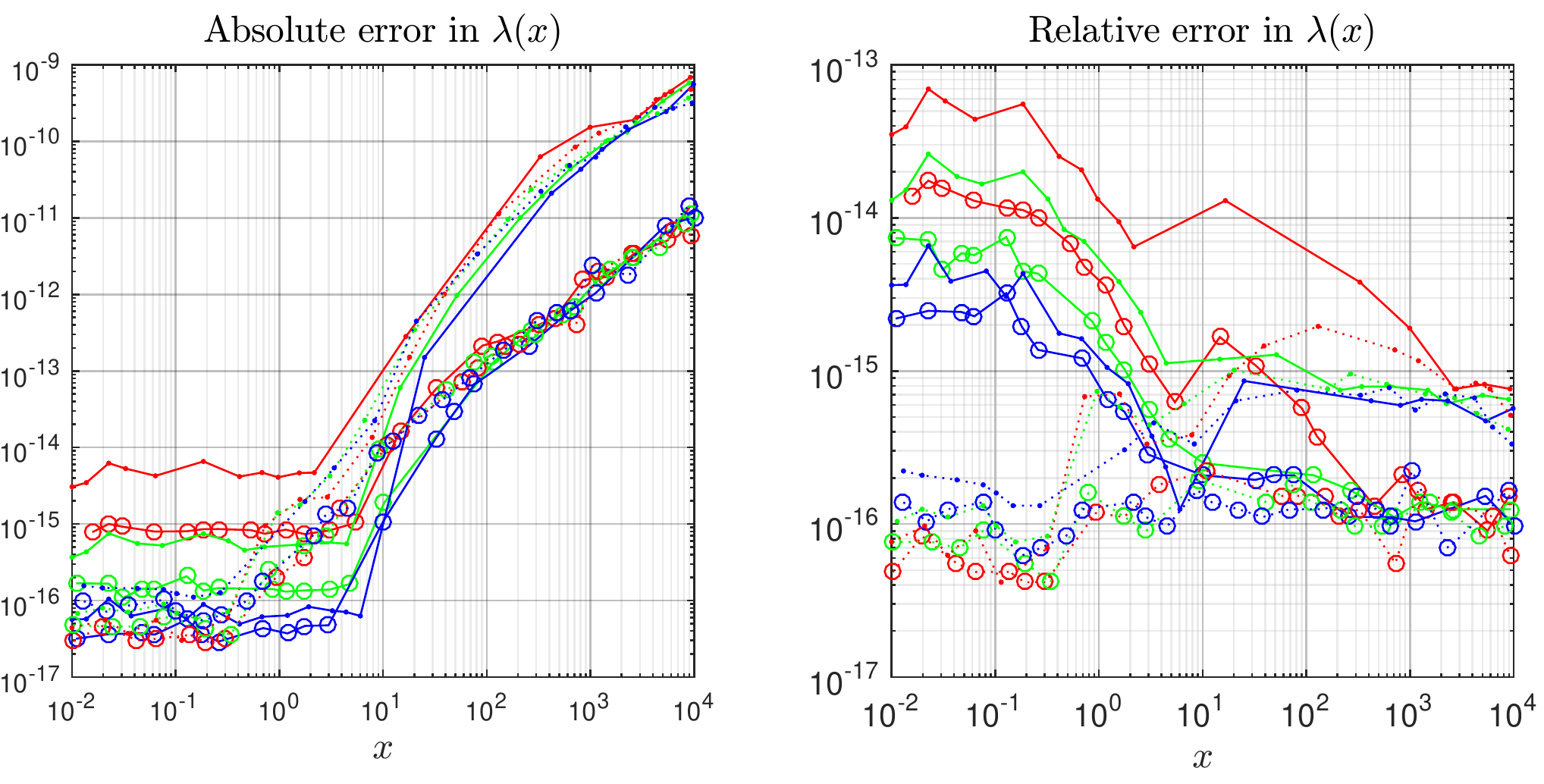}
\end{center}
\caption{Errors in approximating the index. 
The cost $C(x)$ corresponds to variance ($C(x)=x$, red), 
to entropy ($C(x)=\log(x)$, green) 
or to negative precision ($C(x) = -1/x$, blue). 
The discount factor $\beta$ is either $0.9$ (open circles) or $0.999$ (filled circles).
The map-with-a-gap has $\phi_0(x) = x+1$ and $\phi_1(x) = 1/(a_1 + 1/(x+1))$
with $a_1$ equal to $0.001$ (solid lines) or $1$ (dotted lines). }
\label{fig:approximation1}
\end{figure}

Here we suggest that such tabulation is an unnecessary expense,
and conjecture that standard floating-point approximations of the 
decision sequences $A_t(x,a;x)$ correspond to mechanical words 
(at least for the vast majority of floating point values of $x$).
Perhaps such a conjecture might be proven by extending
results by \citet{Kozyakin03}. 
His results concern mappings $\phi_0,\phi_1$
which are strictly increasing but potentially discontinuous.
In the floating-point case one would only require those mappings to be
non-decreasing and piecewise constant, 
but might perhaps impose additional conditions.

Rather than attempting to prove such a conjecture here,
we simply evaluate the accuracy of approximation (\ref{eq:approximation}).
Figure~\ref{fig:approximation1} shows the accuracy based on comparing 
double-precision and quadruple-precision implementations,
with $T=\lceil \log 10^{-17} / \log \beta \rceil$ and $T = \lceil \log 10^{-34} / \log \beta\rceil$ respectively.
As the difference between these approximations is a highly variable function of $x$,
we only show points that are local maxima of the error.
Specifically, we used a logarithmically-spaced grid 
with $10^{-2} = x_1, x_2, \dots, x_{1000} = 10^4$ and plot the error $e(x_i)$ only for points with $e(x_i) = \max_{i-20 \le j \le i+20} e(x_j)$.
The plot shows no line for $x$ less than the first such point or greater than the last such point.

The worst absolute and relative errors are below $10^{-6}$ and $10^{-11}$ respectively.
In any practical application, 
such errors would be swamped by imprecision in the time-series models.
The absolute error remains small as $x$ increases to the fixed point $y_1$ of the mapping
$\phi_1$, and then it increases due to roundoff in computing iterates of 
the map-with-a-gap for large $x$. 
Overall, the worst results are for large discount factors and for 
variance as the cost function. 

Finally, it is possible to substantially accelerate the convergence of the numerator,
for instance with Aitken acceleration~\citep{Brezinski00}, particularly
if one has high accuracy requirements.
For instance, if the $x$-threshold word has period $n$, 
one may accumulate $n$ terms of the sum at a time and apply
acceleration methods to such partial sums. 
Having experimented with such approaches,
we find that further work is required in selecting appropriate 
termination conditions if one is interested in accuracy guarantees
for a wide range of problem instances. 
The difficulty we encountered is that \textit{two} types of linear convergence are 
going on simultaneously, namely the convergence 
due to $\phi_0$ (when $a_0>0$ or $r<1$) and $\phi_1$,
and the convergence due to $\beta$.
In such situations, what looks like a healthy stopping time to 
existing termination criteria can actually be a misleading and unhealthy prematurity.

\subsection{Approximating the Index for Discount Factor \texorpdfstring{$\beta \rightarrow 1$}{Tending to 1}}
For discount factors $\beta > 0.999$ the number of terms $T$ required for accuracy in approximation (\ref{eq:approximation}) becomes prohibitively large.
In such cases, it makes sense to Taylor expand the numerator of 
$\lambda$ as a function of $\beta$.
For brevity, we only do so here for the case $\beta \rightarrow 1$.

Suppose that the itineraries $A_t(x,a;x)$ 
appearing in the definition of the index $\lambda(x)$ correspond to a Christoffel word
with period $n$. Then, in the limit $\beta\rightarrow 1$, 
the results presented elsewhere in this paper show that the 
denominator of the index tends to $1/n$.
Also, letting $X^{a,\infty}_t := \lim_{k\rightarrow\infty} X_{kn+t}(x,a;x)$,
the numerator of index has the limit
\begin{align*}
\sum_{t=0}^\infty \beta^t \left(C(X_t(x,0;x)) - C(X_t(x,1;x))\right)
\hspace{-6cm}& \\
&\rightarrow \hspace{1cm} \sum_{t=0}^\infty \left(C(X_t(x,0;x)) - C(X^{0,\infty}_t) - C(X_t(x,1;x)) + C(X^{1,\infty}_t)\right) \\
&\hspace{4cm} + \frac{1}{n} \sum_{t=0}^{n-1} t\left(C(X^{1,\infty}_t) - C(X^{0,\infty}_t)\right).
\end{align*}
Now, the sequences $C(X_t(x,a;x)) - C(X^{a,\infty}_t)$ in this expression converge to zero,
for $C$ satisfying Condition C.
This suggests approximating $X^{a,\infty}_{kn+t}$ by $X_{Tn-n+t}(x,a;x)$ 
for a suitably large positive integer $T$ and for $k=0, 1, \dots$.
This also suggests 
approximating the first sum by truncating it after $Tn$ terms. 

While the itinerary for $x$ might have a very large and possibly infinite period $n$, 
we did not encounter such situations when tabulating the index.
If this were an issue, it is possible to find a good rational approximation to the
slope of the $x$-threshold word, for instance as in~\cite{Dance15}.

\subsection{Closed Form Expressions and Graphs}
We analyse the behaviour of the index as the cost function $C(\cdot)$
and parameters $\beta, r, a_0,a_1$ vary.

\begin{figure}
\begin{center}
\begin{tabular}{c}
\includegraphics[viewport=0mm 100mm 200mm 195mm, clip=true, width=0.75\textwidth]{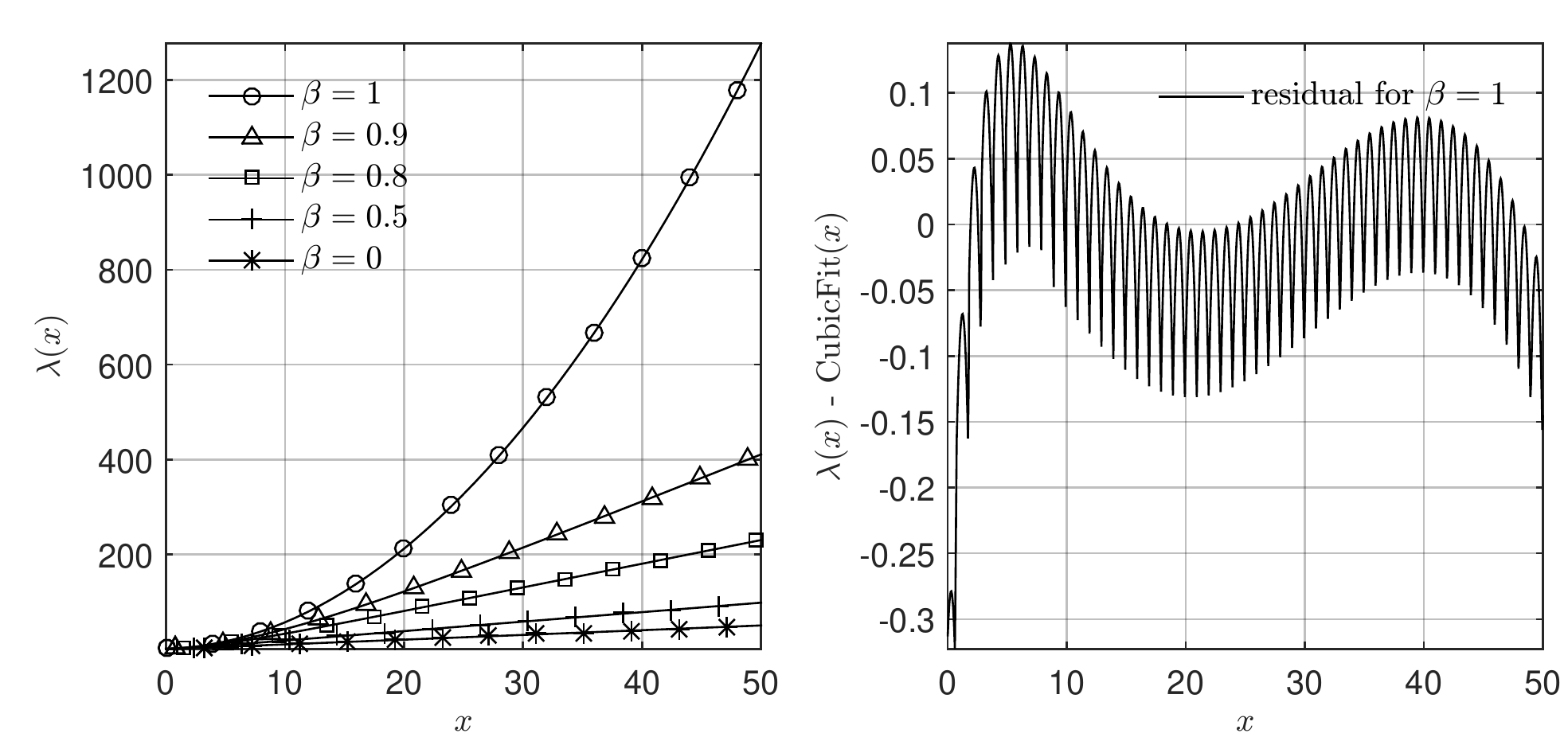}  \\
\includegraphics[viewport=0mm 100mm 200mm 195mm, clip=true, width=0.75\textwidth]{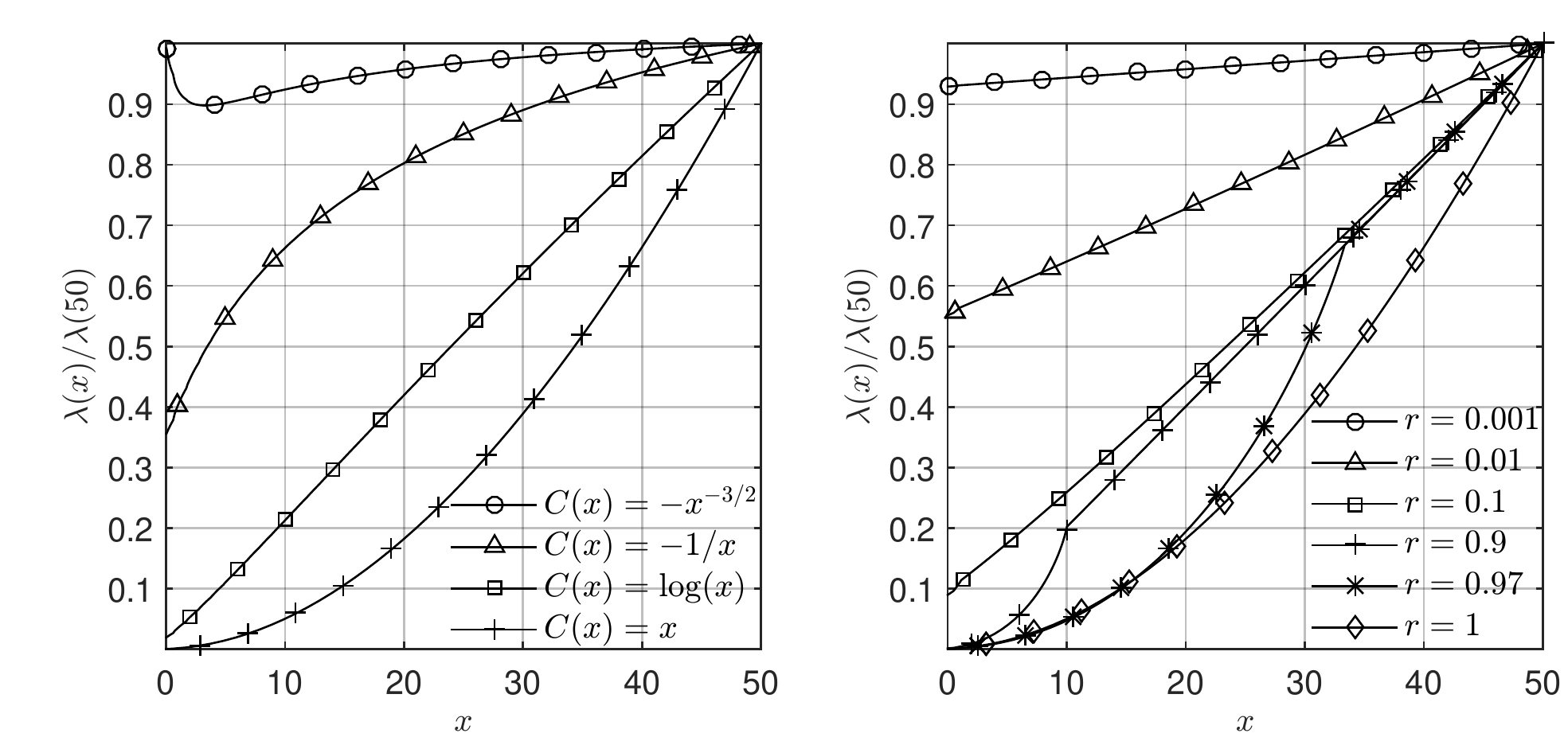}  \\
\includegraphics[viewport=0mm 100mm 200mm 195mm, clip=true, width=0.75\textwidth]{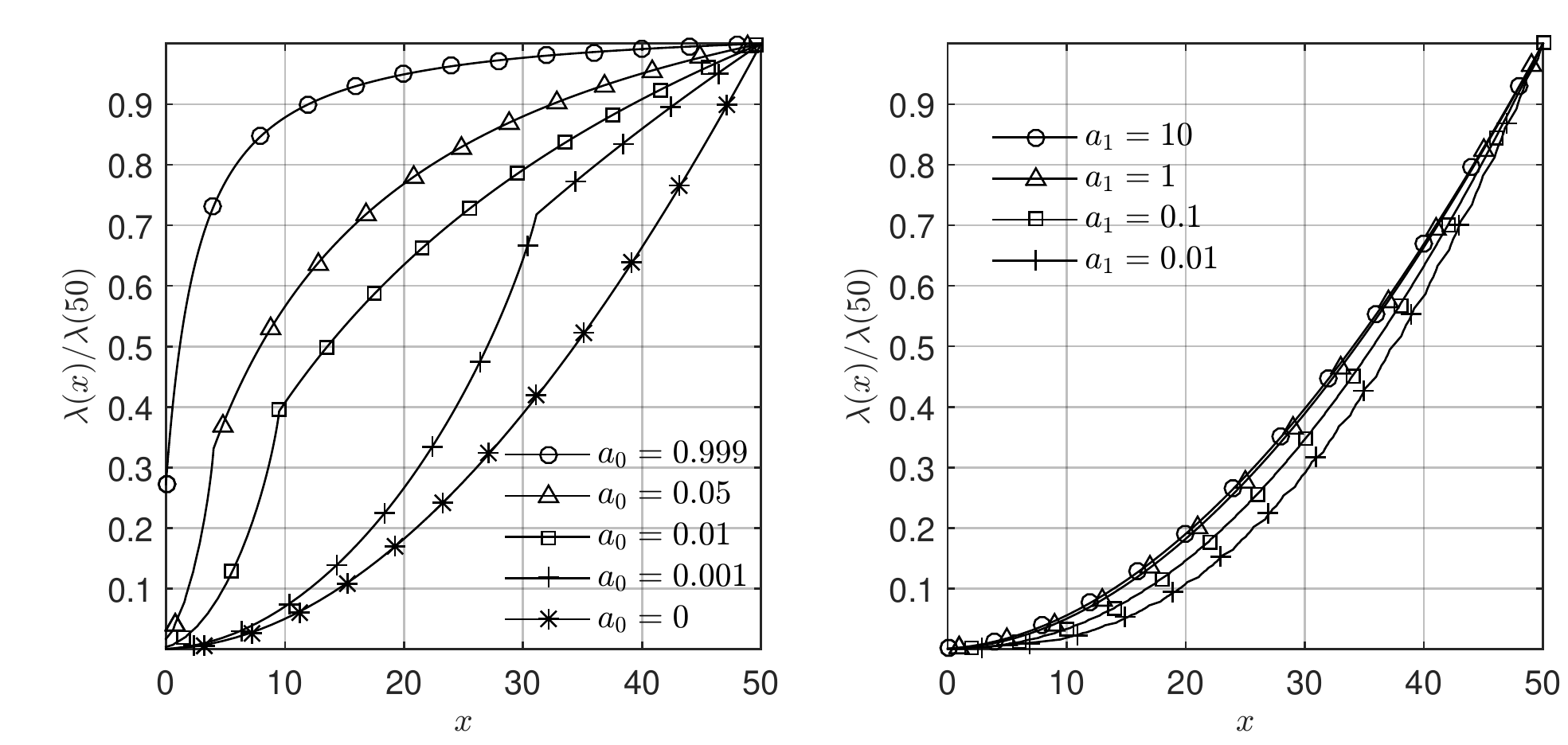}
\end{tabular}
\end{center}
\caption{Behaviour of the index. 
The index as a function of the discount factor $\beta$ (\textit{top-left})
and the residual after fitting a cubic to this curve in the case $\beta = 1$ (\textit{top-right}). 
The index, normalised by $\lambda(50)$, 
as the cost function $C(x)$, the multiplier $r$ and the observation 
precisions $a_0$ and $a_1$ are varied (\textit{other plots}). 
In all plots, all parameters (or function) other than that varied are set to
$\beta = 0.99, \ C(x) = x, \ \phi_0(x) = 1/(a_0 + 1/(rx+1))$ and $\phi_1(x) = 1/(a_1 + 1/(rx+1)),$
with $a_0 = 0, \ a_1 = 1$ and $r=1$.}
\label{fig:parametric}
\end{figure}

Given noise free observations for action $a=1$ and totally uninformative observations
for action $a=0$, it is easy to find a closed form for the index.
\begin{proposition}
Suppose the cost function is $C(x)=x$ and the precision $a_0 = 0$. Then
\begin{align*}
\lim_{a_1 \rightarrow\infty} \lambda(x) &= 
	\frac{1-\beta^{n+1}}{1-\beta} \left( rx+1 - \frac{\beta}{1-\beta^{n+1}} \left( \frac{1-(\beta r)^n}{1-\beta r} - \beta^n \frac{1-r^n}{1-r} \right) \right),
\intertext{for all $\beta \in [0,1)$, 
$r \in [0,1)$ and $x \in [0,1/(1-r))$, 
where $n := \left\lceil \frac{\log(1-(1-r)x)}{\log r} \right\rceil$. Thus}
\lim_{\beta\rightarrow 1} \lim_{r \rightarrow 1} \lim_{a_1 \rightarrow\infty} \lambda(x) &= \lceil x +1 \rceil \left( x+1-\frac{\lceil x \rceil}{2} \right) = \int_{0}^{x+1} \lceil u \rceil \ du ,
\qquad\text{for all $x\in\R_+$.}
\end{align*}
\end{proposition}
\begin{proof}
As $\lim_{a_1\rightarrow\infty} \phi_1(x) = 0$, the orbits involve the 
sequence $0, 1, r+1, r^2+r+1, \dots, r^{n-1}+\cdots+r+1$, 
where $n$ is the least integer for which $r^{n-1}+\cdots+r+1 \ge x$.
The result then follows from the definition of the index, using well-known summation formulae.
\end{proof}

Other closed forms exist, for instance in the limit $\beta\rightarrow 0$, 
whenever the cost is a polynomial function of $x$ 
or whenever the process tends to the continuous-time
process analysed in~\citet{LeNy11}.
However, even for $C(x)=x, \beta\rightarrow 1, r \rightarrow 1$ and $a_0 = 0$,
the integral in the above proposition only gives an approximation to $\lambda(x)$ with a relative
error of under $25\%$ for $a_1 > 4$.

Therefore, Figure~\ref{fig:parametric} graphs the index using the algorithms of the 
previous subsection.
Looking at these graphs, one notices that the index is 
increasing in all-but-one of the cases shown:
indeed $C(x) = -x^{-3/2}$ is \textit{not} covered by Condition C.
Also, the index has cusps at the fixed point $x = \phi_0(x)$ which are clearly 
visible as $a_0$ and $r$ vary.
Finally, the index becomes increasingly serrated as $\beta\rightarrow 1$
and $a_1 \rightarrow 0$.
One would anticipate such serrations on the basis of the above proposition as 
\begin{align*}
\int_{0}^{x+1} \lceil u \rceil \ du \ - \ \frac{1}{2} (x+1) (x+2) = \frac{1}{2} (\lceil x \rceil - x)(x - \lfloor x \rfloor)
\end{align*}
and they are visible in the plot of the residual after subtracting a cubic fit, 
for $a_1 = 1$ and $\beta \rightarrow 1$.
However, in general, the serrations have a complex non-periodic pattern and 
give a slightly ragged appearance to the plot with $a_1 = 0.01$.

\subsection{Performance Relative to Heuristic Policies}
Two heuristic policies have been commonly used for the problem of multi-sensor
scheduling. The \textit{myopic} policy observes the $m$ time series 
with the highest current cost $C_i(x_{i,t})$ and has been used in radar systems~\cite{Moran08}.
Meanwhile, the \textit{round robin} policy chooses a ordering 
of the $n$ arms and observes the time series $m$ at a time, while respecting this order.

\begin{figure}
\begin{center}
\includegraphics[viewport=30mm 90mm 170mm 210mm, clip=true, width=0.327\textwidth]{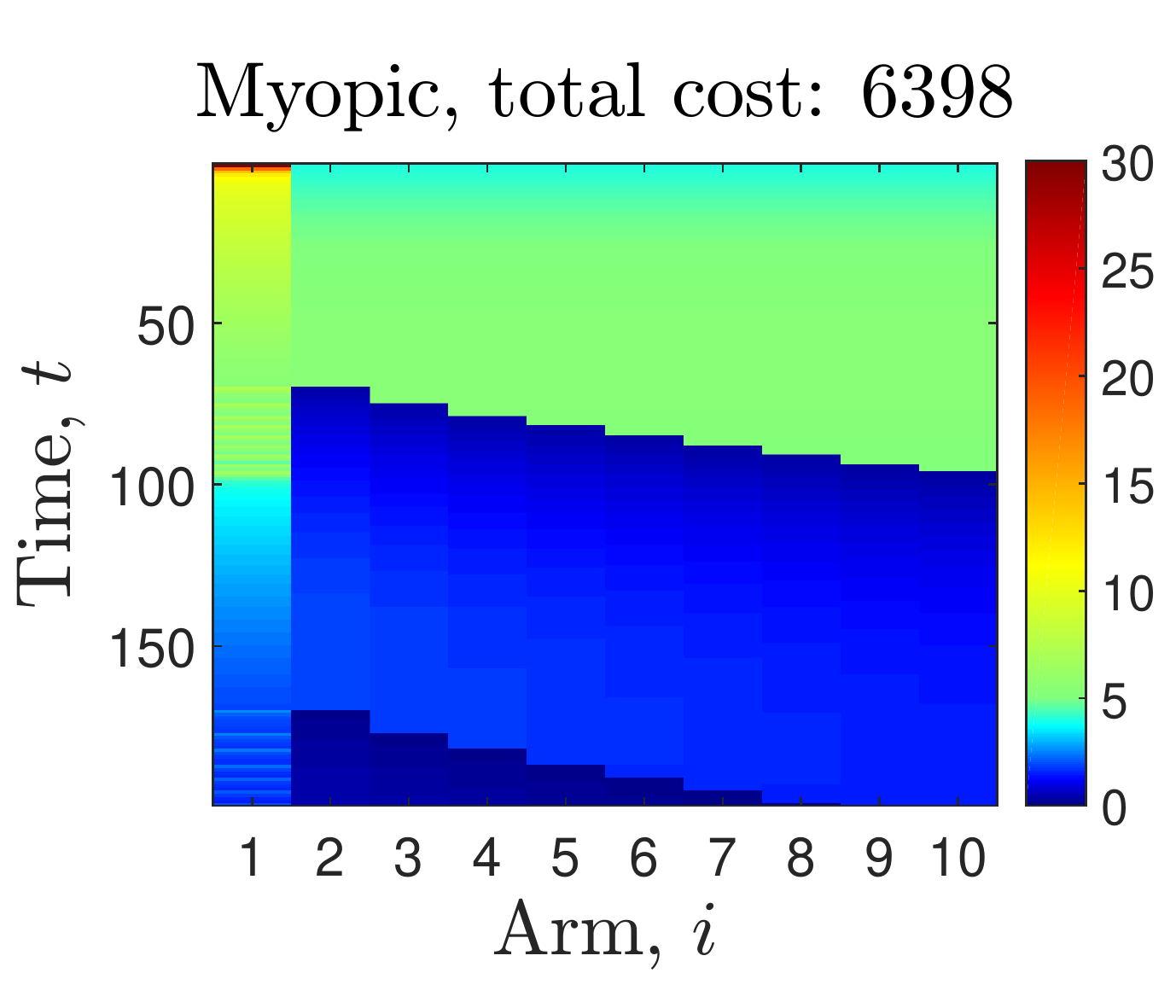} 
\includegraphics[viewport=30mm 90mm 170mm 210mm, clip=true, width=0.327\textwidth]{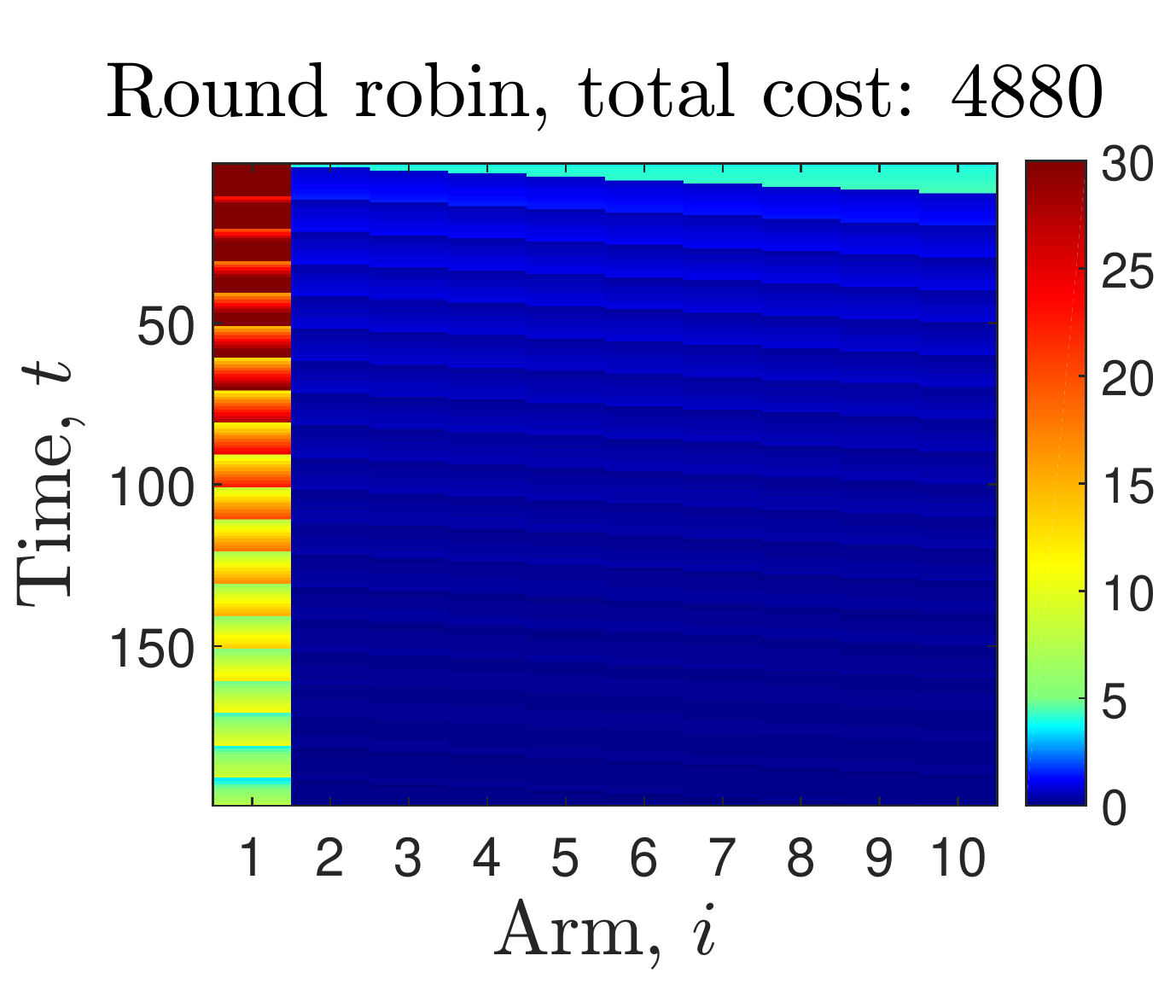} 
\includegraphics[viewport=30mm 90mm 170mm 210mm, clip=true, width=0.327\textwidth]{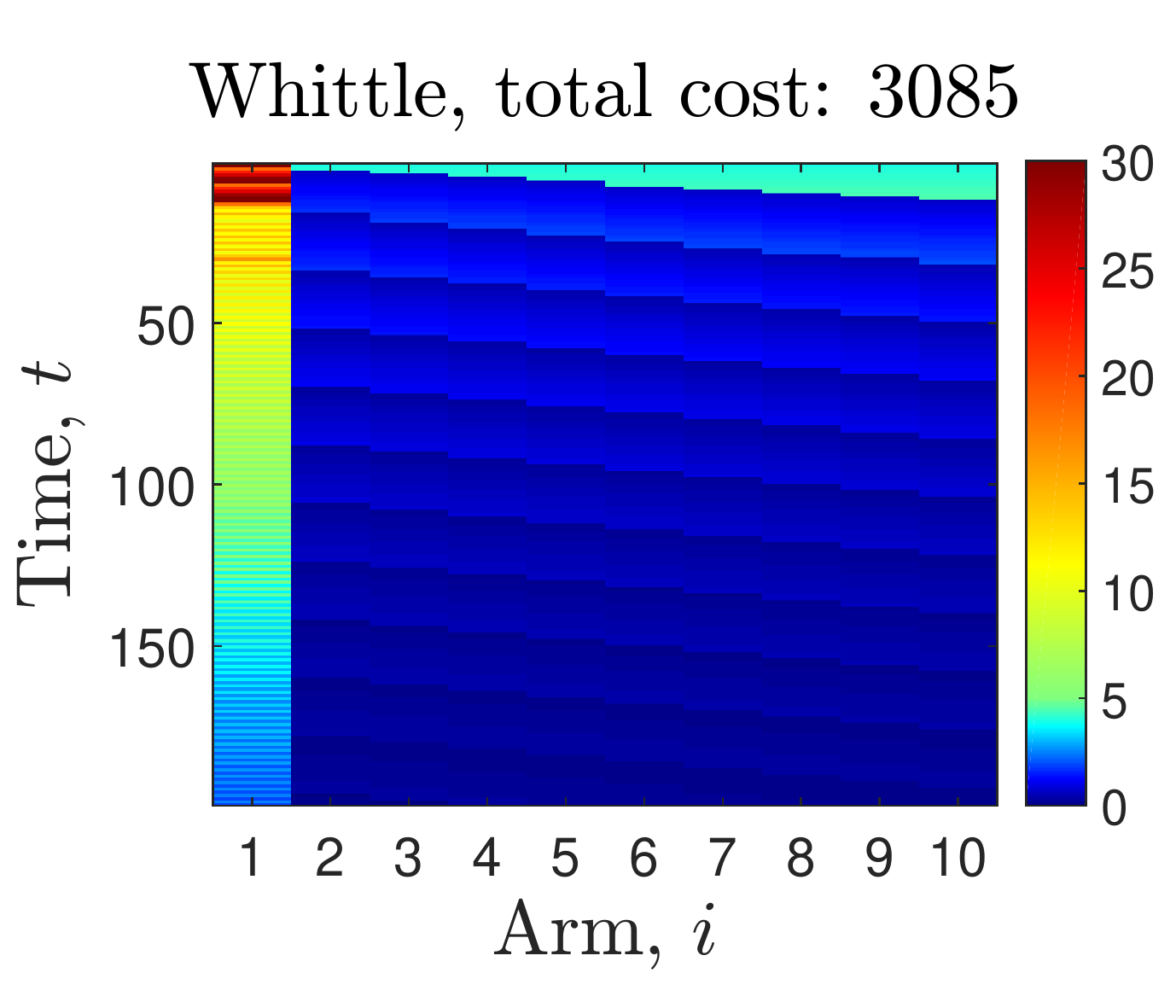}
\end{center}
\caption{Comparison of heuristic policies. Colour represents the 
variance state.}
\label{fig:heuristics}
\end{figure}

Figure~\ref{fig:heuristics} compares the costs incurred by these heuristics
in a simple scenario in which estimation errors for one of the arms (time series) are 
more expensive than for the other arms.
In detail, there are $n=10$ arms, $m=1$ observations per round,
the cost at time $t$ is $10 x_{1,t} + \sum_{i=1}^n x_{i,t}$,
observations have zero cost,
and the initial posterior variance is $x_{i,0} = 4$ for all arms. 
The arms have the same map-with-a-gap given by $\phi_0(x) = x+1$ and $\phi_1(x) = 1/(0.1 + 1/(x+1))$.

Clearly, the myopic policy is over-eager to observe arm $i=1$ and does
so at the expense of arms $i=2, \dots, 9$.
In contrast, the round robin policy makes no special effort to observe arm $i=1$
and incurs substantial expense for that time arm.
Meanwhile, the Whittle policy takes a just medium between these extremes,
and is by far the least costly policy.

%
%
\section{Further Work}
\label{section:conclusion}
This paper presented conditions under which threshold policies are optimal
for observing a single time series with costly observations.
It also explored the implications of this result by showing that it leads to
optimal policies for the linear-quadratic Gaussian problem with costly observations
and that it demonstrates the indexability of related restless bandit problems,
which were both long-standing open questions.

It would be natural to extend this work to situations where more than 
two observation actions are available, perhaps using known generalisations of 
mechanical words~\citep{Glen2009}.
There are also truly-stochastic versions of the one-dimensional problem considered here,
for instance situations where the costs depend on the posterior mean rather than just the 
posterior variance, situations where the quality of the observation is a random variable
and situations involving non-Gaussian time series.

It is also important to understand the structure of optimal 
policies for making costly observations with discrete-time Kalman filters
in \textit{multiple} dimensions.
One attack on this problem would begin by extending the verification theorem of~\citet{NinoMora15}
to multi-dimensional state spaces. 

Finally, we cannot claim the asymptotic optimality of Whittle's index policy for
the problem studied here as the results of~\citet{Verloop16} only apply to \textit{countable}
state spaces. Furthermore, little is known about the performance of
policies for restless bandits in \textit{non-asymptotic} situations involving 
finite numbers of arms and finite time horizons.
\acks{We thank Jos{\'e} Ni{\~n}o-Mora for pointing out his work on the verification theorem
and Sofia Michel for her careful reading of the manuscript.}
\appendix
%
%
\section{Itineraries as Mechanical Words}
\label{appendix:mechanical}
We begin with some elementary properties of compositions of functions satisfying Assumption~A1 and their fixed points (Subsection~\ref{sub:compositions}). This enables us to present a proof of Theorem~\ref{theorem:characterisation} which is based on the Christoffel tree. First we consider the case where itineraries correspond to Christoffel words (Subsection~\ref{sub:Christoffel-case}), then the case where itineraries correspond to Sturmian \mword{s} (Subsection~\ref{sub:Sturmian-case}) and finally we couple these results together to prove the theorem (Subsection~\ref{sub:characterisation-proof}).
We also present a proof of Theorem~\ref{theorem:xs-characterisation} about itineraries from initial points not equal to the threshold (Subsection~\ref{sub:xs-characterisation-proof}) .

\subsection{Compositions and Fixed Points}
\label{sub:compositions}
The following fundamental result about compositions is well known.
\begin{lemma}
\label{lemma:monexist}
Suppose A1 holds and that $w$ is a finite non-empty word. 
Then 
$\phi_w$ is increasing, contractive, continuous and has a unique fixed point $y_w$ on $\IR$. 
\end{lemma}

\begin{proof}
First we show that $\phi_w(x)$ is increasing, by induction on the length of
word $w$.
In the base case, as $w$ is non-empty, we suppose that $\abs{w} = 1$.
The claim then follows immediately from A1.
For the inductive step, assume $\phi_u(x)$ is increasing, 
where $w = au$ for some letter $a\in \{0,1\}$ and some finite non-empty word $u$.
Then for any $x,y\in\IR$ such that $x < y$, 
\begin{align*}
\phi_w(y) &= \phi_u(\phi_a(y)) \\
&> \phi_u(\phi_a(x)) & \text{as $\phi_a(y) > \phi_a(x)$ and $\phi_u$ is increasing} \\
&= \phi_w(x) .
\end{align*}
Therefore $\phi_w$ is increasing.

Now we show that $\phi_w(x)$ is contractive, again by induction on $\abs{w}$.
If $\abs{w} = 1$ then this follows immediately from A1.
Else, say $\phi_u(x)$ is contractive where $w = ua$ and $a\in \{0,1\}$.
Then for any $x,y\in\IR$ such that $x < y$, 
\begin{align*}
\phi_w(y) - \phi_w(x) &= \phi_a(\phi_u(y)) - \phi_a(\phi_u(x)) \\
&< \phi_u(y) - \phi_u(x) & \text{as $\phi_u(y) > \phi_u(x)$ and $\phi_a$ is contractive}\\
&< y-x & \text{as $\phi_u$ is contractive.}
\end{align*}
Therefore $\phi_w$ is contractive.

As $\phi_w$ is contractive,
for any $\epsilon>0$ and $c\in \mathcal{I}$, it follows that
\begin{align*}
\abs{\phi_w(x)-\phi_w(c)} < \abs{x-c} < \epsilon  \quad\text{for any $x\in\mathcal{I}$ with $\abs{x-c} < \epsilon$.} 
\end{align*}
Therefore $\phi_w$ is continuous.

Now we show that the fixed point $y_w$ exists, using the intermediate value
theorem applied to the function $g(x) := x - \phi_w(x)$.
First we show that $g(y_0) \ge 0$. 
Indeed, as $\phi_1$ is contractive, the definition of $y_1$ gives
\begin{align*}
y_0 - y_1 > \phi_1(y_0) - \phi_1(y_1) = \phi_1(y_0) - y_1, \quad\text{so that}\quad \phi_1(y_0) < y_0.
\end{align*}
So, it follows from the definition of $y_0$ that the upper bound 
$\psi(x) := \max\{\phi_0(x), \phi_1(x)\}$ satisfies
$\psi(y_0) = \phi_0(y_0) = y_0$.
As $\phi_u(x)$ is increasing for any finite word $u$, it follows that
\begin{align*}
\phi_w(y_0) = \phi_{w_{2:\abs{w}}} \circ \phi_{w_1}(y_0) \le \phi_{w_{2:\abs{w}}} \circ \psi(y_0)  = \phi_{w_{2:\abs{w}}}(y_0) \le \cdots \le y_0, 
\end{align*}
so that 
\begin{align*}
g(y_0) = y_0 - \phi_w(y_0) \ge y_0 - y_0 = 0.
\end{align*}
A similar argument, using the lower bound $\min\{\phi_0(x),\phi_1(x)\}$,
 gives $g(y_1) \le 0$.
In summary, $g(y_1) \le 0 \le g(y_0),$ where $y_1 < y_0$ by A1,
and $g(x)$ is continuous as $\phi_w(x)$ is continous.
So the intermediate value theorem shows that $g(y) = 0$ for some
$y \in [y_1,y_0]$. 
Therefore a fixed point $y_w$ exists on $\IR$.

Now we show that the fixed point $y_w$ is unique. Suppose both $y$ and $z$ are fixed points of $\phi_w$ with $y > z$.
This leads to the following contradiction: as $\phi_w$ is contractive we have 
\begin{align*}
\frac{\phi_w(y) - \phi_w(z)}{y-z} &< 1,
\intertext{yet as $\phi_w(y) = y$ and $\phi_w(z) = z$ we have}
\frac{\phi_w(y) - \phi_w(z)}{y-z} &= 1 .
\end{align*}
Therefore the fixed point is unique. This completes the proof.
\end{proof}


We make widespread use of the following simple Lemma. 
Given a word $w$, this Lemma gives necessary and sufficient conditions for $\phi_w(x)$ to be greater than or less than $x$.

\begin{lemma} 
\label{lemma:incdec}
Suppose A1 holds, that $x\in \IR$ and $w$ is a finite non-empty word. Then 
\begin{align*}
x < \phi_w(x) \ \Leftrightarrow \ \phi_w(x) < y_w \ \Leftrightarrow \ x < y_w \quad\text{and}\quad
x > \phi_w(x) \ \Leftrightarrow \ \phi_w(x) > y_w \ \Leftrightarrow \ x > y_w.
\end{align*}
\end{lemma}
\begin{proof}
We use Lemma~\ref{lemma:monexist} and the definition of $y_w$ throughout without further mention. 

As $\phi_w$ is increasing, if $x<y_w$ then $\phi_w(x) < \phi_w(y_w) = y_w$.
Similarly, if $x>y_w$ then $\phi_w(x)>y_w$.
Thus if $\phi_w(x)\le y_w$ then $x\le y_w$, by the contrapositive.
But if $\phi_w(x)\ne y_w$ then $x\ne y_w$, as $\phi_w$ is increasing and therefore injective.
So if $\phi_w(x)<y_w$ then $x<y_w$.
Therefore
\begin{align*}
x < y_w \quad \Leftrightarrow \quad \phi_w(x) < y_w.
\end{align*}

As $\phi_w$ is contractive, if $x<y_w$ then $\phi_w(y_w) - \phi_w(x) < y_w-x$.
As $\phi_w(y_w) = y_w$, this rearranges to give $x<\phi_w(x)$.
Similarly, if $x>y_w$ then $x>\phi_w(x)$.
Thus if $x\le \phi_w(x)$ then $x\le y_w$, by the contrapositive.
But if $x \ne \phi_w(x)$ then $x$ is not a fixed point, so $x \ne y_w$.
So if $x<\phi_w(x)$ then $x<y_w$.
Therefore
\begin{align*}
x < y_w \quad \Leftrightarrow \quad x < \phi_w(x).
\end{align*}

A similar argument shows that $x > y_w \Leftrightarrow \phi_w(x) > y_w$
and $x > y_w \Leftrightarrow x > \phi_w(x)$. 
\end{proof}


\begin{lemma}
\label{lemma:ypi}
Suppose A1 holds and $w$ is a finite word with $\abs{w}_0 \abs{w}_1>0$. Then 
\begin{align*}y_1 < y_w < y_0.\end{align*}
\end{lemma}
\begin{proof}
As $\abs{w}_0 > 0$ we have $w =: s01^q$ for some finite word $s$ and some $q\in\Z_+$.
As $y_0 > y_1$ by A1, Lemma~\ref{lemma:incdec} gives $\phi_0(y_1)>y_1$.
Thus the definition of $y_1$ and the fact that $\phi_s$ is increasing give
\begin{align*}
\phi_w(y_1) = \phi_{s01^q}(y_1) = \phi_{s0}(y_1) = \phi_s(\phi_0(y_1)) > \phi_s(y_1) \ge y_1
\end{align*}
where the last step follows by repeating the same argument. But if $\phi_w(y_1)>y_1$
then Lemma~\ref{lemma:incdec} shows that $y_w > y_1$.

A similar argument leads to the conclusion that $y_w < y_0$. 
\end{proof}


\begin{lemma}
\label{lemma:limit}
Suppose A1 holds, $x\in\IR$ and $w$ is a finite non-empty word.
Then 
\begin{align*}
\lim_{n\rightarrow\infty} \phi_{w^n}(x) = y_w.
\end{align*}
\end{lemma}
\begin{proof}
The sequence with elements $x_n := \phi_{w^n}(x)$ for $n\in\N$ 
is monotone and bounded, by Lemma~\ref{lemma:incdec}. 
So the monotone convergence theorem for sequences of real numbers
shows that $l := \lim_{n\rightarrow\infty} x_n$ exists.
But as $\phi_{w}$ is continuous, by Lemma~\ref{lemma:monexist},
the limit $l$ satisfies
\begin{align*}
\phi_{w}\big( l\big) = \phi_{w}\big(\lim_{n\rightarrow\infty} x_n\big) = \lim_{n\rightarrow\infty} \phi_{w}\big(x_n\big) = \lim_{n\rightarrow\infty} x_{n+1} = l.
\end{align*}
So $l$ is a fixed point of $\phi_{w}$. 
By Lemma~\ref{lemma:monexist}, $y_w$ is the unique such fixed point. 
\end{proof}


As a contractive function is not necessarily a contraction mapping, some additional work is required 
to prove the following result which is essential in the Sturmian case addressed in Subsection A.3.

\begin{lemma}
\label{lemma:limit-is-0}
Suppose A1 holds, $w$ is an infinite word and $y_0 \ge a > b\ge y_1$. Then
\begin{align*}
\lim_{n\rightarrow\infty} (\phi_{w_{1:n}}(a) - \phi_{w_{1:n}}(b)) = 0.
\end{align*}
\end{lemma}
\begin{proof}
Let $a_n := \phi_{w_{1:n}}(a)$ and $b_n := \phi_{w_{1:n}}(b)$ for $n \in \N$.
By Assumption A1, $\phi_{w_{n}}$ is a contractive function, so $(a_n-b_n : n\in\N)$ is a decreasing sequence
and as $a, b \in [y_1,y_0]$ Lemma~\ref{lemma:incdec} shows that this is also a bounded sequence.
Therefore the monotone convergence theorem for real-valued sequences shows that
\begin{align*}\text{$\lim_{n\rightarrow\infty} (a_n-b_n)$ exists.}\end{align*}
Now we argue that the limit is zero, by contradiction.
Assume that $a_n - b_n \ge \epsilon$ for all $n\in\N$, where $\epsilon$ is a positive real number.
Let the domain $\mathcal{D}\subset \R^2$ be
\begin{align*}
\mathcal{D} &:= \{ (h,l) \ : \ h, l \in [y_1,y_0], \ h \ge l + \epsilon \}, 
\intertext{let the functions $f_c : \mathcal{D} \rightarrow \R$ for letters $c\in\{0,1\}$ be}
f_c(h,l) &:= \frac{\phi_c(h)-\phi_c(l)}{h - l}
\intertext{where $(h,l)\in\mathcal{D}$ and define the number $q\in\R$ by}
q &:= \sup_{(h,l) \in \mathcal{D}} \left\{ \max_{c\in\{0,1\}} f_c(h,l) \right\}. 
\intertext{Now the functions $f_0, f_1$ are continuous on their domain $\mathcal{D}$ by Lemma~\ref{lemma:monexist}.
Also, the domain $\mathcal{D}$ is a non-empty, bounded and closed set.
So the extreme value theorem for functions of several variables shows that the maximum equals the supremum. 
Thus}
q &= \max_{(h,l) \in \mathcal{D}} \left\{ \max_{a\in\{0,1\}} f_a(h,l) \right\}.
\end{align*}
As $\phi_c$ is contractive for $c\in\{0,1\}$ it follows that
\begin{align*}
q &< 1
\end{align*}
and as $\phi_c$ is increasing we have $q>0$.
So the definition of $q$ and hypothesis that $a>b$ give
\begin{align*}
a_n-b_n \le q^n (a-b) 
< \epsilon \quad\text{for $n > \log((a-b)/\epsilon)/\log(1/q).$}
\end{align*}
Thus there is an $n\in\N$ with $a_n-b_n < \epsilon$.
This contradicts the assumption that $a_n-b_n \ge \epsilon$ for all $n\in\N$. 
Since $\epsilon>0$ was arbitrary, we conclude that 
\begin{align*}
\lim_{n\rightarrow\infty} (a_n-b_n) = 0.
\end{align*}
This completes the proof.
\end{proof}


\subsection{$x$-Threshold Words as Christoffel Words}
\label{sub:Christoffel-case}
We begin with some definitions that go beyond the main text.

The set $\{0,1\}^*$ consists of all finite words on the alphabet $\{0,1\}$,
including the empty string $\epsilon$.
The \textit{morphism} $\morph : \{0,1\}^* \rightarrow \{0,1\}^*$
generated by a mapping 
$s : \{0,1\} \rightarrow \{0,1\}^*$,
substitutes $s(w_k)$ for each letter $w_k$ of a word $w$,
so that 
\begin{align*}
\text{$\morph(\epsilon) = \epsilon$ and $\morph(w) = s(w_1) s(w_2) \cdots s(w_{\abs{w}})$.}
\end{align*}
We work with the morphisms 
$\uuv : \{0,1\}^* \rightarrow \{0,1\}^*$ and $\uvv : \{0,1\}^* \rightarrow \{0,1\}^*$ given by
\begin{align*}
\uuv : \begin{cases} 0 \mapsto 0 \\ 1 \mapsto 01 \end{cases}
&& \uvv : \begin{cases} 0 \mapsto 01 \\ 1 \mapsto 1 \end{cases} .
\end{align*}
Let $\circ$ denote composition of morphisms, 
for example $\uvv \circ \uuv(1) = \uvv(01) = 011.$

\begin{remark} 
These morphisms generate the Christoffel tree through pre-composition.
Say $(u,v)$ is a Christoffel pair and consider the morphism 
\begin{align*}
\mathscr{M} : \begin{cases} 0 \mapsto u \\ 1 \mapsto v \end{cases}.
\end{align*}
Then pre-composition maps the node $(u,v)$ of the Christoffel tree to its children:
\begin{alignat}{99}
\nonumber
&(\mathscr{M}\circ \uuv(0), \ &\mathscr{M}\circ \uuv(1)) &= (\mathscr{M}(0),&\mathscr{M}(01)) &= (u,uv)& \\
\label{eq:precomp}
&(\mathscr{M}\circ \uvv(0), \ &\mathscr{M}\circ \uvv(1)) &= (\mathscr{M}(01),&\mathscr{M}(1)) &= (uv,v)& .
\end{alignat}
\end{remark}

Now let us give a simple upper and lower bound on the $x$-threshold orbit.

\begin{lemma}
\label{lemma:orbit}
Suppose A1 holds and $(x_k : k\in\N)$ is the $x$-threshold orbit. Then
\begin{align*}
x\in [y_1,y_0] \ \Rightarrow \
\phi_1(x) \le x_k < \phi_0(x) \quad\text{for $k\in\N$.}
\end{align*}
\end{lemma}
\begin{proof}
Say $z\in [y_1,y_0)$. Then Lemma~\ref{lemma:incdec} gives 
\begin{align*}
y_1 \le \phi_1(z) \le z < \phi_0(z) < y_0.
\end{align*}
An induction using this fact, immediately shows that for $k\in\N$,
\begin{align*}
y_1 \le x_k < y_0 .
\end{align*}

Now we prove the claim by induction with hypothesis
\begin{align*}
H_k : \quad \phi_1(x) \le x_k < \phi_0(x)
\end{align*}
for $k\in\N$.
The base case $H_1$ is true as $x_1 = \phi_1(x)$ by definition of the $x$-threshold orbit.
For the inductive step, say $H_k$ is true for some $k\in\N$.
Then there are two cases to consider: $x_k\in [\phi_1(x),x)$ and $x_k\in [x,\phi_0(x))$.
If $x_k\in [\phi_1(x),x)$ then $x_{k+1} = \phi_0(x_k)$ and
\begin{align*}
\phi_1(x) &\le x_k \\
&< \phi_0(x_k) && \text{by Lemma~\ref{lemma:incdec} as $x_k < y_0$}\\
&< \phi_0(x) && \text{as $x_k < x$ and $\phi_0(\cdot)$ is increasing,}
\intertext{so $H_{k+1}$ is true.
If $x_k\in [x,\phi_0(x))$ then $x_{k+1} = \phi_1(x_k)$ and}
\phi_1(x) &\le \phi_1(x_k) &&\text{as $x\le x_k$ and $\phi_1(\cdot)$ is increasing} \\
&\le x_k && \text{by Lemma~\ref{lemma:incdec} as $y_1 \le x_k$}\\
&< \phi_0(x), 
\end{align*}
so $H_{k+1}$ is true.
This completes the proof.
\end{proof}


Now we show that $x$-threshold words are Christoffel words in three important
special cases and then we find the general conditions on $x$ for which $x$-threshold words are Christoffel words
(Lemma~\ref{lemma:Ccase}).


\begin{lemma}
\label{lemma:basecase}
Suppose $\pi$ is the $x$-threshold word for $\phi_0,\phi_1$ satisfying A1. Then 
\begin{alignat*}{99}
\pi &= 1		  & \ \Leftrightarrow \ x&\le y_1 \\
\pi &= 01 	  & \ \Leftrightarrow \ x&\in [y_{01},y_{10}]\\
\pi &= 0 		  & \ \Leftrightarrow \ x&\ge y_0.
\end{alignat*}
\end{lemma}
\begin{proof} 
If $\pi = 1$ then the definition of the $x$-threshold word shows that $\phi_1(x) \ge x$. Therefore Lemma~\ref{lemma:incdec} shows that $x \le y_1$.

If $x\le y_1$ then Lemma~\ref{lemma:incdec} shows that $x\le \phi_{1^n}(x) \le y_1$ for all $n\in \N$. Therefore $\pi = 1$.

If $\pi = 01$ then the next letter of $\pi^\omega$ after $(01)^n$ is $0$ for any $n\in\Z_+$.
So $x > \phi_{(01)^n}(\phi_1(x))$.
Therefore Lemma~\ref{lemma:limit} gives
\begin{align*}
x \ge \lim_{n\rightarrow\infty} \phi_{(01)^n} (\phi_1(x)) = y_{01}.
\end{align*}

If $\pi = 01$ then $\pi_2 = 1$. 
So $x \le \phi_{10}(x)$. 
Therefore Lemma~\ref{lemma:incdec} gives 
\begin{align*}
x \le y_{10} .
\end{align*}

If $x \ge y_{01}$ then $x > y_1$ by Lemma~\ref{lemma:ypi}.
So $\phi_1(x) < x$ by Lemma~\ref{lemma:incdec}.
As $\phi_{01}$ is increasing by Lemma~\ref{lemma:monexist}, it follows that for any $n\in\Z_+$, 
\begin{align*}
\phi_{(01)^n}(\phi_1(x)) < \phi_{(01)^n}(x) \le x 
\end{align*}
where the second inequality follows from Lemma~\ref{lemma:incdec} as $x\ge y_{01}$.
Therefore, if $\pi^\omega$ begins with $(01)^n$ then the next letter is $0$.

If $x \le y_{10}$ then Lemma~\ref{lemma:incdec} shows that $\phi_{(10)^n}(x) \ge x$ for all $n\in\Z_+$.
Therefore, if $\pi^\omega$ begins with $(01)^n0$ then the next letter is $1$.

The proof for the case $\pi=0$ is symmetric to that for $\pi=1$. This completes the proof.
\end{proof}


\begin{lemma}
\label{lemma:0011}
Suppose $\phi_0,\phi_1$ satisfy A1, $x\in\IR$ and $\pi$ is the $x$-threshold word. Then
\begin{align*}
\left\{ \begin{aligned}
\abs{\pi}_{11} &>0 & \ \Rightarrow \ && x &< y_{01} \\
\abs{\pi}_{00} &>0 & \ \Rightarrow \ && x &> y_{10}.
\end{aligned}\right.
\end{align*}
\end{lemma}
\begin{proof}
If $\abs{\pi}_{11}>0$ then $x<y_0$ by Lemma~\ref{lemma:basecase}. 
So, either $x\le y_1$, in which case Lemma~\ref{lemma:ypi} shows that $x<y_{01}$,
or $x\in (y_1,y_0)$. 
In the latter case, let $(x_k : k\in\N)$ be the $x$-threshold orbit.
As $\abs{\pi}_{11}>0$ there exists a $k\in\N$ with $\phi_1(x_k) \ge x$
by definition of the $x$-threshold word.
Now $x_k < \phi_0(x)$ by Lemma~\ref{lemma:orbit}, 
so that $\phi_1(\phi_0(x)) > \phi_1(x_k) \ge x$ as $\phi_1$ is increasing.
But $\phi_{01}(x) > x$ implies that $x < y_{01}$ by Lemma~\ref{lemma:incdec}.

The proof for $\abs{\pi}_{00}>0$ is symmetric. This completes the proof.
\end{proof}

\begin{lemma}
\label{lemma:recursion}
Suppose $\phi_0,\phi_1$ satisfy A1. Then for $x\in [y_{10},y_0]$, 
there is a unique $z\in\IR$ with 
\begin{align*}
\phi_0(z) = x.
\end{align*}
Furthermore 
\begin{align*}
\pi(x,\phi_0,\phi_1) = \begin{cases}
\uuv(\pi(\phi_0^{(-1)}(x),\phi_0,\phi_{01})) & \text{if $x \in [y_{10},y_0]$} \\
\uvv(\pi(x,\phi_{01},\phi_{1})) & \text{if $x \in [y_1, y_{01}]$.}
\end{cases} 
\end{align*}
\end{lemma}
\begin{proof}
In the first claim, existence of $z$ follows from the intermediate value theorem, 
as $\phi_0$ is continuous by Lemma~\ref{lemma:monexist}, 
as $y_{01} \in [y_0,y_1] \subseteq \IR$ by Lemma~\ref{lemma:ypi}, 
as $\phi_0(y_{01}) = y_{10} \le x$ by definition of $y_{01}$,
and as $\phi_0(y_0) = y_0 \ge x$ by definition of $y_0$.
Uniqueness follows as $\phi_0$ is increasing.

Now say $x\in [y_{10},y_0]$ consider the claim involving the morphism $\uuv$.
Let 
\begin{align*}
V := \pi(\phi_0^{(-1)}(x),\phi_0,\phi_{01})^\omega \quad\text{and}\quad W := \pi(x,\phi_0,\phi_1)^\omega.
\end{align*}
We show by induction that the hypothesis
\begin{align*}
H_i \ : \ \text{$\uuv(V_{1:i})$ is a prefix of $W$}
\end{align*}
is true for all $i\in\Z_+$, noting that this proves the claim. 

In the base case, $\uuv(\epsilon)=\epsilon$ is a prefix of $W$, so $H_0$ is true.
For the inductive step, say $H_{i-1}$ is true for some $i\in\N$.
Let $(x_n : n\in\N)$ be the $x$-threshold orbit for $\phi_0,\phi_1$,
let $(\tilde x_n : n\in\N)$ be the $\phi_0^{(-1)}(x)$-threshold orbit for $\psi_0 := \phi_0,\psi_1 := \phi_{01}$
and let $k := \abs{\uuv(V_{1:(i-1)})}+1$.
Then 
\begin{align*}
\tilde x_1 = \psi_1(\phi_0^{(-1)}(x)) = \phi_1(x) = x_1
\end{align*} 
and, letting $\psi_w$ denote the composition of $\psi_0,\psi_1$ corresponding to a word $w$,
\begin{align*}
\tilde x_i &= \psi_{V_{1:(i-1)}}(\tilde x_1) & \text{by definition of $V$} \\
&= \phi_{\uuv(V_{1:(i-1)})}(\tilde x_1) & \text{by definition of $\psi_0,\psi_1$} \\
&= \phi_{W_{1:(k-1)}}(\tilde x_1) & \text{as $H_{i-1}$ is true} \\
&= \phi_{W_{1:(k-1)}}(x_1) & \text{as $\tilde x_1 = x_1$} \\
&= x_k & \text{by definition of $W$.}
\end{align*}
As $x\in [y_{10},y_0]$, we have $\tilde x := \phi_0^{(-1)}(x) \in [y_{01},y_0]$.
Letting $\tilde y_{0}, \tilde y_1$ be the fixed points of $\psi_0,\psi_1$, this reads $\tilde x \in [\tilde y_1, \tilde y_0]$.
But $\psi_0,\psi_1$ satisfy A1, as Lemma~\ref{lemma:monexist} shows that these functions are increasing and contractive,
and Lemma~\ref{lemma:ypi} shows that $\tilde y_1 < \tilde y_0$.
Thus Lemma~\ref{lemma:orbit} shows that
\begin{align*}
\tilde x_i < \psi_0(\tilde x) = \phi_0(\phi_0^{(-1)}(x)) = x .
\end{align*}
But we already showed that $x_k = \tilde x_i$ so this gives $x_k < x$. Therefore $W_k = 0$, by definition of the $x$-threshold word.
If $V_i = 0$ then we can conclude that $H_i$ is true.
Otherwise $V_i = 1$ so that $\tilde x_i \ge \phi_0^{(-1)}(x)$.
But we already showed that $W_k = 0$ and $x_k = \tilde x_i$, so $x_{k+1} = \phi_0(x_k) = \phi_0(\tilde x_i) \ge x$.
Therefore $W_{k+1} = 1$ and we conclude that $H_i$ is true.

The proof for the claim involving $\uvv$ is similar. This completes the proof.
\end{proof}


\begin{lemma}
\label{lemma:word-induction}
Suppose $\phi_0,\phi_1$ satisfy A1, $x\in\IR$ and $0v1$ is a finite word. Then 
\begin{equation*}
\left\{ \begin{aligned}
\pi(\phi_0(x), \phi_0,\phi_1) &= \uuv(0v1) 
& \ \Leftrightarrow \ &&
\pi(x, \phi_0,\phi_{01}) &= 0v1 \\
\pi(x, \phi_0,\phi_1) &= \uvv(0v1) 
& \ \Leftrightarrow \ &&
\pi(x, \phi_{01},\phi_1) &= 0v1. 
\end{aligned}\right.
\end{equation*}
\end{lemma}
\begin{proof}
Consider the claim involving the morphism $\uuv$.

If $\pi(\phi_0(x), \phi_0,\phi_1) = \uuv(0v1)$ then $\abs{\pi(\phi_0(x), \phi_0,\phi_1)}_{00}>0$,
as $\abs{0v1}_{01}>0$ for any finite word $v$ and $\uuv(01)=001$.
So Lemma~\ref{lemma:0011} shows that $\phi_0(x) > y_{10}$.
Thus $x> y_{01}$.
Also, $\abs{\pi(\phi_0(x), \phi_0,\phi_1)}_1 = \abs{\uuv(0v1)}_1>0$, so Lemma~\ref{lemma:basecase}
shows that $\phi_0(x) < y_0$.
Thus $x< y_{0}$.
Hence Lemma~\ref{lemma:recursion} gives
\begin{align*}
\uuv(0v1) = \uuv(\pi(x, \phi_0,\phi_{01})).
\end{align*}
As $\uuv$ is injective, it follows that
\begin{align*}
0v1 = \pi(x, \phi_0,\phi_{01}).
\end{align*}

If $\pi(x, \phi_0,\phi_{01}) = 0v1$ then $\abs{\pi(x,\phi_0,\phi_{01})}_{0}>0$
and $\abs{\pi(x,\phi_0,\phi_{01})}_{1}>0$.
So Lemma~\ref{lemma:basecase} shows that $x\in (y_{01},y_0)$.
Hence Lemma~\ref{lemma:recursion} gives
\begin{align*}
\pi(\phi_0(x), \phi_0, \phi_1) = \uuv(0v1).
\end{align*} 

The argument for the claim involving $\uvv$ is symmetric.
This completes the proof.
\end{proof}


\begin{lemma}
\label{lemma:Ccase}
Suppose $\phi_0,\phi_1$ satisfy A1 and $0w1$ is a Christoffel word. Then 
\begin{align*}
x\in [y_{01w},y_{10w}] \ \Leftrightarrow \ \pi(x,\phi_0,\phi_1) = 0w1.
\end{align*}
\end{lemma}
\begin{proof}
We use induction on the depth of $0w1$ in the Christoffel tree, with hypothesis
\begin{align*}
\hspace{-1.5cm} H_n \ : \ \left\{\begin{array}{l}
 	\text{If $0w1$ is at depth $n$ of the tree and $\phi_0,\phi_1$ satisfy A1, then} \\
		\hspace{2cm} x\in [y_{01w},y_{10w}] \ \Leftrightarrow \ \pi(x,\phi_0,\phi_1) = 0w1. \end{array} \right.
\end{align*}
Lemma~\ref{lemma:basecase} shows that the base case $(H_1)$ with $0w1 = 01$ is true. For the inductive step, let $0w1$ be a word at depth $n+1$ of the tree and assume $H_n$ is true. 
Then either $0w1 = \uuv(0v1)$ or $0w1 = \uvv(0v1)$ for some word $0v1$ which is at depth $n$ of the tree. 
If $0w1 = \uuv(0v1)$ then Lemma~\ref{lemma:word-induction} gives
\begin{align*}
\pi(x , \phi_0,\phi_1) = 0w1 
&\quad\Leftrightarrow\quad
\pi(\phi_0^{(-1)}(x) , \phi_0,\phi_{01}) = 0v1 .
\intertext{Now $\phi_{01}$ is increasing and contractive by Lemma~\ref{lemma:monexist} and
$y_{01} < y_{0}$ by Lemma~\ref{lemma:ypi}.
Thus $\phi_0,\phi_{01}$ satisfy A1. So the assumption that $H_n$ is true shows that}
\pi(\phi_0^{(-1)}(x) , \phi_0,\phi_{01}) = 0v1 
&\quad\Leftrightarrow\quad
\phi_0^{(-1)}(x)\in [y_{\uuv(01v)},y_{\uuv(10v)}] .
\intertext{But as $0w1 = \uuv(0v1) = 0\uuv(v)01$, we have
\begin{align*}
\phi_0(y_{\uuv(01v)}) &= \phi_0(y_{001\uuv(v)}) = y_{01\uuv(v)0} = y_{01w} \\
\phi_0(y_{\uuv(10v)}) &= \phi_0(y_{010\uuv(v)}) = y_{10\uuv(v)0} = y_{10w}.
\end{align*}
As $\phi_0$ is increasing and continuous, it follows that}
\phi_0^{(-1)}(x)\in [y_{\uuv(01v)},y_{\uuv(10v)}] 
&\quad\Leftrightarrow\quad
x \in [y_{01w},y_{10w}].
\end{align*}
Therefore $H_{n+1}$ is true. 

If $0w1 = \uvv(0v1)$ then the proof is similar.
This completes the proof.
\end{proof}


\subsection{$x$-Threshold Words as Sturmian $\mathcal{M}$-Words}
\label{sub:Sturmian-case}
It turns out that Lemma~\ref{lemma:Ccase} characterises $x$-threshold words
for nearly all $x\in (y_1,y_0)$. However, its proof cannot be extended to all values of $x$
as it is based on induction on the depth $n\in\N$ in the Christoffel tree, 
and we need to take the limit as $n\rightarrow\infty$ to address the remaining values of $x$. 
Those remaining values correspond to Sturmian $\mathcal{M}$-words, as we show in this subsection.


First we show how Sturmian $\mathcal{M}$-words can be written as the limit of a sequence of words.
\begin{lemma}
\label{lemma:compose-morphisms}
A word $w$ is a Sturmian $\mathcal{M}$-word if and only if it is of the form 
\begin{align*}
w &= \lim_{n\rightarrow\infty} \morph_1 \circ \morph_2 \circ \cdots \circ \morph_n(01)
\end{align*}
for some sequence of morphisms $(\morph_n : n\in\N)$ with $\morph_n \in \{\uuv, \uvv\}$
which is not eventually constant 
(\textit{i.e.} there is no $m\in\N$ such that $\morph_n = \morph_m$ for all $n\in\N$ with $n > m$).
\end{lemma}
The above result is well known, for instance, noting the isomorphism of the
Christoffel tree and the Stern-Brocot tree, see Chapter~4 of~\citet{Graham94}.

\begin{remark} 
The requirement that the sequence of morphisms generating Sturmian $\mathcal{M}$-words
is not eventually constant rules out words such as 
\begin{align*}
\lim_{n\rightarrow \infty} \uuv^n(01) = 0^\omega.
\end{align*}
Indeed, this corresponds to the Christoffel word $0$.
It also rules out words such as
\begin{align*}
\lim_{n\rightarrow\infty} \uvv^n(01) = 01^\omega.
\end{align*}
Indeed, this is not an \mword{ }because there is no $\alpha\in [0,1]$ with 
\begin{align*}
\text{$0 = \floor{\alpha n}-\floor{\alpha(n-1)}$ for $n=1$ and
$1 = \floor{\alpha n}-\floor{\alpha(n-1)}$ for $n = 2, 3, \dots$.}
\end{align*}
\end{remark}


Now we show that every $x$-threshold word corresponds to an \mword.
\begin{lemma}
\label{lemma:pi-is-mword}
Suppose $\phi_0,\phi_1$ satisfy A1 and $x\in\IR$. 
Then $\pi(x,\phi_0,\phi_1)$ is an \mword.
\end{lemma}
\begin{proof}
Lemma~\ref{lemma:basecase} shows that $\pi(x,\phi_0,\phi_1)$ is an \mword{ }for $x\in\IR\backslash (y_1,y_0)$. 
So, for the rest of this proof we assume that $x\in (y_1,y_0)$.
We shall now define procedure which generates a sequence of morphisms $\morph_k$, thresholds $x_k\in\IR$, mappings $\phi_{k,a} : \IR \rightarrow \IR$ for $a\in\{0,1\}$
and we denote the compositions of those mappings by 
$$\phi_{k,w} = \phi_{k,w_{\abs{w}}} \circ \cdots \phi_{k,w_2} \circ \phi_{k,w_1}$$
for any finite word $w$.
The procedure is as follows:
\begin{enumerate}[leftmargin=2cm,itemsep=-1mm]
\item \textbf{let} $k \leftarrow 1$ 
\item \textbf{let} $(x_k,\phi_{k,0},\phi_{k,1}) \leftarrow (x,\phi_0,\phi_1)$
\item \textbf{let} $y_{k,01}, y_{k,10}$ be the fixed points of $\phi_{k,01}, \phi_{k,10}$
\item \textbf{while} $x_k\notin [y_{k,01},y_{k,10}]$
\item \hspace{1cm} \textbf{if} $x_k<y_{k,01}$
\item \hspace{2cm} \textbf{set} $(x_{k+1},\phi_{k+1,0},\phi_{k+1,1},\morph_k) \leftarrow (x_k,\phi_{k,01},\phi_{k,1},\uvv)$
\item \hspace{1cm} \textbf{else}
\item \hspace{2cm} \textbf{set} $(x_{k+1},\phi_{k+1,0},\phi_{k+1,1},\morph_k) \leftarrow (\phi_{k,0}^{(-1)}(x),\phi_{k,0},\phi_{k,01},\uuv)$
\item \hspace{1cm} \textbf{end}
\item \hspace{1cm} \textbf{let} $k\leftarrow k+1$
\item \hspace{1cm} \textbf{let} $y_{k,01}, y_{k,10}$ be the fixed points of $\phi_{k,01}, \phi_{k,10}$
\item \textbf{end}
\end{enumerate}
The sequence of morphisms $\morph_1, \morph_2, \dots$ generated by this procedure is either empty, of finite length $n\in\N$ or of infinite length.
We consider each of these cases in turn.

If the sequence is empty, then $x\in [y_{1,01},y_{1,10}] = [y_{01},y_{10}]$. So Lemma~\ref{lemma:basecase} shows that 
$\pi(x,\phi_0,\phi_1) = 01$, which is an \mword.

If the sequence has finite length, let $n$ be that length.
As the morphisms $\uuv$ and $\uvv$ generate the Christoffel tree by pre-composition,
as remarked at~(\ref{eq:precomp}), the word
\begin{align*}
w &:= \morph_1 \circ \cdots \circ \morph_n(01)
\end{align*}
is a Christoffel word. 
Also, when the procedure sets $\morph_k = \uvv$ for some $k\in\N$,
we have $x_k < y_{k,01} < y_{k,10}$ so Lemma~\ref{lemma:recursion} 
shows that 
\begin{align*}
\pi(x_k,\phi_{k,0},\phi_{k,1}) = \uvv(\pi(x_{k+1},\phi_{k+1,0},\phi_{k+1,1})).
\end{align*} 
Similarly, when the procedure sets $\morph_k = \uuv$ we have
\begin{align*}
\pi(x_k,\phi_{k,0},\phi_{k,1}) = \uuv(\pi(x_{k+1},\phi_{k+1,0},\phi_{k+1,1})).
\end{align*}
Therefore the $x$-threshold word is $\pi(x,\phi_0,\phi_1) = w$ which is an \mword.

Finally, if the sequence does not terminate, 
then Lemma~\ref{lemma:compose-morphisms} shows that the word
\begin{align*}
w &:= \lim_{n\rightarrow\infty} \morph_1 \circ \cdots \circ \morph_n(01)
\end{align*}
is a Sturmian \mword{ }\textit{provided} there is no $m\in\N$ such that $\morph_{n} = \morph_{m}$ for all $n\in\N$ with $n\ge m$.
Also, the argument given for finite sequences above shows that $w = \pi(x,\phi_0,\phi_1)$.
If there were such an $m$ and $\morph_{m} = \uvv$, then 
\begin{align*}
w &= \morph_1 \circ \cdots \circ \morph_{m-1} \circ \lim_{k\rightarrow\infty} \uvv^k (01) \\
&= \morph_1 \circ \cdots \circ \morph_{m-1} (01^\omega).
\end{align*}
Thus the word at stage $m$ is $\pi(x_{m},\phi_{m,0},\phi_{m,1}) = 01^\omega$.
But this is impossible, as the fact that the first letter is $0$ requires $\phi_{m,1}(x_m)<x_m$, so that $x_m>y_{m,1}$, 
whereas the fact that remaining letters are $1$ requires $\phi_{m,101^n}(x_m)\ge x_m$ for all $n\in \Z_+$, so that $y_{m,1} \ge x_m$,
which is a contradiction.
If there were such an $m$ and $\morph_{m} = \uuv$,
then
\begin{align*}
w &= \morph_1 \circ \cdots \circ \morph_{m-1} \circ \lim_{k\rightarrow\infty} \uuv^k (01) \\
&= \morph_1 \circ \cdots \circ \morph_{m-1} (0^\omega)
\end{align*}
Therefore at stage $m$ the word $\pi(x_{m},\phi_{m,0},\phi_{m,1}) = 0$, but this is impossible as $x_{m}<y_{m,0}$.
In conclusion, $\pi(x,\phi_0,\phi_1)$ is an \mword.

This completes the proof.
\end{proof}


\begin{lemma}
\label{lemma:rate-ordered}
Suppose $0\le \alpha < \beta \le 1$ and let $a,b$ be the \mword{s }of rate $\alpha,\beta$ respectively.
Then $a^\omega \prec b^\omega$.
\end{lemma}
\begin{proof}
Consider the first $n\in\N$ with $(a^\omega)_n \ne (b^\omega)_n$. 
Then  
\begin{align*}
\lfloor \alpha (n-1) \rfloor  = \abs{(a^\omega)_{1:(n-1)}}_1 = \abs{(b^\omega)_{1:(n-1)}}_1 = \lfloor \beta (n-1) \rfloor ,
\end{align*}
by definition of $\mathcal{M}$-words, so that
\begin{align*}
(a^\omega)_n = \lfloor \alpha n \rfloor - \lfloor \alpha (n-1) \rfloor 
= \lfloor \alpha n \rfloor - \lfloor \beta (n-1) \rfloor 
< \lfloor \beta n \rfloor - \lfloor \beta (n-1) \rfloor = (b^\omega)_n
\end{align*}
where the inequality holds as $(a^\omega)_n \ne (b^\omega)_n$, as $\alpha< \beta$ and as the floor function is non-decreasing.
Therefore $a^\omega \prec b^\omega$.
\end{proof}


Consider a tree and a node $x$ of the tree. Recall that the \textit{subtree rooted at $x$} is the tree
of all descendents of $x$ that has node $x$ as a root. 
If the tree is a binary tree, the \textit{left subtree of $x$} is the subtree rooted at the left child of $x$,
and the \textit{right subtree of $x$} is the subtree rooted at the right child of $x$.
Thus the subtree rooted at $x$ contains node $x$, but the left and right subtrees of $x$ do not contain node $x$.

\begin{lemma}
\label{lemma:subtree-rate}
Suppose $(u,v)$ is a Christoffel pair. 
Considering $uv$ as a node of the Christoffel tree,
let $l$ and $r$ be Christoffel words in the left and right subtrees of $uv$.
Then
\begin{align*}
\rate(u) < \rate(l) < \rate(uv) < \rate(r) < \rate(v). 
\end{align*}
\end{lemma}
\begin{proof}
First we use induction to show that for any Christoffel pair $(u,v)$, we have
\begin{align}
\label{eq:rateuv}
\rate(u) < \rate(v) .
\end{align}
In the base case $(u,v) = (0,1)$ and (\ref{eq:rateuv}) is true.
For the inductive step, 
say $(u,v)$ is the left child of a Christoffel pair $(a,b)$ with $\rate(a) < \rate(b)$.
Then $(u,v) = (a,ab)$ so that
\begin{align*}
\rate(u) = \frac{\abs{a}_1}{\abs{a}} < \frac{\abs{a}_1+\abs{b}_1}{\abs{a}+\abs{b}} = \rate(v) 
\end{align*}
by the mediant inequality. The proof for right children is similar.
Therefore~(\ref{eq:rateuv}) is true.

As $l$ is in the left subtree of $uv$, it follows from the construction of the Christoffel tree that $l$ consists of $m$ copies of $u$ and $n$ copies of $uv$ concatenated in some order, for some $m,n\in\N$. Thus the mediant inequality and (\ref{eq:rateuv}) give
\begin{align*}
\rate(l) = \rate(u^m(uv)^n) = \frac{m \abs{u}_1 + n \abs{uv}_1}{m \abs{u} + n \abs{uv}} 
\in \left( \frac{\abs{u}_1}{\abs{u}}, \frac{\abs{uv}_1}{\abs{uv}}\right) 
\end{align*}
Therefore $\rate(u) < \rate(l) < \rate(uv)$, as claimed.

The proof for a node $r$ in the right subtree of $uv$ is similar. This completes the proof.
\end{proof}


\begin{lemma}
\label{lemma:a10b}
Suppose $(0a1,0b1)$ is a Christoffel pair. 
Then $a10b = b01a$.
\end{lemma}
\begin{proof}
As $a,b,a10b$ are palindromes, we have
$a10b = (a10b)^R = b^R01a^R = b01a.$
\end{proof}


\begin{lemma}
\label{lemma:prefix-suffix}
Suppose the word $0c1$ is in the subtree of the Christoffel tree rooted at $0p1$. Then $p$ is both a prefix and suffix of $c$.
\end{lemma}
\begin{proof}
There are four cases to consider:
\begin{enumerate}
\item The word $0c1$ has no parent, in which case $c = p = \epsilon$ and the claim holds.
\item The word $c$ is of one of the forms $0^m$ or $0^m10^m$ for some $m\in\N$. 
In that case either $p=c$ or $p = 0^l$ for some $l\in\Z_+$ with $l<m$ and the claim holds.
\item The word $c$ is of one of the forms $1^m$ or $1^m01^m$ for some $m\in\N$. This is similar to the previous case.
\item The Christoffel pair of the parent of $0c1$ is of the form $(0a1,0b1)$.
\end{enumerate}

In the last case, we use induction on the length of the path $n\in\Z_+$ through the Christoffel tree from $0p1$ to $0c1$. 
In the base case, $n=0$, we have $c = p$ and the claim is true.
For the inductive step, say $0c1$ is a child of the node with Christoffel pair $(0a1,0b1)$,
that $0a10b1$ is $n$ steps along the path from $0p1$,
and that $p$ is both a prefix and suffix of $a10b$, so that $a10b = pq = rp$ for some words $q,r$.
If $0c1$ is a left-child, then $c = a10a10b = a10rp$ and Lemma~\ref{lemma:a10b} gives
\begin{align*}
c = a10a10b = a10b01a = pq01a
\end{align*}
so $p$ is a prefix and suffix of $c$. Similarly, if $0c1$ is a right child, then $c = a10b10b$ and 
\begin{align*}
pq10b = a10b10b = b01a10b = b01rp .
\end{align*}
This completes the proof.
\end{proof}


\begin{lemma}
\label{lemma:pi-is-decreasing}
Suppose $\phi_0,\phi_1$ satisfy A1. Then 
\begin{align*}
\text{$\pi(x,\phi_0,\phi_1)^\omega$ is a lexicographically non-increasing function of $x\in\IR$.}
\end{align*}
\end{lemma}
\begin{proof}
If $a,b\in\IR$ and $\pi(a,\phi_0,\phi_1)^\omega \succ \pi(b,\phi_0,\phi_1)^\omega$, then there is a finite word $u$ such that
\begin{align*}
\pi(a,\phi_0,\phi_1)^\omega &= u1v, & \pi(b,\phi_0,\phi_1)^\omega &= u0w,
\intertext{for some words $v,w$. So the definition of threshold orbits and Lemma~\ref{lemma:incdec} give}
y_{1u}\ge\phi_{1u}(a)&\ge a, & y_{1u}<\phi_{1u}(b)&< b.
\end{align*}
Therefore $a < b$. This completes the proof.
\end{proof}


In the main text, we defined the fixed point $y_s$ of a Sturmian $\mathcal{M}$-word as the limit
of a sequence of fixed points $y_{01w^{(n)}}$ or $y_{10w^{(n)}}$ where the words $(0w^{(n)}1 : n\in\N)$ correspond
to a particular path in the Christoffel tree. 
We now define a subsequence associated with each of these sequences of words and fixed points. 
Consider a Sturmian $\mathcal{M}$-word 
\begin{align*}
0s = \lim_{n\rightarrow\infty} \uuv^{a_1} \circ \uvv^{b_1} \circ \cdots \circ \uuv^{a_n} \circ \uvv^{b_n}(01) 
\end{align*}
where $a_1\in\Z_+$ and $a_{n+1}, b_n \in \N$ for $n\in\N$.
We define the sequences $(u^{(n)} : n\in\N)$ and $(l^{(n)} : n\in\N)$ of the central portions of the Christoffel words as
\begin{align}
\label{def:ys}
0u^{(n)}1 &:= \uuv^{a_1} \circ \uvv^{b_1} \circ \cdots \circ \uuv^{a_n} (01), &
0l^{(n)}1 &:=  \uuv^{a_1} \circ \uvv^{b_1} \circ \cdots \circ \uuv^{a_n} \circ \uvv^{b_n}(01) .
\end{align}
These words form a subsequence of $(0w^{(n)}1 : n\in\N)$, so if $\lim_{n\rightarrow\infty} y_{01w^{(n)}}$
exists, then so does $\lim_{n\rightarrow\infty} y_{01l^{(n)}}$ and these limits are equal.


\begin{lemma}
\label{lemma:ylyu}
Suppose A1 holds, $n\in\N$ and that $0u^{(n)}1$ and $0l^{(n)}1$ are as in (\ref{def:ys}). Then 
\begin{align*} 
y_{01l^{(n)}} < y_{01l^{(n+1)}} < y_{10u^{(n+1)}} < y_{10u^{(n)}}.
\end{align*}
\end{lemma}
\begin{proof}
By definition, $0l^{(n+1)}1$ is in the left subtree of $0l^{(n)}1$,
so Lemma~\ref{lemma:subtree-rate} gives 
\begin{align*}
\rate(0l^{(n+1)}1) < \rate(0l^{(n)}1).
\end{align*}
Hence Lemma~\ref{lemma:Ccase} and Lemma~\ref{lemma:rate-ordered} give 
\begin{align*}
\pi(y_{01l^{(n+1)}})^\omega = (0l^{(n+1)}1)^\omega \prec (0l^{(n)}1)^\omega = \pi(y_{10l^{(n)}})^\omega .
\end{align*}
Therefore Lemma~\ref{lemma:pi-is-decreasing} shows that 
\begin{align*}
y_{01l^{(n)}} < y_{01l^{(n+1)}}.
\end{align*}

But $0l^{(n+1)}1$ is in the right subtree of $0u^{(n+1)}1$. So the same argument gives 
\begin{align*}
y_{01l^{(n+1)}} < y_{10u^{(n+1)}}.
\end{align*}
Similarly, $0u^{(n+1)}1$ is in the right subtree of $0u^{(n)}1$, so 
\begin{align*}
y_{10u^{(n+1)}} < y_{10u^{(n)}}.
\end{align*}
This completes the proof.
\end{proof}

Note that the argument in the above proof also shows that any word $0w^{(m)}1$ lying strictly between $0u^{(n)}1$ and $0l^{(n)}1$ on the path through the Christoffel tree from $0u^{(n)}1$ to $0l^{(n)}1$
has 
\begin{align*}
y_{10l^{(n)}} < y_{01w^{(m)}} < y_{10w^{(m)}} < y_{01u^{(n)}}.
\end{align*}
Similarly, any word $0w^{(m)}1$ lying strictly between $0l^{(n)}1$ and $0u^{(n+1)}1$ 
on the path from $0l^{(n)}1$ to $0u^{(n+1)}1$ 
has 
\begin{align*}
y_{10l^{(n)}} < y_{01w^{(m)}} < y_{10w^{(m)}} < y_{01u^{(n+1)}}.
\end{align*}
Thus if the subsequences $(y_{01l^{(n)}} : n\in\N)$ and $(y_{10u^{(n)}} : n\in\N)$ have limits, 
then so do the full sequences $(y_{01w^{(m)}} : m\in\N)$ and $(y_{10w^{(m)}} : m\in\N)$.


\begin{lemma}
\label{lemma:suffixword}
Suppose $(0w^{(n)}1 : n \in \N)$ is the sequence of words traversed along an infinite path down the Christoffel tree.
Let $(n_i : i \in\N)$ be an increasing sequence on $\N$.
Then there exists an infinite word $s$ and an increasing sequence $(k_i : i\in\N)$ on $\N$ such that
$s_{1:k_i}$ is a suffix of both $w^{(n_i)}$ and $w^{(n_{i+1})}$, for all $i\in\N$.
\end{lemma}
\begin{proof}
We show that $s_{1:k_i} := w^{(n_i)}$ for $i\in\N$, is well-defined and satisfies this claim.
As $w^{(n_i)}$ is a prefix of all of its descendents by Lemma~\ref{lemma:prefix-suffix}, it follows that $s_{1:k_i}$ is a prefix of $s_{1:k_{i+1}}$.
Also, as $w^{(n_i)}$ is a suffix of all of its descendents, it follows that $s_{1:k_i}$ is a suffix of both $w^{(n_i)}$ and $w^{(n_{i+1})}$.
\end{proof}


\begin{lemma}
\label{lemma:ys-equal}
Suppose $\phi_0,\phi_1$ satisfy A1 and $0s$ is a Sturmian $\mathcal{M}$-word. 
Consider the sequence of Christoffel words $(0w^{(n)}1 : n\in\N)$ traversed on the infinite path through the Christoffel tree towards $0s$ (as defined in the main text just before Theorem~\ref{theorem:characterisation}).
Then the fixed points 
\begin{align*}
y_{01s} := \lim_{n\rightarrow\infty} y_{01w^{(n)}} \qquad\text{and} \qquad 
y_{10s} := \lim_{n\rightarrow\infty} y_{10w^{(n)}} 
\end{align*}
exist and are equal.
\end{lemma}
\begin{proof}
Let $a_n, b_n, 0u^{(n)}1, 0l^{(n)}1$ for $n\in\N$ be as in (\ref{def:ys}).

Existence of the fixed points follows from the monotone convergence theorem for real-valued sequences.
Indeed, Lemma~\ref{lemma:ylyu} shows that $(y_{01l^{(n)}} : n\in\N)$ is an increasing sequence,
that $(y_{10u^{(n)}} : n\in\N)$ is a decreasing sequence, and that these sequences are bounded.

By Lemma~\ref{lemma:suffixword} we have $u^{(n)} = c^{(n)}w_{1:k_n}$ and $l^{(n)} = d^{(n)}w_{1:k_n}$
for all $n\in\N$ for some sequences of finite words $(c^{(n)} : n\in\N)$ and $(d^{(n)} : n\in \N)$,
for some infinite word $w$ and for some increasing sequence $(k_n:n\in\N)$ on $\N$.

As $\phi_a(x)$ is an increasing function of $x\in\IR$ for any finite word $a$, by Lemma~\ref{lemma:monexist},
\begin{align*}
y_{10u^{(n)}}
&= \phi_{10u^{(n)}}(y_{10u^{(n)}})  \\
&< \phi_{10u^{(n)}}(y_{0})  \\
&= \phi_{10c^{(n)}w_{1:k_n}}(y_0)  \\
&= \phi_{w_{1:k_n}}(\phi_{10c^{(n)}}(y_0))  \\
&< \phi_{w_{1:k_n}}(y_0) .
\end{align*}
Similarly, we have
\begin{align*}
y_{01l^{(n)}} &> \phi_{w_{1:k_n}}(y_1) .
\end{align*}
Thus
\begin{align*}
\lim_{n\rightarrow \infty} (y_{10u^{(n)}}-y_{01l^{(n)}}) &\le \lim_{n\rightarrow \infty} (\phi_{w_{1:k_n}}(y_0)  - \phi_{w_{1:k_n}}(y_1)) = 0 
\end{align*}
where the last step is Lemma~\ref{lemma:limit-is-0}. But $y_{10u^{(n)}}>y_{01l^{(n)}}$ for $n\in\N$ by Lemma~\ref{lemma:ylyu}. Therefore
\begin{align*}
y_{10s} = \lim_{n\rightarrow\infty} y_{10u^{(n)}} = \lim_{n\rightarrow\infty} y_{01l^{(n)}} = y_{01s}.
\end{align*}
This completes the proof.
\end{proof}

In view of the above Lemma, from now on we shall write $y_s = y_{10s} = y_{01s}$.


\begin{lemma}
\label{lemma:Scase}
Suppose $\phi_0,\phi_1$ satisfy A1 and $0s$ is a Sturmian $\mathcal{M}$-word.
Then 
\begin{align*}
\pi(x,\phi_0,\phi_1) = 0s \ \Leftrightarrow \ x = y_{s}.
\end{align*}
\end{lemma}
\begin{proof}
Let $(u^{(n)} : n\in\N)$ and $(l^{(n)} : n\in\N)$ be the sequences (\ref{def:ys}) appearing in the definition of the fixed point of $0s$. For fixed $\phi_0,\phi_1$ let us write $\pi(x)$ in place of $\pi(x,\phi_0,\phi_1)$.

Say $x = y_{s}$. Recall that $x = y_{10s} = y_{01s}$ by Lemma~\ref{lemma:ys-equal}. 
As $y_{10u^{(n)}} > y_{10s}$ for all $n\in\N$ by Lemma~\ref{lemma:ylyu}, and $\pi(z)^\omega$ is a lexicographically non-increasing function of $z$, by Lemma~\ref{lemma:pi-is-decreasing}, it follows that
\begin{align*}
\pi(x)^\omega \succ \pi(y_{10u^{(n)}})^\omega = (0u^{(n)}1)^\omega
\end{align*}
where the equality follows from Lemma~\ref{lemma:Ccase}.
A similar argument gives 
$\pi(x)^\omega \prec (0l^{(n)}1)^\omega$ for $n\in\N$. Therefore
\begin{align*}
0s = \lim_{n\rightarrow\infty} (0u^{(n)}1)^\omega \preceq \pi(x)^\omega \preceq \lim_{n\rightarrow\infty} (0l^{(n)}1)^\omega = 0s .
\end{align*}

Now say $\pi(x) = 0s$ for some $x\in\IR$.
Then $y_{01l^{(n)}} < x < y_{10u^{(n)}}$ as $\pi(z)^\omega$ is a lexicographically non-increasing function of $z$ and $(0u^{(n)}1)^\omega \prec 0s \prec (0l^{(n)}1)^\omega$ for all $n\in\N$.
Therefore 
\begin{align*}
y_{s} = \lim_{n\rightarrow\infty} y_{01l^{(n)}} \le x \le \lim_{n\rightarrow\infty} y_{10u^{(n)}} = y_{s}.
\end{align*}
This completes the proof.
\end{proof}


\subsection{Proof of Theorem~\ref{theorem:characterisation}}
\label{sub:characterisation-proof}
\begin{proof}
The existence of fixed points $y_{01p}, y_{10p}$ follows from Lemma~\ref{lemma:monexist}
and the existence of $y_s$ follows from Lemma~\ref{lemma:ys-equal}.

The fact that $\sigma(z|z)=1\pi(z,\phi_0,\phi_1)^\omega$ is a lexicographically 
non-decreasing function of $z\in\IR$ follows from Lemma~\ref{lemma:pi-is-decreasing}.

Lemma~\ref{lemma:basecase} addresses the value of $\sigma(z|z)$ for $z\le y_1$ and $z\ge y_0$,
Lemma~\ref{lemma:Ccase} addresses the case $z\in [y_{01p},y_{10p}]$
and Lemma~\ref{lemma:Scase} address the case $z = y_s$. 
Thus the image of $\pi(z,\phi_0,\phi_1)$ as $z$ ranges over $\IR$ contains all $\mathcal{M}$-words.
Using Lemma~\ref{lemma:pi-is-mword}, it follows that this image is exactly the set of $\mathcal{M}$-words.

This completes the proof.
\end{proof}



\subsection{Proof of Theorem~\ref{theorem:xs-characterisation}}
\label{sub:xs-characterisation-proof}
Based on the work of~\citet{Kozyakin03}, we begin by showing that maps-with-gaps formed from functions satisfying Assumption~A1 are \textit{locally-growing relaxation functions}  (Lemmas~\ref{lemma:phi0110} and~\ref{lemma:A1-proves-kozy}) and that the itineraries of such functions are \textit{$1$-balanced words} (Lemmas~\ref{lemma:kozy} and~\ref{lemma:F}). 
As $1$-balanced words correspond to factors of lower mechanical words (Lemma~\ref{lemma:Dulucq}), 
this enables us to describe the itineraries of maps-with-gaps, for arbitrary initial states and thresholds (Lemma~\ref{lemma:mechanical-factor}). 
We couple this description with an easy result about lexicographic ordering (Lemma~\ref{lemma:lexs}) and with a result about the number of factors of mechanical words (Lemma~\ref{lemma:mignosi}) to bound the number of discontinuities of
the itinerary of a map-with-a-gap as a function of its threshold (Lemma~\ref{lemma:Ot3}).
Finally, we prove Theorem~\ref{theorem:xs-characterisation}.


\begin{definition}
Let $\IR$ be an interval of $\R$. 
A function $f : \IR \rightarrow \IR$ is a \textbf{locally-growing relaxation function} with threshold $z\in\IR$ if 
\begin{enumerate}
\item $f(z)<z<f(z^-)<\infty$
\item $f(f(z^-)) \le f(f(z))$
\item $f(x)$ is  increasing for $x \in [f(z),z)$
\item $f(x)$ is  increasing for $x \in [z,f(z^-))$.
\end{enumerate}
\end{definition}


In~\citet{Kozyakin03}, this terminology was used for a smaller class of functions $f$. 
In particular, the domain and range were restricted to the interval $[0,1)$ and there
was a requirement that $f$ is continuous on each of the intervals $[f(z),z)$ and $[f(z^-),f(z))$.

The following two Lemmas show that maps-with-gaps whose parts satisfy Assumption~A1 lead to locally-growing relaxation functions.


\begin{lemma}
\label{lemma:phi0110}
Suppose $\phi_0,\phi_1$ satisfy A1 and $x\in [y_1,y_0]$.
Then
$\phi_{01}(x) < \phi_{10}(x).$
\end{lemma}
\begin{proof}
Suppose that $x\in [y_1,y_0)$. Then using Lemma~\ref{lemma:incdec} gives
\begin{align*}
\phi_1(\phi_0(x)) - \phi_1(x) &< \phi_0(x) - x & \text{as $\phi_1$ is contractive and $x<y_0$} \\
\phi_0(x) - \phi_0(\phi_1(x)) &\le x - \phi_1(x) & \text{as $\phi_0$ is contractive and $x\ge y_1$} \\
\phi_{01}(x) &< \phi_{10}(x) & \text{by adding these inequalities.}
\end{align*}
A symmetric argument holds if $x\in (y_1,y_0]$. This completes the proof.
\end{proof}


\begin{lemma}
\label{lemma:A1-proves-kozy}
Suppose $\phi_0,\phi_1$ satisfy A1, that $z\in (y_1,y_0)$ and let $f(x) := \phi_{\I{x\ge z}}(x)$ for $x\in\IR$. Then $f$ is a locally-growing relaxation function with threshold $z$.
\end{lemma}
\begin{proof}
By definition $f(z^-) = \phi_0(z^-) = \phi_0(z)$, as $\phi_0$ is continuous by Lemma~\ref{lemma:monexist}.
Also $f(z) = \phi_1(z)$, $f(f(z^-)) = \phi_{01}(z)$ and $f(f(z)) = \phi_{10}(z)$. 
So the conditions defining a locally-growing relaxation function read as follows:
\begin{enumerate}
\item $\phi_1(z) < z < \phi_0(z) < \infty$, which is true for $z\in (y_1,y_0)$ by Lemma~\ref{lemma:incdec}
\item $\phi_{01}(z) < \phi_{10}(z)$, which is the result of Lemma~\ref{lemma:phi0110}
\item $\phi_0(x)$ is increasing for $x\in [\phi_1(z),z)$, which holds by A1
\item $\phi_1(x)$ is increasing for $x\in [z,\phi_0(z))$, which holds by A1.
\end{enumerate}
This completes the proof.
\end{proof}


The following definition is due to~\citet{Morse40}.

\begin{definition} 
An word $w$ is \textbf{$1$-balanced} if 
\begin{align*}
| \ \abs{u}_1 - \abs{v}_1 | \le 1
\end{align*}
for all factors $u,v$ of $w$ with $\abs{u}=\abs{v}$.
\end{definition}


The next two Lemmas use an argument from~\citet{Kozyakin03}, to show that the itineraries of a locally-growing relaxation function are $1$-balanced.
We denote the fractional part of a real number $x$ by $\{x\} := x-\lfloor x \rfloor .$

\begin{lemma}
\label{lemma:kozy}
Suppose $f$ is a locally-growing relaxation function with threshold $z$. 
Let 
\begin{align*}
F(x) := \displaystyle\frac{f(t\{x\}-t\I{\{x\}\in [\alpha,1)}+z)-z}{t} + \floor{x} + 1 \qquad\text{for $x\in\R$}
\end{align*}
where $t:=f(z^-)-f(z)$ and $\alpha := (f(z^-)-z)/t$. 
Then 
\begin{enumerate}
\item[F1.] $F(0) \in [0,1)$
\item[F2.] $F(x+1)=F(x)+1$ for all $x\in\R$
\item[F3.] $F$ is increasing
\item[F4.] The itineraries of $f$ and $F$ agree in the sense that
\begin{align*}
f^{(n-1)}(x) \ge z \ \Leftrightarrow \ \left\{F^{(n-1)}\left(\frac{x-z}{t}\right)\right\} \in [0,F(0))
\end{align*}
for all $n\in\N$ and all $x \in [f(z),f(z^-))$.
\end{enumerate}
\end{lemma}
\begin{proof}
First, note that the function $F$ is well-defined as $f(z)\in\R$ and $f(z)<z<f(z^-)<\infty$ so that $t\in\R_{++}$ and
\begin{align}
\label{alpha01}
\alpha = \frac{f(z^-)-z}{f(z^-)-f(z)} \in (0,1) .
\end{align}
It is easy to see that F1 and F2 hold. Indeed 
\begin{align*}
F(0) &= \frac{f(z)-z}{t} + 1 = \frac{f(z)-z+f(z^-)-f(z)}{f(z^-)-f(z)} = \alpha \in (0,1)
\end{align*}
and as the ratio in the definition of $F(x)$ only depends on $\{x\}$, we have
\begin{align}
\label{FF1}
F(x+1) - F(x) = \floor{x+1}-\floor{x} = 1 \qquad \text{for $x\in\R$.}
\end{align}
Now we show that F3 holds. 
\begin{itemize}
\item If $x \in [0,\alpha)$ then $tx+z\in [z,f(z))$.
But $f(\cdot)$ is increasing on $[z,f(z))$ and $t\in\R_{++}$.
It follows that $F(x) = (f(tx+z)-z)/t$ is increasing for $x\in [0,\alpha)$. 
\item As $f(f(z^-)) \le f(f(z))$ and $f$ is increasing on $[z,f(z^-))$, it follows that
\begin{align*}
F(\alpha^-) &= \frac{\lim_{x\uparrow f(x^-)} f(u) - z}{t}+1 
\le \frac{f(f(z^-)) - z}{t}+1 \le \frac{f(f(z))-z}{t}+1 = F(\alpha) .
\end{align*}
\item If $x\in [\alpha,1)$ then $tx-t+z \in [f(z^-),z)$. 
But $f(\cdot)$ is increasing on $[f(z^-),z)$ and $t\in\R_{++}$. 
It follows that $F(x) = (f(tx-t+z)-z)/t$ is increasing for $x\in [\alpha,1)$.
\item The definition of $F$ gives
\begin{align*}
F(1^-) &= \frac{f(z^-)-z}{t} +1  = \alpha+1 = F(1).
\end{align*}
\end{itemize}
In summary, $F(x)$ is increasing for $x \in [0,1]$.
But as $F(x+1)=F(x)+1$ for $x\in\R$ we conclude that $F$ is increasing for $x\in\R$.
Therefore F3 holds. 

Now, we note that for any $m\in\Z_+$ and any $x\in [\alpha-1,\alpha)$ we have
\begin{align}
\label{qinrange}
\frac{f^{(m)}(tx+z)-z}{t} \in [\alpha-1,\alpha).
\end{align}
Indeed, for $m=0$ we have $(f^{(m)}(tx+z)-z)/t = x \in [\alpha-1,\alpha)$.
Also if $y\in [f(z),f(z^-))$ then the assumptions about $f$ show that $f(y) \in [f(z),f(z^-))$,
while if $x\in [\alpha-1,\alpha)$ then the definitions of $t$ and $\alpha$ show that $tx+z \in [f(z),f(z^-))$.
Thus $(f^{(m)}(tx+z)-z)/t \in [\alpha-1,\alpha)$.

Now, we show by induction that for any for any $x\in [\alpha-1,\alpha)$ and $n\in \Z_+$ we have
\begin{align}
\label{Finduction}
F^{(n)}(\{x\}) - \frac{f^{(n)}(tx+z)-z}{t} \in \Z.
\end{align}
In the base case $n=0$ we have 
$F^{(0)}(\{x\}) - (f^{(0)}(tx+z)-z)/t=\{x\} - ((tx+z)-z)/t \in \Z.$
For the inductive step, let $p:=F^{(n)}(\{x\})$ and $q:=(f^{(n)}(tx+z)-z)/t$ and suppose that $p-q\in \Z$.
Then we have 
\begin{align*}
F^{(n+1)}(\{x\}) &= F(p) = F(q)+p-q 
\intertext{as $F$ satisfies F2 and by the assumption that $p-q\in\Z$. 
Furthermore, $q\in [\alpha-1,\alpha)$, by~(\ref{qinrange}) so that for some $k\in\Z$,} 
F^{(n+1)}(\{x\})&=p-q + \begin{cases}
\displaystyle\frac{f(tq+z)-z}{t} + 1 & \text{if $q \in [0,\alpha)$} \\
\displaystyle\frac{f(t(q+1)-t+z)-z}{t} & \text{if $q \in [\alpha-1,0)$} 
\end{cases}\\
&= \frac{f(tq+z)-z}{t} + k \\
&= \frac{f(t \frac{f^{(n)}(tx+z)-z}{t} - z)}{t} + k \\
&= \frac{f^{(n+1)}(tx+z)-z}{t} + k .
\end{align*}
Therefore~(\ref{Finduction}) is true. 

From~(\ref{qinrange}) and~(\ref{Finduction}) it follows that for any $x\in [\alpha-1,\alpha)$ and $n\in\Z_+$ we have
\begin{align}
\{F^{(n-1)}(\{x\})\} \in [0,\alpha) 
&\Leftrightarrow \left\{\frac{f^{(n-1)}(tx+z)-z}{t}\right\} \in [0,\alpha) \nonumber \\
&\Leftrightarrow \frac{f^{(n-1)}(tx+z)-z}{t} \in [0,\alpha) \nonumber \\
&\Leftrightarrow f^{(n-1)}(tx+z) \ge z . \label{sequal}
\end{align}
Therefore F4 holds. This completes the proof.
\end{proof}


\begin{lemma}
\label{lemma:F}
Suppose $F : \R\rightarrow\R$ satisfies Claims F1-F3 of Lemma~\ref{lemma:kozy}. 
Let $\alpha := F(0)$ and for $x\in\R$ let $s(x)$ denote the infinite word with letters
\begin{align*}
s_n(x) := \begin{cases} 
	1 & \text{if $\{F^{(n-1)}(x)\} \in [0,\alpha)$}\\
	0 & \text{if $\{F^{(n-1)}(x)\} \in [\alpha,1)$}
\end{cases}
\end{align*} 
for $n\in\N$.
Then $s(x)$ is $1$-balanced. 
\end{lemma}
\begin{proof}
For any $n \in \N$ and any $x\in [0,1)$ we have
\begin{align*}
&\hspace{-0.5cm} \lfloor F^{(n)}(x)-\alpha \rfloor - \lfloor F^{(n-1)}(x)-\alpha \rfloor \\
&= \lfloor F(r+z) - \alpha \rfloor - \lfloor r+z-\alpha \rfloor && 
	\text{for $r:=\{F^{(n-1)}(x)\}, z:= \lfloor F^{(n-1)}(x) \rfloor$} \\
&=\lfloor F(r) +z - \alpha \rfloor - \lfloor r+z-\alpha \rfloor && 
	\text{as $F(u)+1=F(u)$ for $u\in\R$}\nonumber \\
&=\lfloor F(r) - \alpha \rfloor - \lfloor r-\alpha \rfloor && 
	\text{as $\lfloor r + z \rfloor = \lfloor r \rfloor +z$ since $z\in\Z$}\nonumber \\
&= - \lfloor r-\alpha \rfloor &&
	\text{as $\alpha = F(0) \le F(r) < F(1)=\alpha+1$}\nonumber \\
&	&&
	\text{since $F(\cdot)$ is increasing and $r\in [0,1)$}\nonumber \\
&= \lceil \alpha-r \rceil \nonumber \\
&= \begin{cases} 1 & \text{if $r\in[0,\alpha)$} \\ 0 & \text{if $r\in [\alpha,1)$}\end{cases} \nonumber \\
&= s_n(x) .
\end{align*}
Therefore, for any $m\in\Z_+$,
\begin{align}
\sum_{k=1}^m s_k(x) 
= \sum_{k=1}^m \left( \lfloor F^{(k)}(x)-\alpha \rfloor - \lfloor F^{(k-1)}(x)-\alpha \rfloor \right) 
= \lfloor F^{(m)}(x)-\alpha \rfloor - \lfloor x-\alpha \rfloor .
\label{snform}
\end{align}

Now we show by induction that for any $m\in\Z_+$,
\begin{align}
F^{(m)}(z)-F^{(m)}(y)\in [0,1) \qquad \text{for any $y,z\in\R$ with $z-y\in [0,1)$.}
\label{Fmform}
\end{align}
The base case with $m=0$ reads $z-y \in [0,1)$ which is true.
For the inductive step, let $z':=F^{(m-1)}(z)$ and $y':=F^{(m-1)}(y)$ for $m-1\in\Z_+$ 
and assume that $z'-y'\in [0,1)$. 
Then $F(y') \le F(z')$ as $y'\le z'$ and $F$ is increasing.
Also $F(z') < F(y'+1)=F(y')+1$ as $z'-y' < 1,$ as $F$ is  increasing
and as $F(u+1)=F(u)+1$ for any $u \in \R$.
Therefore $F^{(m)}(z)-F^{(m)}(y) = F(z')-F(y') \in [0,1)$.

For any choice of $m \in \Z_+$ and $x,y \in [0,1)$, we wish to prove that $\abs{\Delta} \le 1$ for
\begin{align*}
\Delta &:= \abs{s_{1}(x) s_{2}(x) \dots s_{m}(x)}_1 - \abs{s_{1}(y) s_{2}(y) \dots s_{m}(y)}_1 \\
&= \sum_{k=1}^m (s_k(x)-s_k(y)) 
= \underbrace{\left( \lfloor F^{(m)}(x)-\alpha \rfloor - \lfloor F^{(m)}(y)-\alpha \rfloor \right)}_{=:A} - \underbrace{\left(\lfloor x-\alpha \rfloor - \lfloor y-\alpha \rfloor \right)}_{=:B} 
\end{align*}
using equation~(\ref{snform}).
But if $x\ge y$ then $x-y\in [0,1)$ so~(\ref{Fmform}) shows that $A$ and $B$
are both of the form $\lfloor a \rfloor - \lfloor b \rfloor$ for some $a-b \in [0,1)$ 
and it follows that both $A$ and $B$ are in $\{0,1\}$.
Whereas if $x\le y$ then $y-x\in [0,1)$ so both $A$ and $B$ are in $\{-1,0\}$.
We conclude that
$\abs{\Delta} = \abs{A-B} \le \abs{1-0} = 1.$
\end{proof}


The following is Theorem~3.1 of~\citet{Dulucq90} and provides the missing link between $1$-balanced words and mechanical words. 
\begin{lemma}
\label{lemma:Dulucq}
Suppose $w$ is a word of length $n\in\N$ with $\abs{w}_0 \abs{w}_1 > 0$.
Then $w$ is $1$-balanced if and only if 
$w_k = \left\lfloor \frac{pk+r}{q} \right\rfloor - \left\lfloor \frac{p(k-1)+r}{q} \right\rfloor$ 
for $k = 1, 2, \dots, n$, where $p,q,r\in\Z$ satisfy
$0<p<q\le n$, $0\le r < q \le n$ and $\gcd(p,q) = 1.$
\end{lemma}


Now we are ready to describe the itineraries of maps-with-gaps, for arbitrary initial states and thresholds.

\begin{lemma}
\label{lemma:mechanical-factor}
Suppose $\phi_0,\phi_1$ satisfy A1 and that $n\in\N$, $x\in\IR$, $s\in\bR$ and $a\in\{0,1\}$. 
Then 
\begin{align*}
\sigma(x|s)_{1:n} = l^m w
\end{align*}
for some $l\in\{0,1\}$, some $m \in \{0, 1, \dots, n\}$, 
and some factor $w$ of a lower mechanical word.
\end{lemma}
\begin{proof}
First, note that if $a,b\in\IR$, then Lemma~\ref{lemma:limit} shows that 
\begin{align}
\label{eq:1m-works}
b> y_1 &\Rightarrow \text{$\phi_{1^m}(a) < b$ for some $m\in\N$}\\
\label{eq:0m-works}
b\le y_0 &\Rightarrow \text{$\phi_{0^m}(a) \ge b$ for some $m\in\N$.}
\end{align}
We consider seven cases.
\begin{enumerate}
\item Say $s\le y_1$ and $x\ge s$. 
Then $\sigma(x|s) = 1 \sigma(\phi_1(x)|s)$.
But if $x > y_1$ then Lemma~\ref{lemma:incdec} gives $\phi_1(x) > y_1 \ge s$,
whereas if $x\le y_1$ then $\phi_1(x)\ge x\ge s$.
In both cases, $\sigma(\phi_1(x)|s) = 1\sigma(\phi_{11}(x)|s)$.
Repeating this argument gives $\sigma(x|s) = 1^\omega$, so the claim is true.
\item Say $s\le y_1$ and $x<s$.
Then $\sigma(x|s)$ begins with $0$, and 
as $s < y_0$ it follows from (\ref{eq:0m-works}) 
that there is a least $m\in\N$ with $\phi_{0^m}(x) \ge s$.
Thus $\sigma(x|s) = 0^m\sigma(\phi_{0^m}(x)|s) = 0^m1^\omega$ by Case~1, 
so the claim is true.
\item Say $s\in (y_0,y_1)$ and $x\in [\phi_1(s), \phi_0(x))$.
As $\phi_0, \phi_1$ satisfy A1, Lemmas~\ref{lemma:A1-proves-kozy},~\ref{lemma:kozy} and~\ref{lemma:F} together show that $\sigma(x|s)$ is a $1$-balanced word.
Thus Lemma~\ref{lemma:Dulucq} shows that $\sigma(x|s)_{1:n} = w$ for some
factor $w$ of a lower mechanical word, so the claim is true.
\item Say $s\in (y_0,y_1)$ and $x < \phi_1(s)$.
Then Lemma~\ref{lemma:incdec} shows that $\phi_1(s) < s$. 
Thus $x<s$, so $\sigma(x|s)$ begins with $0$,
and $\phi_1(s)<y_0$, so (\ref{eq:0m-works}) shows that 
there is a least $m\in\N$ with $\phi_{0^m}(x) \ge \phi_1(s)$.
Thus $\sigma(x|s) = 0^m\sigma(\phi_{0^m}(x)|s)$
where $\sigma(\phi_{0^m}(x)|s)_{1:n} = w$
for some factor $w$ of a lower mechanical word, by Case~3, so the claim is true. 
\item Say $s\in (y_0,y_1)$ and $x \ge \phi_0(s)$.
Then arguing as in Case~4, but using (\ref{eq:1m-works}),
shows that $\sigma(x|s)_{1:n} = 1^mw$ for some $m\in\N$ and some factor $w$ of a lower mechanical word.
\item Say $s\ge y_0$ and $x< s$.
Then arguing as in Case~1 shows that $\sigma(x|s) = 0^\omega$.
\item Say $s\ge y_0$ and $x\ge s$.
Then arguing as in Case~2 shows that $\sigma(x|s) = 1^m0^\omega$
for some $m\in\N$.
\end{enumerate}
This completes the proof.
\end{proof}


The following is an analogue of Lemma~\ref{lemma:pi-is-decreasing} in which only the threshold varies.

\begin{lemma}
\label{lemma:lexs}
Suppose $\phi_0,\phi_1$ satisfy A1 and $x\in\IR$. Then $\sigma(x|s)$ is a lexicographically non-increasing function of $s\in\bR$.
\end{lemma}
\begin{proof}
If $s,t\in\bR$ and $\sigma(x|s) \succ \sigma(x|t)$,
then for some finite word $u$ and words $v,w,$ 
\begin{align*}
\sigma(x|s) &= u1v, & \sigma(x|t) &= u0w.
\end{align*}
So the definition of the itinerary gives $t > \phi_u(x) \ge s$.
This completes the proof.
\end{proof}


The following is Theorem~17 and part of Corollary~18 of~\citet{Mignosi91}.

\begin{lemma}
\label{lemma:mignosi}
Let $\mathcal{A}_n$ be the set of factors of length $n\in\N$ of lower mechanical words.
Then
\begin{align*}
\card(\mathcal{A}_n) = 1 + \sum_{i=1}^n (n-i+1) \totient(i) = \frac{2n^3}{\pi^2} + O(n^2 \log n).
\end{align*}
\end{lemma}


Now we are ready to bound the number of discontinuities of
the itinerary of a map-with-a-gap as a function of its threshold.
 
\begin{lemma}
\label{lemma:Ot3}
Suppose $\phi_0,\phi_1$ satisfy A1, that $t\in\Z_+$, $x\in\IR$ and $a\in \{0,1\}$. Then
the mapping $s\mapsto A_{1:t}(x,a;s)$ for $s\in\bR$ has at 
most a polynomial number $p(t)$ of discontinuities.
\end{lemma}
\begin{proof}
Let $\mathcal{A}_n$ be the set of all factors of length $n$ of lower mechanical words.
Let $\mathcal{F}_t$ be the set of all words of the form $l^m w$ 
for some $l\in\{0,1\}$, some $m\in\{0, 1, \dots, n\}$ and some factor $w$ 
of a lower mechanical word. 
By Lemma~\ref{lemma:mechanical-factor}, the word $A_{1:t}(x,a;s)$ is in $\mathcal{F}_t$.
Also, Lemma~\ref{lemma:lexs}, shows that the mapping $s \mapsto A_{1:t}(x,a;s)$ is lexicographically non-increasing.
Thus, the number of discontinuities of this mapping is at most 
\begin{align*}
\card(\mathcal{F}_t) &= \card\left( \{ l^m w : l\in \{0,1\}, m \in\Z_+, m \le t, w \in \mathcal{A}_{t-m} \}\right) = 2 \sum_{m=0}^t \card(\mathcal{A}_{t-m}) . 
\end{align*}
Finally, Lemma~\ref{lemma:mignosi} shows the right-hand side is $O(t^4)$. This completes the proof.
\end{proof}


\begin{proof}[Proof of Theorem~\ref{theorem:xs-characterisation}.]
The theorem simply couples together Lemmas~\ref{lemma:mechanical-factor},~\ref{lemma:lexs}
and~\ref{lemma:Ot3}.
\end{proof}

%
%

\section{Proof of Lemma~\ref{lemma:major}}
\label{appendix:major}
Demonstrating Lemma~\ref{lemma:major} 
by combining results in~\cite{Marshall10} requires as much text as a direct proof.
\\

\begin{proof}
Let $\mathcal{X}$ be the set of sequences $X$ with components
$X_k = \sum_{i=1}^k x_i$  for $k=1, 2, \dots, n$ where 
$x_1,x_2, \dots, x_n$ is a non-decreasing sequence on $\R_{++}$. 
Let $g:\mathcal{X}\rightarrow \R$ be the function
\begin{align*}
g(X) := f_1(X_1) + \sum_{i=2}^n f_i(X_i-X_{i-1}) .
\end{align*}
Let $f_i'(\cdot)$ denote the (sub)gradient of $f_i(\cdot)$.
For $i=1, 2, \dots, n-1$, the (sub)gradients of $g(\cdot)$ are
\begin{align*}
\frac{\partial g(X)}{\partial X_i} 
=  f_i'(x_i) - f_{i+1}'(x_{i+1}) 
\le  f_{i+1}'(x_i) - f_{i+1}'(x_{i+1})
\le 0
\end{align*}
where the first inequality holds as $f_{i}'(x) \le f_{i+1}'(x)$ for $x\in\R_{++}$ (by Hypothesis~4)
and the second holds as 
$x_i \le x_{i+1}$ and $f_{i+1}(\cdot)$ is convex (by Hypothesis~3).
Also,
\begin{align*}
\frac{\partial g(X)}{\partial X_n} &=  f_n'(x_n) \le 0
\end{align*}
as $f_n(\cdot)$ is non-increasing (by Hypothesis~3).
Therefore $g(\cdot)$ is non-increasing in all of its arguments.
As the sequences $A,B$ with components $A_k := \sum_{i=1}^k a_i, B_k :=\sum_{i=1}^k b_k$ 
are in $\mathcal{X}$ (by Hypothesis~1) and $A_k \le B_k$ for $k=1, 2, \dots, n$ (by Hypothesis~2), 
it follows that $g(A) \ge g(B)$.
So the definition of $g(\cdot)$ gives
\begin{align*}
\sum_{i=1}^n f(a_i) = g(A) \ge g(B) = \sum_{i=1}^n f(b_i) 
\end{align*}
as claimed.
\end{proof}

%
%

\section{Proof of Lemma~\ref{lemma:integrated}}
\label{appendix:integrated}
We start by recalling the definition of the matrix $M(w)$ corresponding to a given
finite word $w$, which corresponds to the composition of Kalman-Filter variance updates,
and introducing some related matrices $K, S(w)$ and $X$. 
We then prove Claims~1 to~5 of Lemma~\ref{lemma:integrated} in turn.

\begin{definition}
Let $I$ be the 2-by-2 identity matrix.
For $r\in (0,1]$ and $0\le a\le b$, let
\begin{align*}
F&:= \begin{pmatrix} r & 1/r \\ ar & (a+1)/r \end{pmatrix}, &
G&:= \begin{pmatrix} r & 1/r \\ br & (b+1)/r \end{pmatrix}, &
K:= \begin{pmatrix} r & 1/r \\ r-r^3 & -r \end{pmatrix} .
\end{align*}
Let $M(\epsilon) = I, M(0)=F, M(1)=G$ and for any finite word $w$ let
\begin{align*}
M(w) &= M(w_\abs{w}) \cdots M(w_2) M(w_1), & S(w) &= \sum_{i=0}^{\abs{w}} M(w_{1:i}) .
\end{align*}
Let 
\begin{align*}
X:= \begin{pmatrix} -r/(1-r^2) & 0 \\ 0 & 1/r \end{pmatrix}.
\end{align*}
\end{definition}

\begin{remark}
We use the the following facts repeatedly without mention.
Clearly $\det(F)=\det(G)=1$, so that $\det(M(w))=1$ for any word $w$.
Also, $KF=F^{-1}K, KG=G^{-1}K$ and $K^2=I$.
Thus for $A\in \{KF,KG,K\}$ we have $A^2=I$, so $A$ is an \textit{involutory matrix}.
Thus $KM(w)^{-1}K = M(w^R)$, where $w^R$ denotes the reverse
$w_n \dots w_2 w_1$ of a word $w=w_1w_2\dots w_n$.
\end{remark}

\paragraph{Notation.} For a vector $v\in\R^m$ where $m\in\N$, we write $v>0$ if $v_i>0$ for $i=1, 2, \dots, m$ and $v\ge 0$ if $v_i\ge 0$ for $i=1, 2, \dots, m$.
Similarly, for a matrix $P\in\R^{m\times n}$ where $m,n\in\N$, we write $P> 0$ ($P\ge 0$) if $P_{ij}>0$ ($P_{ij}\ge 0$) for $i=1,2, \dots, m$ and $j=1, 2, \dots, n$. 
\\


\subsection{Claim~1}
We require one simple Lemma.
\begin{lemma}
\label{lemma:y0ub}
Suppose $a\in\R_+$ and $r\in(0,1]$. Then the fixed point $y_0$ satisfies 
\begin{align*}
y_0 \le \frac{1}{1-r^2}. 
\end{align*}
\end{lemma}
\begin{proof}
As the positive root of 
\begin{align*}
y_0 = \frac{r^2 y_0 + 1}{ar^2 y_0 + 1+a}
\end{align*}
is a decreasing function of $a$ for $a\in\R_+$, setting $a=0$ gives an upper bound. 
This upper bound $u$ satisfies $u = r^2 u + 1$, so that $u = 1/(1-r^2).$
\end{proof}

\begin{proof}[Proof of Claim~1 of Lemma~\ref{lemma:integrated}.]
We prove the claim for $x$ satisfying
\begin{align*}
\phi_p(0) \le x \le \frac{1}{r-r^2},
\end{align*}
noting that 
\begin{align*}
\phi_p\left( \frac{1}{1-r^2} \right) \le \frac{1}{1-r^2} \le \frac{1}{r-r^2} 
\end{align*}
where the first inequality follows from Lemma~\ref{lemma:incdec} (in Appendix~A)
as Lemma~\ref{lemma:y0ub} gives $1/(1-r^2) \ge y_0 \ge y_p$,
and the second inequality holds as $r\in (0,1]$.

For any word $w$ and for $k=1, 2, \dots, \abs{w}$, let
\begin{align*}
\begin{pmatrix}
u_k \\ v_k
\end{pmatrix} := M(w_{1:k}) \begin{pmatrix} x \\ 1 \end{pmatrix} .
\end{align*}
Clearly $u_k,v_k$ are positive as $x\ge \phi_p(0) \ge 0$ and 
\begin{align*}
M(w_{1:k}) \begin{pmatrix} x \\ 1 \end{pmatrix} \ge 
\begin{pmatrix} r & \frac{1}{r} \\ 0 & \frac{1}{r} \end{pmatrix}^k \begin{pmatrix} 0 \\ 1 \end{pmatrix} > 0.
\end{align*} 
Now, for any $0\le z\le 1/(r-r^2)$ and $H = M(w_k)$, noting that $a,b,r \ge 0$  and $r\le 1$ gives
\begin{align*}
\frac{H_{11} z+H_{12}}{H_{21} z+H_{22}} \le r^2 z + 1 \le \frac{r^2}{r-r^2} + 1 \le \frac{1}{r-r^2}.
\end{align*}
For $k=1, 2, \dots, \abs{w}$, induction using this inequality proves that
\begin{align*}
\frac{u_k}{v_k} \le \frac{1}{r-r^2}.
\end{align*}
(For $k=1$, put $z=x\le 1/(r-r^2)$. For $k>1$, assume that $z=u_{k-1}/v_{k-1} \le 1/(r-r^2).$)
Thus 
\begin{align*}
\begin{pmatrix} u_{k+1}-u_k \\ v_{k+1}-v_k \end{pmatrix}
= (M(w_{k+1})-I) \begin{pmatrix} \frac{u_k}{v_k} \\ 1 \end{pmatrix} v_k 
\ge \begin{pmatrix} r-1 & \frac{1}{r} \\ 0 & \frac{1}{r}-1 \end{pmatrix} \begin{pmatrix} \frac{1}{r-r^2} \\ 1 \end{pmatrix} v_k \ge 0.
\end{align*}
But both $a_k(x), b_k(x)$ are of the form $u_k$ and $c_k(x), d_k(x)$ are of the form $v_k$ for appropriate $w$.
Thus $a_{1:m}(x),b_{1:m}(x),c_{1:m}(x)$ and $d_{1:m}(x)$ are non-decreasing and positive.
\end{proof}

\subsection{Claim~2}
To prove Claim~2 we need two simple Lemmas.
\begin{lemma}
\label{lemma:wtw}
If $w$ is a word, then 
$M(w)-M(w^R) = \text{tr}(KM(w))K.$
\end{lemma}
\begin{proof}
For any $C\in \R^{2\times 2}$, direct calculation gives
$C-K\text{adj}(C)K = \text{tr}(KC) K.$ But $\det(M(w))=1$, so
$M(w)-M(w^R) = M(w)-KM(w)^{-1}K = M(w)-K\text{adj}(M(w))K = \text{tr}(KM(w))K.$
\end{proof}

\begin{lemma}
\label{lemma:xpositive}
Suppose $p$ is a palindrome, $r\in (0,1]$ and $n\in \Z_+$. Then for some $x\ge 0$,
\begin{align*}
M((10p)^n10) - M((01p)^n01) = xK.
\end{align*}
\end{lemma}
\begin{proof}
For any matrices $P,Q\in\R^{2\times 2}$ with $\det(P)=1$ and $Q\ge 0$, direct calculation gives
\begin{align*}
&\hspace{-1cm}[QGFP]_{22} [QFGP]_{21} - [QFGP]_{22} [QGFP]_{21} \\
&=(b-a) \det(P) 
( (1-r^2+a+b+a b) Q_{22}^2+(2+a+b) Q_{22} Q_{21}+Q_{21}^2 ) \ge 0
\end{align*}
as $1\ge r^2, b\ge a\ge 0$.
Therefore, for any words $w,w'$,
\begin{align}
\label{eq:w01w}
\frac{M(w01w')_{22}}{M(w01w')_{21}} \ge \frac{M(w10w')_{22}}{M(w10w')_{21}}.
\end{align}

Let $A=M((10p)^n10),B=M((01p)^n01)$. Repeated application of~(\ref{eq:w01w}) gives
\begin{align*}
\frac{B_{22}}{B_{21}} = \frac{M(01p01p\cdots 01)_{22}}{M(01p01p\cdots 01)_{21}}
\ge 
\frac{M(10p01p\cdots 01)_{22}}{M(10p01p\cdots 01)_{21}}
\ge
\frac{M(10p10p\cdots 10)_{22}}{M(10p10p\cdots 10)_{21}}
=\frac{A_{22}}{A_{21}} .
\end{align*}
As $A, B \ge 0$ and Lemma~\ref{lemma:wtw} gives $A=B+xK$ for some $x\in\R$, it follows that
\begin{align*}
(B+xK)_{21} B_{22} &\ge (B+xK)_{22} B_{21} \quad \Rightarrow \quad K_{21} B_{22} x \ge K_{22} B_{21} x .
\end{align*}
Finally, the fact that $K_{22} \le 0 < K_{21}$ and $B\ge 0$ gives $x\ge 0$.
\end{proof}

\begin{proof}[Proof of Claim~2 of Lemma~\ref{lemma:integrated}.]
For some $t\ge 0$, Lemma~\ref{lemma:xpositive} gives
\begin{align*}
(d/dx)(b_1(x)-a_1(x)) 
&= [M((10p)^n1)-M((01p)^n0)]_{11} \\
&= [(F^{-1}-G^{-1}) M((01p)^n01) + t F^{-1} K]_{11} \\
&= [(0,0) M((01p)^n01)]_1 + t[KF]_{11} \\
&= t \left(r F_{11} + (1/r) F_{21} \right)\\
&\ge 0 
\intertext{and}
b_1(\phi_p(0))-a_1(\phi_p(0)) 
&= [M(p(10p)^n1)-M(p(01p)^n0)]_{12} \\
&= [(F^{-1}-G^{-1}) M(p(01p)^n01) + t F^{-1} K M(p)]_{12} \\
&= t[KFM(p)]_{12} \\
&= t( r M(p0)_{12} + (1/r) M(p0)_{22}) \\
&\ge 0 .
\intertext{Also, if $k>1$ and $w=p_{1:(k-2)}$, then}
(d/dx)(b_k(x)-a_k(x)) &=
[M((10p)^n10w)-M((01p)^n01w)]_{11} \\
&=  t [M(w) K]_{11} \\
&= t (M(w)_{11} r + M(w)_{12} r(1-r^2)) \\
&\ge 0
\intertext{and as $p$ is a palindrome, $p=s w^R$ for some word $s$, so}
b_k(\phi_p(0))-a_k(\phi_p(0)) 
&= [M(p(10p)^n10w)-M(p(01p)^n01w)]_{12} \\
&=  t [M(w) K M(p)]_{12} \\
&=  t [K M(w^R)^{-1} M(w^R) M(s)]_{12} \\
&= t [K M(s)]_{12} \\
&= t ( r M(s)_{12} + (1/r) M(s)_{22})\\
&\ge 0 .
\end{align*}
This completes the proof.
\end{proof}

\subsection{Claim~3}
Claim~3 of Lemma~\ref{lemma:integrated} is more challenging than Claims~1 and~2.
We begin with six Lemmas.


\begin{lemma}
\label{lemma:21bigger}
Suppose $p$ is any palindrome and $r\in (0,1]$. Then for any $k\in\Z_+$,
\begin{align*}
\Delta_k := [M(((10p)^\omega)_{1:k}) - M(((01p)^\omega)_{1:k})]_{21} \ge 0.
\end{align*}
\end{lemma}
\begin{proof}
If $k=1$ then
\begin{align*}
\Delta_k 
&= [M(1)-M(0)]_{21} \\
&= [G-F]_{21} \\
&= (b-a)r \\
&\ge 0.
\intertext{If $k = (n+1) \abs{01p} + 1$ for some $n\in\Z_+$ then Lemma~\ref{lemma:xpositive} shows there is an $x\ge 0$ such that}
\Delta_k
&= [M((10p)^n10p1)-M((01p)^n01p0)]_{21} \\
&= [M(p1)(xK+M((01p)^n01)) - M(p0) M((01p)^n01)]_{21} \\
&= [x M(p1) K + (G-F) M((01p)^{n+1})]_{21} \\
&= x( M(p1)_{21} r + M(p1)_{22} (r-r^3)) \\
&\qquad + (b-a) (r  M((01p)^{n+1})_{12} + (1/r)  M((01p)^{n+1})_{22})\\
&\ge 0.
\intertext{Otherwise, there is a prefix $w$ of $p$ and an $n\in\Z_+$ such that}
\Delta_k
&= [M((10p)^n10w)-M((01p)^n01w)]_{21} \\
&= [M(w) ( M((10p)^n10)-M((01p)^n01) )]_{21} \\
&= [xM(w) K]_{21} \qquad\text{for some $x\ge 0$ by Lemma~\ref{lemma:xpositive}}\\
&= x( M(w)_{21} r + M(w)_{22} (r-r^3)) \\
&\ge 0.
\end{align*}
This completes the proof.
\end{proof}

\begin{lemma}
\label{lemma:22negative}
Suppose $p=ws$ is a palindrome, $n \in \Z_+$ and $r\in (0,1]$. Then
\begin{align*}
[M(p(10p)^n10w)-M(p(01p)^n01w)]_{22} \le 0 .
\end{align*}
\end{lemma}
\begin{proof}
First note that for any finite word $u$,
\begin{align}
\label{eq:M22M21}
M(u)_{22}\ge M(u)_{21}.
\end{align}
Indeed, if $u=\epsilon$ then $M(\epsilon) = I$ so the inequality holds.
Otherwise, for some $c\in \{a,b\}$
\begin{align*}
M(u)_{22}-M(u)_{21} 
= \left[ M(u_{2:\abs{u}})  M(u_1) \begin{pmatrix} -1 \\ 1 \end{pmatrix}  \right]_2 
= \left[M(u_{2:\abs{u}}) \frac{1}{r} \begin{pmatrix} 1-r^2 \\ 1+(1-r^2)c \end{pmatrix} \right]_2 
\ge 0
\end{align*}
as the definition of $M(\cdot)$ assumes that $0\le a < b$ so that $c\ge 0$, 
and as $r\in (0,1]$ and $M(u_{2:\abs{u}}) \ge 0$.

Also, by Lemma~\ref{lemma:xpositive}, for some $x\ge 0$
\begin{align*}
[M(p(10p)^n10w)-M(p(01p)^n01w)]_{22} 
&= [M(s)^{-1} (M((10p)^{n+1})-M((01p)^{n+1})) M(p)]_{22} \\
&= x [M(s)^{-1} K M(p)]_{22} \\
&= x [K M(p s^R)]_{22} \\
&= x r \left( (1-r^2) M(p s^R)_{21} -  M(p s^R)_{22} \right) \\
&\le x r \left( (1-r^2) M(p s^R)_{22} -  M(p s^R)_{22} \right) \\
&\le 0
\end{align*}
where the penultimate line is~(\ref{eq:M22M21}).
\end{proof}

\begin{lemma}
\label{lemma:bigX}
Suppose $p$ is a palindrome, $n \in \Z_+$ and $r\in (0,1]$. Then
\begin{align*}
\left[
	\Bigl( M((10p)^n10)X - M((01p)^n01)X+M((10p)^n1)-M((01p)^n0) \Bigr) M(p)
\right]_{22} &= 0 .
\end{align*}
\end{lemma}
\begin{proof}
Let $P=M(p)$. Solving $KP=P^{-1}K$ shows that
there exist $f,h\in\R$ such that 
\begin{align*}
P = \begin{pmatrix}
\displaystyle\frac{1-f^2 r^2+f h r^2+f^2 r^4}{f r^2+h} & f \\
\displaystyle\frac{-1-f h+h^2+f h r^2}{f r^2+h} r^2 & h \end{pmatrix} .
\end{align*}
Directly substituting this expression into
\begin{align*}
Q_n:=\left[\Bigl(FG(PFG)^nX-GF(PGF)^nX+G(PFG)^n-F(PGF)^n\Bigr)P\right]_{22}
\end{align*}
shows that $Q_0 = Q_1 = 0$. 
(Showing that $Q_1=0$ directly is algebraically tedious and we had to 
check this with computer algebra. 
The authors would be interested in a short demonstration that $Q_1=0$
as this may give insight into related problems.)

Lemma~\ref{lemma:wtw}, then the fact that $\text{tr}(K) = 0,$ 
then the cyclic property of the trace give
\begin{align*}
\text{tr}(PFG) 
= \text{tr}(GFP+\text{tr}(KPFG)K) 
= \text{tr}(GFP)
= \text{tr}(PGF).
\end{align*}
So $PFG$ and $PGF$ are 2-by-2 matrices with unit determinant whose traces are equal.
For some matrices of eigenvectors $U,V$ and some eigenvalue $\lambda \ge 1$,
such matrices may be written in the form
\begin{align*}
PFG &= U\Lambda U^{-1}, & PGF &= V\Lambda V^{-1}, & \Lambda := \begin{pmatrix} \lambda & 0 \\ 0 & 1/\lambda \end{pmatrix} .
\end{align*}
Thus, for some $\alpha,\beta\in\R$,
\begin{align*}
Q_n &= \alpha \lambda^n + \beta \lambda^{-n}
\end{align*}
But $Q_0 = 0$ implies that $\beta = -\alpha$.
Thus $Q_1 = \alpha (\lambda - 1/\lambda) = 0$ implies that either $\alpha = 0$
or $\lambda = 1$. In either case, $Q_n = \alpha (\lambda^n - \lambda^{-n}) = 0$ for all $n\ge 0$.
\end{proof}

\begin{lemma}
\label{lemma:SKM}
Suppose $p$ is a palindrome. Then $S(p) K M(p) = K S(p)$.
\end{lemma}
\begin{proof}
We proceed by induction on the length of $p$.
In the base case, $S(\epsilon)KM(\epsilon) = IKI = KS(\epsilon)$ and
if $a$ is a letter then $S(a)KM(a) = (I+M(a))KM(a)=K(I+M(a)^{-1})M(a)=KS(a)$.
Otherwise, say $p=aqa$ where $a$ is a letter and $S(q) K M(q) = K S(q)$.
Then
\begin{align*}
S(p)KM(p)&=(I+M(a)+S(q)M(a)+M(aqa))K M(aqa) \\
&=(I+M(a)+M(aqa))K M(aqa) + S(q) K M(a)^{-1} M(a) M(q) M(a) \\
&=K(I+M(a)^{-1}+M(aqa)^{-1}) M(aqa) + S(q) K M(q) M(a) \\
&=K(M(aqa) + M(aq) + I) + K S(q) M(a) \\
&=KS(p) .
\end{align*}
\end{proof}

\begin{lemma}
\label{lemma:KSX}
Suppose $w$ is a finite word and $r\in (0,1]$. Then $[K(S(w)-XM(w))]_{22} \ge 0$.
\end{lemma}
\begin{proof}
We proceed by induction on the length of $w$.
In the base case, 
\begin{align*}
[K(S(\epsilon)-XM(\epsilon))]_{22} = [K(I-X)]_{22} = 1-r \ge 0 .
\end{align*}
Otherwise, say $w=ul$ where $\abs{l}=1$, $\abs{u}<\infty$ and $[K(S(u)-XM(u))]_{22}\ge 0$. Then
\begin{align*}
[K(S(w)-XM(w))]_{22} &= [K(S(u)+M(w)-X(M(w)-M(u))-XM(u))]_{22}\\
&= [K(S(u)-XM(u))]_{22} + [K(I-X(I-M(l)^{-1}))M(w)]_{22} \\
&\ge [K(I-X(I-M(l)^{-1})) M(w)]_{22} \\
&= (1-r) (r^2 M(w)_{12}+M(w)_{22})\\
&\ge 0
\end{align*}
where in the penultimate line we substituted the definitions of $K$ and $X$,
noting that $M(l) \in \{F,G\}$,
and where the final inequality follows as $r \le 1$ and $M(w) \ge 0$.
\end{proof}

\begin{lemma}
\label{lemma:Ssum}
Suppose $p$ is a palindrome and $r\in (0,1]$. Then for any $n\in\Z_+$
\begin{align*}
[(S(10p)-I)M(p(10p)^n)-(S(01p)-I)M(p(01p)^n)]_{22} \ge 0 .
\end{align*}
\end{lemma}
\begin{proof}
Let $P:=M(p)$. Then
\begin{align*}
&[(S(10p)-I)M(p(10p)^n)-(S(01p)-I)M(p(01p)^n)]_{22} \\
&=[S(p)(FG(PFG)^n-GF(PGF)^n)P+(G(PFG)^n-F(PGF)^n)P]_{22} \\
&=[S(p)(FG(PFG)^n-GF(PGF)^n)P-(FG(PFG)^n-GF(PGF)^n)XP]_{22} \\
&=[S(p) x K P - xKXP]_{22} \\
&=[x K (S(p) - X)P]_{22} \\
&\ge 0 
\end{align*}
where the second equality uses Lemma~\ref{lemma:bigX},
the third holds for some $x\ge 0$ by Lemma~\ref{lemma:xpositive},
the fourth follows from Lemma~\ref{lemma:SKM}
and the final inequality is Lemma~\ref{lemma:KSX}.
\end{proof}

\begin{proof}[Proof of Claim~3 of Lemma~\ref{lemma:integrated}.]
Say $1\le k\le m$. Using Lemmas~\ref{lemma:21bigger},~\ref{lemma:22negative} and~\ref{lemma:Ssum} successively gives
\begin{align*}
\sum_{i=1}^k (d_i(x)-c_i(x)) 
&= \sum_{i=1}^k [M((10p)^n (10p)_{1:i})-M((01p)^n (01p)_{1:i})]_{21} x \\
&\quad	+ \sum_{i=1}^k 	[M((10p)^n (10p)_{1:i})-M((01p)^n (01p)_{1:i})]_{22} \\
&\ge \sum_{i=1}^k [M((10p)^n (10p)_{1:i})-M((01p)^n (01p)_{1:i})]_{21} \frac{M(p)_{12}}{M(p)_{22}} \\
&\quad	+ \sum_{i=1}^k 	[M((10p)^n (10p)_{1:i})-M((01p)^n (01p)_{1:i})]_{22} \\
&= \frac{1}{M(p)_{22}} \sum_{i=1}^k 
		[M(p(10p)^n (10p)_{1:i})-M(p(01p)^n (01p)_{1:i})]_{22} \\
&\ge \frac{1}{M(p)_{22}} \sum_{i=1}^m
		[M(p(10p)^n (10p)_{1:i})-M(p(01p)^n (01p)_{1:i})]_{22} \\
&= \frac{1}{M(p)_{22}} [(S(10p)-I)M(p(10p)^n)-(S(01p)-I)M(p(01p)^n)]_{22} \\
&\ge 0 .
\end{align*}
This completes the proof.
\end{proof}

\subsection{Claim~4}
The proof of Claim~4 requires only one preparatory Lemma.

\begin{lemma}
\label{lemma:inverses}
Suppose $w$ is any finite word. Then
\begin{align*}
[M(w)^{-1}]_{21} \le 0 \le [M(w)^{-1}]_{22} .
\end{align*} 
\end{lemma}
\begin{proof}
We use induction on the length of $w$. 
In the base case, $M(\epsilon)=I$, for which
\begin{align*}
[I^{-1}]_{21} = 0 \le 1 = [I^{-1}]_{22} .
\end{align*} 
Otherwise, suppose $w=0u$ (the case $w=1u$ is similar), let $U := M(u)^{-1}$ and assume that 
\begin{align*}
U_{21} \le 0 \le U_{22} .
\end{align*} 
Then the induction assumption and and fact that $a, r \ge 0$ give
\begin{align*}
[M(w)^{-1}]_{21} &= [U F^{-1}]_{21} = \frac{1}{r} ((1+a) U_{21} - ar^2U_{22}) \le 0 \\ 
[M(w)^{-1}]_{22} &= [U F^{-1}]_{22} = \frac{1}{r} (U_{22} r^2 - U_{21}) \ge 0.
\end{align*} 
This completes the proof.
\end{proof}

\begin{proof}[Proof of Claim~4 of Lemma~\ref{lemma:integrated}.]
Claim~3 of Lemma~\ref{lemma:integrated} applied for $k=1$ shows that 
\begin{align*}
c_1(x)\le d_1(x).
\end{align*}

For $k=2, 3, \dots, m$, 
as Lemma~\ref{lemma:21bigger} shows that $c_k(x)-d_k(x)$ is a decreasing function of $x$,
it suffices to prove that 
\begin{align*}
\left[ M((01p)^n (01p)_{1:k}) \begin{pmatrix} \phi_p\left(\frac{1}{1-r^2}\right) \\ 1 \end{pmatrix} \right]_2
\ge 
\left[ M((10p)^n (10p)_{1:k}) \begin{pmatrix} \phi_p\left( \frac{1}{1-r^2}\right) \\ 1 \end{pmatrix} \right]_2.
\end{align*}
The left-hand side minus the right-hand side, up to a positive factor, is
\begin{align*}
& \hspace{-1cm}\left[ \left( M((01p)^n (01p)_{1:k}) - M((10p)^n (10p)_{1:k}) \right) M(p) \begin{pmatrix} 1 \\ 1-r^2 \end{pmatrix} \right]_2 \\
&= -z \left[ M(p_{1:(k-2)}) K M(p) \begin{pmatrix} 1 \\ 1-r^2 \end{pmatrix} \right]_2 
\quad\text{for some $z \ge 0$ by Lemma~\ref{lemma:xpositive}}\\ 
&= -z \left[ M(p_{1:(\abs{p}-k+2)})^{-1} K \begin{pmatrix} 1 \\ 1-r^2 \end{pmatrix} \right]_2\\
&= -z \left[ M(p_{1:(\abs{p}-k+2)})^{-1} \begin{pmatrix} \frac{1}{r} \\ 0 \end{pmatrix} \right]_2\\
&= -\frac{z}{r} [M(p_{1:(\abs{p}-k+2)})^{-1}]_{21}\\
&\ge 0
\end{align*}
where the last line is Lemma~\ref{lemma:inverses}. 
Therefore 
\begin{align*}
c_k(x) \ge d_k(x) \quad \text{for $k=2, 3, \dots, m$ and $x\le \phi_p\left(\frac{1}{1-r^2}\right)$.}
\end{align*}
This completes the proof.
\end{proof}

\subsection{Claim~5}
\begin{proof}[Proof of Claim~5 of Lemma~\ref{lemma:integrated}.]
We show that 
\begin{align*}
\phi_w(0) \le y_{01w} < y_{10w} \le \phi_w\left( \frac{1}{1-r^2} \right)
\end{align*}
for any finite word $w$ (not just for palindromes $p$).

The first inequality follows as 
\begin{align*}
\phi_w(0) \le \phi_w(\phi_{01}(0)) = \phi_{01w}(0) \le y_{01w} 
\end{align*}
as $0 \le \phi_{01}(0)$, as $\phi_w$ is increasing, as $0 \le y_{01w}$ and by Lemma~\ref{lemma:incdec} (in Appendix~A).

The second inequality holds as $\phi_{01}(x) < \phi_{10}(x)$ for $x\in\R_+$. Thus
\begin{align*}
y_{01w} = \phi_w(\phi_{01}(y_{01w})) < \phi_w(\phi_{10}(y_{01w})) = \phi_{10w}(y_{01w})
\end{align*}
so applying Lemma~\ref{lemma:incdec} gives
\begin{align*}
y_{01w} < y_{10w}.
\end{align*}

The third inequality holds as  
\begin{align*}
y_{10w} = \phi_w(y_{w10}) < \phi_w(y_0) \le \phi_w\left( \frac{1}{1-r^2} \right)
\end{align*}
by definition of $y_{10w}$, as $y_{w10} < y_0$, as $\phi_w$ is increasing and by Lemma~\ref{lemma:y0ub}.

This completes the proof.
\end{proof}

%
%
\section{LQG Control with Costly Observations}
\label{appendix:LQG}
\paragraph{Problem.} The analysis of this paper gives optimal policies for a version of the 
classic linear-quadratic-Gaussian (LQG) control problem 
in which observations are costly and the nature of each
observation is controlled through a query action.
In this problem, the state is partially observed through measurements 
as described by the system equations
\begin{align*}
X_0 &\sim \mathcal{N}(x_0,v_0), &
X_{t+1}|X_t,u_t &\overset{\text{i.i.d.}}{\sim} \mathcal{N}(A X_t + B u_t, \Sigma_X), &
Y_{t+1}|X_{t+1},a_t &\overset{\text{i.i.d.}}{\sim} \mathcal{N}(X_{t+1},\Sigma_Y(a_t))
\end{align*}
for $t\in\Z_+$, where 
$X_t \in \R$ is the state with initial mean $x_0$ and variance $v_0$,
$u_t\in\R$ is the control,
$Y_{t+1}\in\R$ is a measurement which depends on a query action $a_t\in\{0,1\}$
and where $\Sigma_X, \Sigma_Y(a_t) > 0$ are variances.
For measurement cost $c(a_t) \in \R$, the objective is to find a non-anticipative policy $\pi$ that selects actions $u_t, a_t$ 
so as to minimise the $\beta$-discounted performance functional
\begin{align*}
\E\left( \sum_{t=0}^\infty \beta^{t} (DX_t^2 + Fu_t^2 + c(a_t))\ \bigg| \ \pi, x_0, v_0 \right) .
\end{align*}
Thus the policy can select $u_t, a_t$ based only on the observed history 
$H_t$ at time $t$, which consists of 
$x_0, v_0, a_0, a_1, \dots, a_{t-1}, u_0, u_1, \dots, u_{t-1}, Y_1, Y_2, \dots, Y_t$.
Under the Bayesian filter, the information state is given by the posterior mean
$x_t := \E[X_t | H_t]$ and variance $v_t := \E[(X_t-x_t)^2|H_t]$.
As $\E[X_t^2 | H_t] = x_t^2 + v_t$, it is not hard to see that the problem reduces to the dynamic program
\begin{align}
\label{eq:DP}
V(x_t,v_t) = \min_{u_t\in\R,a_t\in\{0,1\}} \bigg\{
	Dx_t^2 + Dv_t + Fu_t^2 + c(a_t) + \beta \E[V(x_{t+1},v_{t+1})|x_t,v_t,a_t,u_t] 
\bigg\} 
\end{align}
where the expectation is over the following Markovian transitions of the information state:
\begin{align*}
x_{t+1} | x_t, v_t, a_t, u_t &\sim \mathcal{N}(Ax_t+Bu_t,A^2v_t + \Sigma_X - \phi_{a_t}(v_t)) \\
v_{t+1} | x_t, v_t, a_t, u_t &= \phi_{a_t}(v_t) .
\end{align*}

\begin{corollary}
Suppose 
$A \in [-1,1]$, 
$B \in \R$ with $B\ne 0$,
$D \in\R_{++}$, 
$F\in\R_{+}$, 
$\beta\in (0,1)$, 
$\Sigma_Y(a) \in [0,\infty]$ for $a\in \{0,1\}$ with $\Sigma_Y(0)\ge\Sigma_Y(1)$,
and that $c(a) \in \R$ for $a\in \{0,1\}$ with $c(0)\le c(1)$.
Then an optimal policy for the problem of linear-quadratic-Gaussian control with costly observations is to set 
\begin{align*}
a_t &= \begin{cases} 1 & \text{if $v_t \ge z$} \\ 0& \text{if $v_t < z$} \end{cases} 
\quad \text{and} \quad u_t = - L x_t
\end{align*}
for some $L \in \R$ and $z \in \bR$.
In particular 
\begin{align*}
L &= \frac{A}{B+\frac{F}{\beta BR}}
\end{align*}
where $R$ is the unique positive root of the quadratic equation
\begin{align*}
-\beta B^2 R^2 + (\beta B^2 D + \beta A^2 F - F) R + DF = 0.
\end{align*}
\end{corollary}
\begin{proof}
For trial solutions of the form $V(x,v) = R x^2 + R v + g(v)$, where $R\in\R$
and $g:\R_+\rightarrow \R$ are to be determined, the expectation in (\ref{eq:DP}) is
\begin{align*}
\hspace{2cm}&\hspace{-2cm} 
\E[V(x_{t+1},v_{t+1})|x_t,v_t,a_t,u_t] \\
&= R ((Ax_{t+1}+Bu_t)^2+A^2v_t+\Sigma_X-\phi_{a_t}(v_t)+\phi_{a_t}(v_t))+ g(\phi_{a_t}(v_t)) .
\end{align*}
Thus (\ref{eq:DP}) is solved if
\begin{align}
Rx_t^2 + Rv_t + g(v_t) &= \min_{u_t\in\R} \bigg\{ 	Dx_t^2 + Fu_t^2 + \beta R (Ax_{t+1}+Bu_t)^2 \bigg\} \nonumber \\
\label{eq:separation}
&\qquad + \min_{a_t\in\{0,1\}} \bigg\{ c(a_t) +\beta R\Sigma_X + (D+\beta RA^2)v_t+ \beta g(\phi_{a_t}(v_t)) \bigg\} .
\end{align}
Now the minimum with respect to $u_t$ is achieved if the coefficient ($F+\beta B^2 R$) of $u_t^2$ is positive, in which case the minimiser is
\begin{align*}
u_t = - \frac{\beta A B R}{F+\beta B^2 R} x_t.
\end{align*}
So (\ref{eq:separation}) is solved if $R$ satisfies
\begin{align*}
R &= D + F\left( \frac{\beta A B R}{F+\beta B^2 R} \right)^2 + \beta R \left(A-B\frac{\beta A B R}{F+\beta B^2 R}\right)^2 
\end{align*}
and if $g(\cdot)$ satisfies the dynamic program 
\begin{align}
\label{eq:gDP}
g(v) = \min_{a\in\{0,1\}} \bigg\{ c(a) + \beta R \Sigma_X + \alpha v + \beta g(\phi_a(v)) \bigg\}  
\end{align}
where $\alpha := D - (1-\beta A^2) R$. 

After simple algebra, the condition on $R$ is equivalent to the quadratic equation
\begin{align*}
-\beta B^2 R^2 + (\beta B^2 D + \beta A^2 F - F) R + DF = 0.
\end{align*}
Using Descartes' rule of signs and considering the cases $F=0$ and $F>0$ separately, 
we see that this equation has a unique positive root
for $\beta B^2 > 0$ and $D>0$.

To apply Theorem~1 to the dynamic program for $g(\cdot)$ we must ensure that $\alpha \ge 0$.
Noting that $m:=1-\beta A^2 > 0$ by the hypotheses about $A,\beta$, we see that
$\alpha \ge 0$ if $R \le D/m$.
But, substituting $R=y+(D/m)$ in the equation for $R$ gives
\begin{align*}
0 &= [-\beta B^2 R^2 + (\beta B^2 D + \beta A^2 F - F) R + DF]_{R=y+(D/m)} \\
&= -\ \beta B^2 y^2 \ -\ ((\beta^2A^2+\beta)B^2 (D/m) + mF) y \ -\ \beta^2 A^2 B^2 (D/m)^2
\end{align*}
in which the coefficients of $y^0,y^1,y^2$ are all negative by the hypotheses about $A,B,D,F$ and $\beta$.
So this quadratic equation for $y$ has no positive roots
and it follows that $\alpha \ge 0$.
Therefore Theorem~1 
shows that there is an optimal policy for (\ref{eq:gDP}) that sets $a_t = 1$ if and only if
$v_t \ge z$ for some $z\in\bR$. 
This completes the proof.
\end{proof}

\vskip 0.2in
\bibliography{jmlr}
\end{document}